\documentclass{article}

\usepackage[preprint]{icml2026}  
\makeatletter
\renewcommand{\Notice@String}{}
\makeatother
\definecolor{cvprblue}{rgb}{0.21,0.49,0.74}
\usepackage[pagebackref,breaklinks,colorlinks,allcolors=cvprblue]{hyperref}
\newcommand{\std}[1]{\textbf{\scriptsize\tiny\textcolor{gray}
{$\pm#1$}}} 
\usepackage{xspace}
\newcommand{\Vol}{\operatorname{Vol}}
\newcommand{\E}{\mathbb{E}}
\newcommand{\Prb}{\mathbb{P}}
\newcommand{\R}{\mathbb{R}}
\newcommand{\cX}{\mathcal{X}}
\newcommand{\cY}{\mathcal{Y}}

\newcommand{\diag}{\mathrm{diag}}
\newcommand{\tr}{\mathrm{tr}}
\newcommand{\detm}{\mathrm{det}}
\newcommand{\cN}{\mathcal{N}}
\newcommand{\TV}{d_{\mathrm{TV}}}
\newcommand{\KL}{D_{\mathrm{KL}}}
\newcommand{\method}{\emph{MERS}\xspace}
\newcommand{\probcover}{\emph{ProbCover}\xspace}
\newcommand{\maxherding}{\emph{MaxHerding}\xspace}

\newcommand{\buffer}{$\mathcal{|M|}$\xspace}

\newcommand{\myparagpar}[1]{\noindent\textbf{#1}\;}

\usepackage{algpseudocode}
\renewcommand{\algorithmiccomment}[1]{\hfill$\triangleright$ #1}
\usepackage{graphicx}
\usepackage{xcolor}         

\usepackage{amsthm}      

\theoremstyle{definition}
\newtheorem{definition}{Definition}

\newtheorem{theorem}{Theorem}
\newtheorem{lemma}{Lemma}
\newtheorem{proposition}{Proposition}
\newtheorem{corollary}{Corollary}
\usepackage{mathtools} 

\usepackage{textcomp}

\usepackage{subcaption}
\usepackage{paralist}
\usepackage{enumitem}
\usepackage{xspace}
\usepackage{amsmath}
\usepackage{amsfonts} 
\usepackage{amssymb}
\usepackage{mathtools}
\usepackage{dblfloatfix}

\usepackage{booktabs} 
\usepackage{xcolor}         
\usepackage{tikz}
\usetikzlibrary{positioning, arrows.meta}
\usetikzlibrary{shapes.geometric}

\usepackage{graphicx} 
\usepackage{booktabs,xcolor,colortbl, makecell}
\definecolor{tableShade}{RGB}{245,245,245}
\usepackage{array}
\newcolumntype{C}[1]{>{\centering\arraybackslash}m{#1}}
\begin{document}

\twocolumn[
\icmltitle{Leveraging Complementary Embeddings for Replay Selection in Continual Learning with Small Buffers}

\begin{center}
{\Large\bf Danit Yanowsky$^{1}$ \quad Daphna Weinshall$^{1}$\par}
\vspace{0.4em}
{\normalsize $^{1}$School of Computer Science and Engineering\par}
{\normalsize The Hebrew University of Jerusalem\par}
\vspace{0.4em}
{\normalsize \ttfamily \{danit.yanowsky,daphna.weinshall\}@mail.huji.ac.il\par}
\end{center}
\printAffiliationsAndNotice{}

]

\begin{abstract}
Catastrophic forgetting remains a key challenge in Continual Learning (CL). In replay-based CL with severe memory constraints, performance critically depends on the sample selection strategy for the replay buffer. Most existing approaches construct memory buffers using embeddings learned under supervised objectives. However, class-agnostic, self-supervised representations often encode rich, class-relevant semantics that are overlooked. We propose a new method, \emph{Multiple Embedding Replay Selection} (\method), which replaces the buffer selection module with a graph-based approach that integrates both supervised and self-supervised embeddings. Empirical results show consistent improvements over SOTA selection strategies across a range of continual learning algorithms, with particularly strong gains in low-memory regimes. On CIFAR-100 and TinyImageNet, \method outperforms single-embedding baselines without adding model parameters or increasing replay volume, making it a practical, drop-in enhancement for replay-based continual learning.
\end{abstract}

\section{Introduction}

\begin{figure}[b!]  

    \centering
    \includegraphics[width=\linewidth]{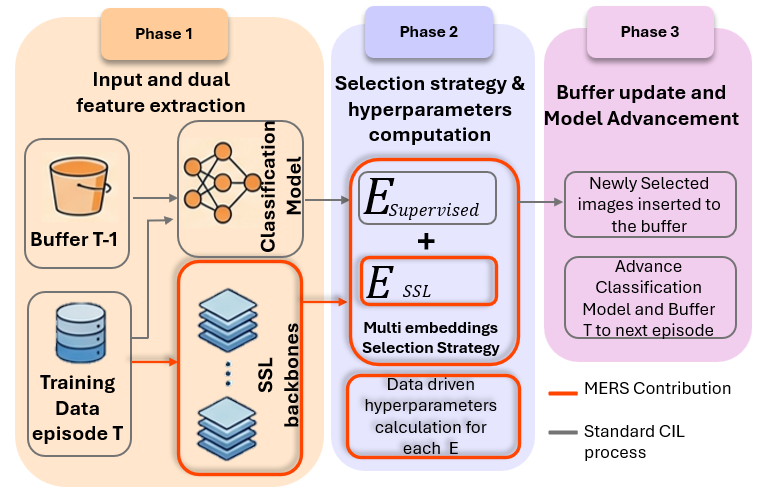}
    \caption{Illustration of our \method in the class-incremental learning (CIL) setup, after training episode T.}
\label{fig:mers_flow}
\vspace{-0.4cm}
\end{figure}


Continual Learning (CL) deals with the challenge of training models while acquiring knowledge from a stream of data whose distribution changes over time. Unlike conventional training on a fixed dataset, many real-world settings, such as autonomous driving, personalized assistants, or robotic agents, must cope with non-stationary environments, where new concepts appear and old ones may become rare or disappear. A central obstacle in this setting is \emph{catastrophic forgetting}~\cite{mccloskey1989catastrophic, ratcliff1990connectionist}: when trained naively on new data, neural networks tend to overwrite previously acquired knowledge, leading to severe performance degradation on past tasks.

This challenge is particularly acute in the \emph{class-incremental learning} (CIL) scenario, where each episode introduces new classes, and at test time the model must jointly classify all classes seen so far. Among the many approaches proposed to mitigate forgetting, replay-based methods have emerged as a simple and effective family of techniques. Experience replay (ER)~\cite{rolnick2019experience}, and its variants such as ER-ACE~\cite{caccia2021new} and MIR~\cite{aljundi2019online}, maintain a small \emph{memory buffer} of past examples and interleave them with current data during training. Under tight memory constraints, however, performance hinges on \emph{which} examples are stored for replay. A large body of work has therefore focused on exemplar selection strategies that aim to maximize diversity or representativeness of the buffer, for example through herding \cite{rebuffi2017icarl}, clustering-based selection~\cite{bang2021rainbowmemorycontinuallearning, chaudhry2021using}, or coverage-based methods 
\cite{shaul2024teal}. 


Most existing selection strategies operate in a \emph{single} representation space: they rely on embeddings produced by the current supervised model, typically the penultimate layer of the classifier. However, a supervised embedding tends to specialize to the current episode: it concentrates geometry along class-discriminative directions and compresses directions that are presently irrelevant. In class-incremental learning, this can make rehearsal and buffer construction fragile - exemplars that look “representative” under the old supervised geometry (e.g., via uniform sampling or mean/coverage criteria) are not necessarily the ones that preserve separability as new classes arrive. This is closely related to distribution shift in domain transfer, which we leverage in the theoretical analysis in Section~\ref{sec:theory}.

To mitigate this risk, we pair the supervised embedding with a self-supervised embedding. The latter typically induces a broader, more nearly uniform feature distribution, and is therefore less likely to neglect directions that are uninformative for current classes but crucial for future ones. Related ideas appear in continual representation learning that incorporates self-supervised objectives (e.g., CaSSLe \citep{fini2022casssle}, SSCIL \citep{ni2021sscil}). In contrast, we do not seek to replace the supervised representation; instead, we integrate supervised and self-supervised geometries through the lens of point coverage. The goal is to reach a sweet spot, retaining strong discrimination on current classes while maintaining a more uniform geometry that better anticipates unseen classes.


Building on these insights, we propose \emph{Multiple-Embedding Replay Selection} (\method), a simple, modular enhancement to replay-based continual learning (see illustration in Fig.~\ref{fig:mers_flow}). Conceptually, \method replaces the usual single-embedding selection step by a \emph{coverage objective} defined jointly over several embedding spaces, e.g., a supervised classifier embedding and a self-supervised SimCLR embedding. We show that this objective can be cast as a weighted maximum $k$-coverage problem over groups, in which each candidate example covers a neighborhood of points in each embedding space. \method automatically adapts the scale of each embedding using non-parametric density estimation, and assigns a weight to each embedding that reflects its effective contribution. Intuitively, this encourages the buffer to cover dense and diverse regions \emph{across} all embeddings, rather than overfitting to the geometry of a single view.

From a methodological standpoint, \method can be understood as a principled extension of coverage-based selection from active learning to the replay setting, preserving the underlying replay backbone while generalizing it to operate jointly over multiple embedding spaces. 
For a fixed memory budget, our greedy selection algorithm retains the classical $(1 - 1/e)$ approximation guarantee for submodular coverage, while remaining practical to implement and efficient in both time and space. Crucially, \method is a \emph{drop-in module}: it requires no architectural modifications to the continual learner, introduces no additional trainable parameters, and can be seamlessly combined with existing replay-based methods, such as ER, ER-ACE, or MIR, by simply replacing the buffer update rule.

We evaluate \method in the class-incremental setting on Split CIFAR-100 and Split TinyImageNet, using three replay-based continual learning algorithms and both supervised and self-supervised embeddings. Across all methods and datasets, \method consistently outperforms single-embedding baselines under the same memory budget, with particularly pronounced gains in low-buffer regimes. We further analyze the role of each embedding and the effect of our data-driven alignment and weighting scheme, showing that the Multiple Embedding formulation is key to the observed improvements.


Summary of main contributions:
\begin{itemize}[itemsep=0.2em, topsep=0.1em, parsep=0pt, partopsep=0pt]
\item We propose \textbf{\method}, a coverage-based replay selection framework that jointly leverages supervised and self-supervised embeddings to capture complementary data geometry under tight memory constraints.
\item We introduce a non-parametric alignment strategy based on $k$-NN density estimation that adapts selection scales and weights across embeddings without adding persistent model parameters.
\item We show that \method achieves \emph{state-of-the-art performance} on Split CIFAR-100 and  Split TinyImageNet, with especially strong gains in low-memory regimes.
\end{itemize}

\section{Related Work}
\label{sec:related}

\paragraph{Continual learning paradigms.}
CL approaches are often grouped into: (i) regularization-based methods that constrain parameter updates to preserve prior knowledge (e.g., EWC \citep{kirkpatrick2017overcoming}, LwF \citep{li2017learning}); (ii) architecture-based methods that expand capacity across tasks (e.g. HAT \cite{serrà2018overcomingcatastrophicforgettinghard}, DAN \cite{yoon2018lifelonglearningdynamicallyexpandable}); and (iii) Replay-based methods that maintain a small memory of exemplars for replay (e.g., ER \cite{rolnick2019experience}, ER-ACE \cite{caccia2021new}). In CIL, rehearsal is particularly competitive under tight memory budgets because it is able to preserve decision boundaries as the label set grows \cite{hou2019learning}. STAR~\cite{eskandar2025starstabilityinducingweightperturbation} introduces a method-agnostic replay mechanism with adaptive sample reweighting, achieving state-of-the-art results under tight memory constraints.

\textbf{Selection strategies.}
A central challenge in replay-based continual learning is \emph{exemplar selection}. Early methods such as iCaRL employed herding to approximate class centroids in a fixed feature space \cite{hou2019learning}. More recent approaches fall into two families: (i) gradient-based methods \citep[e.g., GSS][]{aljundi2019gradient} that prioritize samples likely to induce interference, and (ii) representativeness-oriented methods \citep[e.g., TEAL][]{shaul2024teal} that retain typical samples based on neighborhood structure.

\textbf{Coverage-based selection and its guarantees.}
Coverage-based methods cast exemplar selection as a geometric covering problem. Most prior CL heuristics compute coverage in a single embedding at a fixed scale \citep[e.g.][]{isele2018selectiveexperiencereplaylifelong}
. 
Solving a related problem for active learning, \probcover casts buffer selection as graph coverage 
\citep{yehuda2022active}, while \maxherding introduces kernel smoothing \cite{bae2024generalized}. In contrast, our method generalizes coverage to multiple embeddings and adapts locality per embedding using nonparametric statistics, which is critical in tiny-buffer regimes.

\textbf{Self-supervised representations for CL.}
Self‑supervised learning (SSL) captures class‑agnostic invariances that naturally complement supervised features \citep{uelwer2025survey}. Methods such as SimCLR \citep{chen2020simple}, VICReg \cite{bardes2021vicreg}, and DINO \cite{caron2021emerging} learn rich embeddings without label supervision.

These SSL representations have already demonstrated effective transfer to object detection, semantic segmentation, depth estimation, robotics manipulation, and few‑shot recognition, often rivaling or surpassing supervised pretraining \cite{uelwer2025survey}.  Yet most rehearsal‑based CIL methods still choose exemplars solely in the \emph{supervised} feature space of the current classifier, with only a handful operating purely in an SSL space \citep[e.g.][]{ni2021sscil}, with known selection strategies such as herding applied unchanged \cite{lee2024pretrainedmodelsbenefitequally}. 

In this work we exploit supervised and SSL embeddings in a \emph{complementary manner}, preserving both class-discriminative and class-agnostic structure and yielding consistent gains in tiny-buffer continual-learning regimes.

\textbf{Multi-view learning.} This is an ML paradigm where data is represented through multiple distinct feature sets or "views" (e.g., text and image) \cite{yu2025review}. Common approaches include co-training and multi-view representation learning \cite{zheng2023comprehensive}. The central idea is to leverage the complementary information in these views to improve performance, often by enforcing consistency or agreement across them. In contrast, our approach aims to exploit variability among representation in order to achieve a more representative set of examples, rather than achieving a single coherent view of the data.

\section{Our method: \method}

The proposed method, termed \emph{Multiple Embedding Replay Selection} (\method), is designed to enhance replay-based approaches within the CIL framework. The method, illustrated in Fig.~\ref{fig:mers_flow}, replaces the buffer selection rule with a coverage-based method, which integrates in turn supervised and self-supervised embeddings. Its primary benefits are expected to manifest in low-memory buffer regimes.
Buffer selection is performed independently for each class, using a fixed per-class budget. All definitions below apply to the samples of a single class unless stated otherwise.

\subsection{Notations and definitions}

Let $X=\left\{x_i\right\}_{i=1}^n$ represent a set of $n$ data points, where $x_i\in\ \mathcal{X}$. For this dataset define the graph $G = (V, E)$, with vertices $V=\left\{v_i\right\}_{i=1}^n$ where $v_i\leftrightarrow x_i$, and edges $e_{i,j}=D(x_i,x_j)$ for some distance metric $D \colon \mathcal{X} \times \mathcal{X} \to \mathbb{R}_{\geq 0}$. 

With multiple embeddings, each dataset can now be represented by a collection of graphs $\left\{V,E^{(m)}\right\}$, where $m\in[M]$ indexes the embeddings, $v_i\leftrightarrow x_i$ and $e_{i,j}^{(m)}=D(z_i^{(m)},z_j^{(m)}$) for an embedding $f^{(m)} : \mathcal{X} \to \mathcal{Z}^{(m)}$.
\begin{definition}[$\delta$-ball]
\label{def:ball}
Fix $\delta > 0$, and consider an embedding $f^{(m)} : \mathcal{X} \to \mathcal{Z}^{(m)}$ where $z^{(m)}_{x} = f^{(m)}(x)$.
Define 
\begin{equation*}
B^{(m)}_{\delta}(x)\;=\;\bigl\{x' \in X \;\bigl|\;D\!\bigl(z^{(m)}_{x'}, z^{(m)}_{x}\bigr) \le \delta\bigr\}
\end{equation*}
$B^{(m)}_{\delta}(x)$ denotes the set of points whose embedding lies inside the ball of radius~$\delta$ centered at $x$ in embedding~$m$.
\end{definition}
\noindent

\myparagpar{Maximum k-Coverage with multiple groups.}
The optimization problem, which lies at the heart of our method, can be shown to be a known variant of the k-coverage problem, whose 2-groups version is defined as follows:
\begin{definition}[\texorpdfstring{Maximum $k$-Coverage with two groups}{Maximum k-Coverage with two groups}]
\label{def:coverage}
Let $U$ be a universe of elements, partitioned into two disjoint subsets
$U^1$ and $U^2$ such that $U = U^1 \cup U^2$ and $U^1 \cap U^2 = \emptyset$.
Each element $e \in U^i$ is associated with a nonnegative weight 
$\alpha_i(e) \in \mathbb{R}_{\ge 0}$, where the weight functions 
$\alpha_1, \alpha_2$ may differ between the two groups.

Let $\mathcal{S} = \{S_1, S_2, \dots, S_l\}$ be a family of subsets of $U$,
and let $k \in \mathbb{N}$ be a budget parameter.
For a subcollection $\mathcal{A} \subseteq \mathcal{S}$, define the 
\emph{coverage weight} as
\begin{equation*}
 \mathrm{Coverage}(\mathcal{A}) 
~=~ \sum_{e \in \bigcup_{S \in \mathcal{A}} S \cap U^1} \alpha_1(e) 
   ~+~ \sum_{e \in \bigcup_{S \in \mathcal{A}} S \cap U^2} \alpha_2(e).
\end{equation*}
The goal is to select a sub-collection $\mathcal{A} \subseteq \mathcal{S}$ of size
at most $k$ that maximizes $\mathrm{Coverage}(\mathcal{A})$.
\end{definition}
This formulation can be extended to multiple embeddings. 

\subsection{Coverage-based selection, a single embedding}
\label{subsec:coverage_selection}

A coverage-based selection strategy seeks a small representative subset of $X$ by maximizing a suitable notion of \emph{coverage} on a graph built from the data. 
%
%
More specifically, \probcover~\citep{yehuda2022active} selects a subset $L^* \subset X$ of size at most $b$ that maximizes the number of points covered by the union of corresponding $\delta$-balls:
%
\begin{equation*}
  L^*
  \;=\;
  \arg\max_{L \subseteq X,\; |L| = b}
  \left\vert
    \bigcup_{x \in L} B_\delta(x)
  \right\vert.
\end{equation*}(Superscript $(m)$ is omitted given a single embedding).

Equivalently, \probcover seeks a subset $L^*$ such that the number of points that lie within distance $\delta$ of at least one point in $L^*$ is maximized. This is equivalent to the $b-$set max coverage problem. \maxherding~\citep{bae2024generalized} generalizes this idea by replacing hard $\delta$-ball coverage with a continuous kernel-based similarity measure (e.g., an RBF kernel centered at each selected point), where the underlying objective remains a (soft) notion of coverage.

\subsection{Coverage-based selection, multiple embeddings}
\label{subsec:multi-embedding-coverage}

We now generalize the coverage objective to the multiple embedding setting considered in this work. 
Intuitively, each embedding captures different aspects of the data geometry; we therefore aim to select a buffer that covers dense regions \emph{across all embeddings}. To this end, we define a weighted multiple embedding coverage objective.

\begin{definition}[Buffer selection with weighted Multiple Embedding coverage]
\label{def:weighted-coverage}
Let $\alpha_1,\dots,\alpha_M \ge 0$ denote non-negative weights that reflect the relative importance of each embedding. For a candidate subset $L \subseteq X$, define
\vspace{-6pt}
\begin{equation}
\label{eq:F(L)}
  F(L)
  \;=\;
  \sum_{m=1}^M \alpha_m
  \Bigl|
    \bigcup_{x_i \in L} B^{(m)}_{\delta_m}(x_i)
  \Bigr|.
\end{equation}
\vspace{-6pt}
Given a budget $b$, the buffer-selection problem becomes
\begin{equation*}
  L^*
  \;=\;
  \arg\max_{L \subseteq X,\; |L| = b} F(L).
\end{equation*}
\end{definition}
The optimization problem in~\eqref{eq:F(L)} is equivalent to a special case of the \emph{weighted maximum $k$-coverage} problem with $M$ groups. To make this connection explicit, define, for each embedding $m$, a ground set $U_m$ that contains one element $u_i^{(m)}$ for every datapoint $x_i \in X$, and define the global ground set $U = \biguplus_{m=1}^M U_m$ (disjoint union). For each datapoint $x_i$, associate the subset
\vspace{-6pt}
\begin{equation}
  S_i
  \;=\;
  \bigcup_{m=1}^M
  \bigl\{
    u_j^{(m)} \in U_m
    \,\big|\,
    x_j \in B^{(m)}_{\delta_m}(x_i)
  \bigr\}.
\end{equation}
For any $L \subseteq X$ we can rewrite (\ref{eq:F(L)}) as follows:
\begin{equation}
\label{eq:weighted-coverage}
  F(L)
  \;=\;
  \sum_{m=1}^M \alpha_m
  \Bigl|
    \bigcup_{i : x_i \in L} \bigl(S_i \cap U_m\bigr)
  \Bigr|.
\end{equation}
From~(\ref{eq:weighted-coverage}) and Def.~\ref{def:coverage}, maximizing $F(L)$ subject to $|L| = b$ is equivalent to a weighted maximum $k$-coverage problem with $M$ groups over the family $\{S_i\}_{i=1}^N$, where all elements $e \in E^{(m)}$ share a common weight $\alpha_m$. The resulting objective is a non-negative, normalized, monotone, submodular set function
(see Appendix~\ref{app:submodular}).Therefore, the greedy algorithm that iteratively selects the element with the largest marginal gain achieves a $(1 - 1/e)$-approximation to the optimal solution~\citep{vazirani2001approximation}. A full proof is provided in Appendix~\ref{sec:Detailed_Theoretical_Analysis}.

\subsection{Embedding alignment}

\textbf{Bandwidth selection for the RBF kernel in \maxherding.}  
Coverage-based selection methods rely on hyper-parameters that control similarity range and partition granularity, which become especially problematic in Multiple Embedding settings where embeddings $\{\mathcal{E}^{(m)}\}_{m=1}^{M}$ originate from heterogeneous backbones with incompatible geometric scales.
When integrating \maxherding\ into \method, the relevant parameter is the RBF bandwidth $\sigma$ 
where
\[
\kappa_{RBF}(\mathbf{x},\mathbf{x}')=\exp\!\bigl(-\lVert\mathbf{x}-\mathbf{x}'\rVert^{2}/(2\sigma^{2})\bigr).
\]
Following the widely adopted \emph{median heuristic} \cite{garreau2018largesampleanalysismedian}, we set $\sigma$ to the median cosine distance among all exemplars in the current episode. This choice aligns kernel similarities with the intrinsic geometry and sparsity of each embedding $\mathcal{E}^{(m)}$, ensuring consistent behavior across embeddings, as validated in Section~\ref{sec:ablation}.

\textbf{Weighting each embedding.}
We now discuss the estimation of the vector of weights $\{\alpha_m\}$ defined in (\ref{eq:weighted-coverage}).

First, we recall the definition of the $k$-NN density estimation. Once again,  let $\mathcal{M}_c = \{  x\in X \mid y(x)=c\}$. For any $x\in\mathcal{M}_c$,  let $\mathcal{N}^{(m)}_K(\mathbf{x})$ denote the set of its $k$ nearest neighbors in $\mathcal{M}_c \setminus {\mathbf{x}}$ in embedding $\mathcal{E}^{(m)}$. Let $\rho^{(m)}_k(x) $ denote the mean distance from $x$ to set $\mathcal{N}_K(\mathbf{x})$. In embedding $m$, the kNN density estimate at $x$ is defined as follows:
\begin{equation}\label{eq:knn_density}
\widehat{f}^{(m)}_k(x) = \frac{k}{\rho^{(m)}_k(x)}
\end{equation}
For embedding $m$, we now define its weight as follows:
\vspace{-4pt}
\begin{equation}
\label{eq:weight}
\alpha_m=\frac{\operatorname{median}(\widehat f^{(m)}_k(x))}{\operatorname{median}(\widehat f^{(m)}_1(x))}
\end{equation}
The reasoning behind this definition is as follows: if two point clouds differ only by a scale factor, the distribution of $\alpha$ remains unchanged, resulting in $\alpha_1 = \alpha_2$. In practice, however, the supervised embedding $\mathcal{E}_{\text{Supervised}}$ tends to exhibit \emph{micro-clusters} - tightly grouped, nearly identical samples within a class - more so than the self-supervised embedding $\mathcal{E}_{\text{self-supervised}}$. These local geometric effects reduce the nearest-neighbor distance $\rho_1$ without a proportional
reduction in $\rho_k$, thereby increasing the ratio $\rho_k / \rho_1$.
This effect is significantly weaker in the self-supervised embedding
$\mathcal{E}_{\text{self-supervised}}$, whose geometry is more uniform.
As a result, we typically observe:
\begin{equation}
\label{eq:beta}
\beta \;=\; \frac{\alpha_{\text{Supervised}}}{\alpha_{\text{self-supervised}}} \;>\; 1 .
\end{equation}
Our greedy algorithm maximizes the \emph{weighted coverage score} defined in (\ref{eq:weighted-coverage}). Because the algorithm also enforces \emph{diversity} through disjoint $k$-NN balls, dense supervised balls contain far fewer candidate edges than large self-supervised balls. Multiplying the supervised edge count by $\beta$ thus equalizes the edge mass that each selected point can cover, ensuring that the sampler does not over‑represent the sparse self-supervised space and achieves a balanced, diverse subset across both embeddings.

\subsection{The \method algorithm}

Pseudo-code for the \maxherding-variant  of \method is provided in Alg.~\ref{alg:MaxHerding-multiple-embedding} (see Appendix for the \probcover-variant).

\begin{algorithm}[ht]
\caption{\method \maxherding}
\label{alg:MaxHerding-multiple-embedding}
\textbf{Input:} Dataset $C = \{x_1, \dots, x_n\}$, kernels $k_m$, weights $\alpha_m$, buffer $\mathcal{m}$, budget $b$.

\textbf{Output:} Updated memory buffer $\mathcal{M}$.

\begin{algorithmic}[1]

\State $k(x, x') \gets \sum_{m=1}^M \alpha_m k_m(z^{(m)}_x, z^{(m)}_{x'})$

\For{$t = 1$ to $b$}  \Comment{Greedy MaxHerding selection}
    \For{each $x_j \in C \setminus S$}
    \State \parbox[t]{0.8\linewidth}{%
    $\begin{aligned}
                &G(x_j) \gets \frac{1}{n} \sum_{i=1}^n \max(k(x_i, x_j) - c_i, 0)
            \end{aligned}$}
    \EndFor
    \State $x_t \gets \arg\max_{x_j \in C \setminus S} G(x_j)$; \quad $S \gets S \cup \{x_t\}$

    \For{$i = 1$ to $n$}
        \State $c_i \gets \max(c_i, k(x_i, x_t))$
    \EndFor
\EndFor

\State $\mathcal{M} \gets \mathcal{M} \cup S$
\end{algorithmic}

\Return $\mathcal{M}$
\end{algorithm}
\vspace{-6pt}

\section{Theoretical analysis}
\label{sec:theory}

Appendix~\ref{sec:Detailed_Theoretical_Analysis}  provides a theoretical justification for sampling from a
mixture of Supervised (SL) and Self-Supervised (SSL) embeddings. The key premise
is that SL representations can become \emph{episode-specialized}, concentrating
variation in class-discriminative directions and compressing directions that are
currently irrelevant, whereas SSL representations tend to preserve a broader set
of non-label factors and induce a more isotropic geometry. While SL specialization is clearly beneficial, encoding task-relevant structure and domain knowledge, we show that the broader, less discrimination-oriented geometry of SSL representations can improve robustness to domain shift and to the emergence of new classes.

\textbf{Modeling SL/SSL as class-conditional perturbations.}
We model each class-conditional distribution in a reference feature space
$\R^n$ by a Gaussian $\cN(\mu,\Sigma)$. We investigate a reference class whose conditional distribution is $G_0=\cN(\mu_0,\Sigma)$, fixing $\mu_0=0$ w.l.o.g.
Each embedding used for buffer sampling induces a \emph{modified} class-conditional
distribution for the reference class. With SSL embedding, we model the modified distribution by one of two isotropic proxies:
$G^{(1)}_{\mathrm{SSL}}=\cN(0,\sigma\Sigma)$ or $G^{(2)}_{\mathrm{SSL}}=\cN(0,\sigma I_n)$, which are justified by the presumed uniformity of SSL embeddings. The distribution over the SL embedding is modeled by an anisotropic proxy
$G_{\mathrm{SL}}=\cN(0,\Sigma^{1/2}D\Sigma^{1/2})$ with
$D=\diag(\alpha,\ldots,\alpha,\beta,\ldots,\beta)$, $\alpha>\beta>0$, where $\alpha$
acts on $m$ discriminative directions and $\beta$ compresses the remaining $n-m$
directions. To factor out irrelevant global-scale effects, we enforce equal global compression between the SL and SSL embeddings by matching the volumes of their covariance ellipsoids, i.e., their determinants.

\textbf{Anisotropy increases KL under equal volume.} Under equal-volume normalization, the anisotropic SL perturbation yields a larger class-conditional distortion than the SSL proxies as measured by $\KL(G_0\|\cdot)$. In particular,
\vspace{-6pt}
\begin{equation*}
\label{eq:KL}
\begin{split}
\KL\!\left(G_0 \,\middle\|\, G_{\mathrm{SL}}\right)
&\ge
\KL\!\left(G_0 \,\middle\|\, G^{(1)}_{\mathrm{SSL}}\right), \\
\KL\!\left(G_0 \,\middle\|\, G_{\mathrm{SL}}\right)
&\ge
\KL\!\left(G_0 \,\middle\|\, G^{(2)}_{\mathrm{SSL}}\right),
\quad\beta < \beta_0.
\end{split}   
\end{equation*}
for some $\beta_0$. Moreover, in the highly anisotropic regime $\beta \to 0$, the resulting KL gap can grow arbitrarily large.

\textbf{A domain-adaptation view of class-incremental training.}
We now formulate the episode-to-episode shift as a domain adaptation problem. A central quantity in this framework is the \emph{train--test risk gap}, defined as the difference between the empirical risk on the training set and the risk on a test set; a larger gap indicates poorer generalization. Accordingly, our objective is to minimize this gap.

In our setting, only class $Y = 1$ is carried over from the previous episode and is therefore represented by a limited buffer of stored examples, while all remaining classes are represented by freshly sampled data. As a result, the training and test distributions coincide for all but the buffered class: $P_{\mathrm{tr}}(X \mid Y = y) = P_{\mathrm{te}}(X \mid Y = y)\quad \forall y \neq 1$, whereas $P_{\mathrm{tr}}(X \mid Y = 1) \neq P_{\mathrm{te}}(X \mid Y = 1)$. The following result characterizes the effect of this class-conditional shift on the risk gap for any classifier $h$:
\begin{equation*}
\label{eq:risk}
RiskGap \!\le \!\KL\big(P_{\mathrm{te}}(X\!\mid\! Y=1)\|P_{\mathrm{tr}}(X\!\mid\! Y=1)\big).
\end{equation*}
\begin{table*}[!b]
\centering
\fontsize{9}{11}
\rmfamily
\setlength{\tabcolsep}{1mm}\caption{Average Accuracy Across All Tasks (AAA) on CIFAR-100 \textbf{ER ACE STAR}. For each \buffer, the best AAA is highlighted in bold.}
\label{tab:aaa_cifar-100_er_ace_star}
\rowcolors{2}{tableShade}{white}%
\begin{tabular}{lcccccc}
\toprule
 & \multicolumn{1}{c}{Random} & \multicolumn{1}{c}{\probcover} & \multicolumn{2}{c}{\maxherding} & \multicolumn{1}{c}{Herding} & \multicolumn{1}{c}{TEAL}  \\
\cmidrule(lr){2-2} \cmidrule(lr){3-3} \cmidrule(lr){4-5} \cmidrule(lr){6-6} \cmidrule(lr){7-7} 
\buffer & Supervised & Supervised & Supervised & MERS & Supervised & Supervised  \\
\midrule
100 & 41.71 \std{0.18} & 47.98 \std{0.13} & 49.32 \std{0.13} & \textbf{50.96} \std{0.19} & 41.98 \std{0.20} & 48.41 \std{0.33}  \\
300 & 50.25 \std{0.35} & 57.10 \std{0.25} & 57.11 \std{0.14} & \textbf{58.96} \std{0.21} & 51.79 \std{0.27} & 56.56 \std{0.23}  \\
500 & 54.14 \std{0.32} & 59.88 \std{0.17} & 60.32 \std{0.27} & \textbf{61.64} \std{0.16} & 55.22 \std{0.30} & 60.06 \std{0.15}  \\
1000 & 60.03 \std{0.27} & 64.17 \std{0.30} & 63.92 \std{0.12} & \textbf{65.54} \std{0.29} & 61.50 \std{0.08} & 64.05 \std{0.28}\\
2000 & 65.12 \std{0.22} & 67.82 \std{0.20} & 68.08 \std{0.34} & \textbf{69.23} \std{0.27} & 65.38 \std{0.19} & 67.57 \std{0.07}  \\

\bottomrule
\end{tabular}
\label{tab:aaa_er_ace_cifar100}
\end{table*}
In other words, the train-test risk gap is bounded by the KL-divergence between the class conditional distribution of the buffered class in the train and test distributions.

\textbf{Implication for SSL vs.\ SL sampling.}
Together, \eqref{eq:KL} and \eqref{eq:risk} give us our final result: under equal-volume normalization, sampling in
an SSL geometry leads to a tighter bound on the
train-test risk gap than sampling in the anisotropic SL geometry, especially
in the small-$\beta$ regime, and therefore implies better generalization.

\section{Methodology}
\label{sec:empirical_methodology}

In our empirical evaluation, \method is evaluated while enhancing 3 distinct experience replay continual learning algorithms, detailed in Section~\ref{sec:ER-methods}. We report results in comparison with several common exemplar selection strategies, which are described in Section ~\ref{sec:selection-methods}. 
Section~\ref{sec:ssl-methods} describes the 3 alternative SSL methods used for evaluation. Section~\ref{sec:dataset} describes the two datasets used in our evaluation, following customary practice in the evaluation of CIL methods.  Evaluation metrics are described in Section~\ref{sec:metrics}. All experiments use a class-balanced replay buffer.

\subsection{Continual Learning Algorithms}
\label{sec:ER-methods}
We evaluate \method with three rehearsal-based continual learning baselines: \textbf{ER}~\cite{rolnick2019experience}, which replays buffered past examples; \textbf{ER-ACE}~\cite{caccia2021new}, which decouples losses for new and replayed data; and \textbf{ER-ACE-STAR}~\cite{eskandar2025starstabilityinducingweightperturbation}, which augments ER-ACE with an adaptive, method-agnostic replay reweighting strategy.

\subsection{Baseline selection strategies}
\label{sec:selection-methods}

We compare against representative exemplar selection strategies: (i) \textbf{Random} selects exemplars uniformly at random from each class; (ii) \textbf{Herding}~\cite{welling2009herding,rebuffi2017icarl} selects samples to approximate the class mean in feature space; (iii) \textbf{Rainbow Memory}~\cite{bang2021rainbow} balances multiple criteria such as diversity and uncertainty; (iv) \textbf{TEAL}~\cite{shaul2024teal} clusters samples and selects representative exemplars; 
(v) \textbf{\probcover}~\cite{bae2024generalized} selects points based on class-coverage using the \probcover approach; (vi) \textbf{\maxherding}~\cite{bae2024generalized} selects points based on class-coverage using the \maxherding approach.


\subsection{Self-supervised learning baselines}
\label{sec:ssl-methods}
We evaluate three SOTA self-supervised representation learning methods: \textbf{SimCLR}~\cite{chen2020simple}, a contrastive approach maximizing agreement between augmented views; \textbf{VICReg}~\cite{bardes2021vicreg}, which enforces invariance with variance and covariance regularization without negatives; and \textbf{DINOv2}~\cite{9709990,Oquab2023DINOv2}, which learn transferable representations via self-distillation from an EMA teacher. VICReg and SimCLR are trained from scratch at each episode using only the current episode’s data. DINOv2 embeddings are extracted from a frozen model (see Appendix~\ref{subsec:ssl_training}).
\subsection{Datasets} 
\label{sec:dataset}

We evaluate on two standard CIL benchmarks: \textbf{Split CIFAR-100}~\citep{chaudhry2019continual,rebuffi2017icarl}, which divides CIFAR-100 into 10 episodes of 10 classes each (500 training and 100 test images per class), and \textbf{Split TinyImageNet}~\citep{le2015tiny}, which splits TinyImageNet into 10 episodes of 20 classes each (500 training and 50 test images per class).

\subsection{Evaluation Metrics in CIL}
\label{sec:metrics}

We report five standard CIL metrics: 
%
\begin{itemize}[leftmargin=*,itemsep=0.em, topsep=0.em, parsep=0pt, partopsep=0pt]
\item \emph{Average Accuracy}:  $AA_t$ is the mean accuracy over all tasks learned up to task $t$.
\item \emph{Final Average Accuracy}: $FAA = AA_T$.
\item \emph{Anytime Average Accuracy}: $AAA = \frac{1}{T}\sum_{t=1}^{T} AA_t$.
\item \emph{Forgetting}: $F = \frac{1}{T-1}\sum_{i=1}^{T-1}\left(\max_{j \le T} A_{i,j} - A_{i,T}\right)$, where $A_{i,j}$ - accuracy on task $i$ after learning task $j$.
\item \emph{Stability}: accuracy on previously learned tasks. 
$S = \frac{1}{T-1} \sum_{t=2}^{T} \frac{1}{t-1} \sum_{i=1}^{t-1} A_{i,t},$
where $A_{i,t}$ denotes the accuracy on task $i$ after learning task $t$.
\end{itemize}



\begin{table*}[!t]
\centering
\fontsize{9}{11}
\rmfamily
\setlength{\tabcolsep}{1mm}\caption{Average Accuracy Across All Tasks (AAA) on CIFAR-100 \textbf{ER ACE}. For each \buffer, the best AAA is highlighted in bold.}
\label{tab:aaa_cifar-100_er_ace}
\rowcolors{2}{tableShade}{white}%
\begin{tabular}{lcccccc}
\toprule
 & \multicolumn{1}{c}{Random} & \multicolumn{1}{c}{\probcover} & \multicolumn{2}{c}{\maxherding} & \multicolumn{1}{c}{Herding} & \multicolumn{1}{c}{TEAL}  \\
\cmidrule(lr){2-2} \cmidrule(lr){3-3} \cmidrule(lr){4-5} \cmidrule(lr){6-6} \cmidrule(lr){7-7} 
\buffer & Supervised & Supervised & Supervised & MERS & Supervised & Supervised  \\
\midrule
100 & 41.31 \std{0.30} & 45.98 \std{0.35} & 47.04 \std{0.21} & \textbf{48.32} \std{0.15} & 40.77 \std{0.19} & 42.91\std{0.13}  \\
300 & 49.90 \std{0.28} & 54.03 \std{0.21} & 54.19 \std{0.29} & \textbf{55.68} \std{0.39} & 48.70 \std{0.29} & 53.13 \std{0.19}  \\
500 & 53.72 \std{0.20} & 57.17 \std{0.07} & 58.01 \std{0.11} & \textbf{58.99} \std{0.19} & 52.94 \std{0.10} & 56.80 \std{0.14}\\
1000 & 58.88 \std{0.09} & 61.72 \std{0.24} & 62.35 \std{0.30} & \textbf{63.52} \std{0.23} & 58.80 \std{0.22} & 61.10 \std{0.28}  \\
2000 & 64.21 \std{0.35} & 66.47 \std{0.25} & 66.35 \std{0.25} & \textbf{66.96} \std{0.10} & 64.53 \std{0.14} & 65.66 \std{0.24} \\

\bottomrule
\end{tabular}
\end{table*}

\section{Empirical Results}
\subsection{Main results}\label{subsec:main_results}
In our empirical evaluation, we assess two variants of \method that rely on two related coverage-based methods, denoted \method \probcover and \method \maxherding, as described above. 
%
To assess robustness to memory constraints, we varied the capacity of the replay-buffer, from 100 to 1000 on the Split CIFAR-100 benchmark; the resulting FAA is reported in Fig.~\ref{fig:mers_max_herding_er_ace_cifar100_faa}, while AAA is reported in Tables~\ref{tab:aaa_cifar-100_er_ace_star}-~\ref{tab:aaa_cifar-100_er_ace}. On Split TinyImageNet benchmark, we consider buffer sizes ranging from 200 to 6000, with FAA results shown in Fig~\ref{fig:max_herding_tinyimg}. The complete results, including AAA, are provided in Appendix~\ref{sec:appendix} (see Fig.~\ref{fig:mers_max_herding_er_ace_cifar100_aaa}).

\begin{figure}[H]
    \centering
\begin{subfigure}[H]{\columnwidth}
        \centering
        \includegraphics[width=\linewidth]{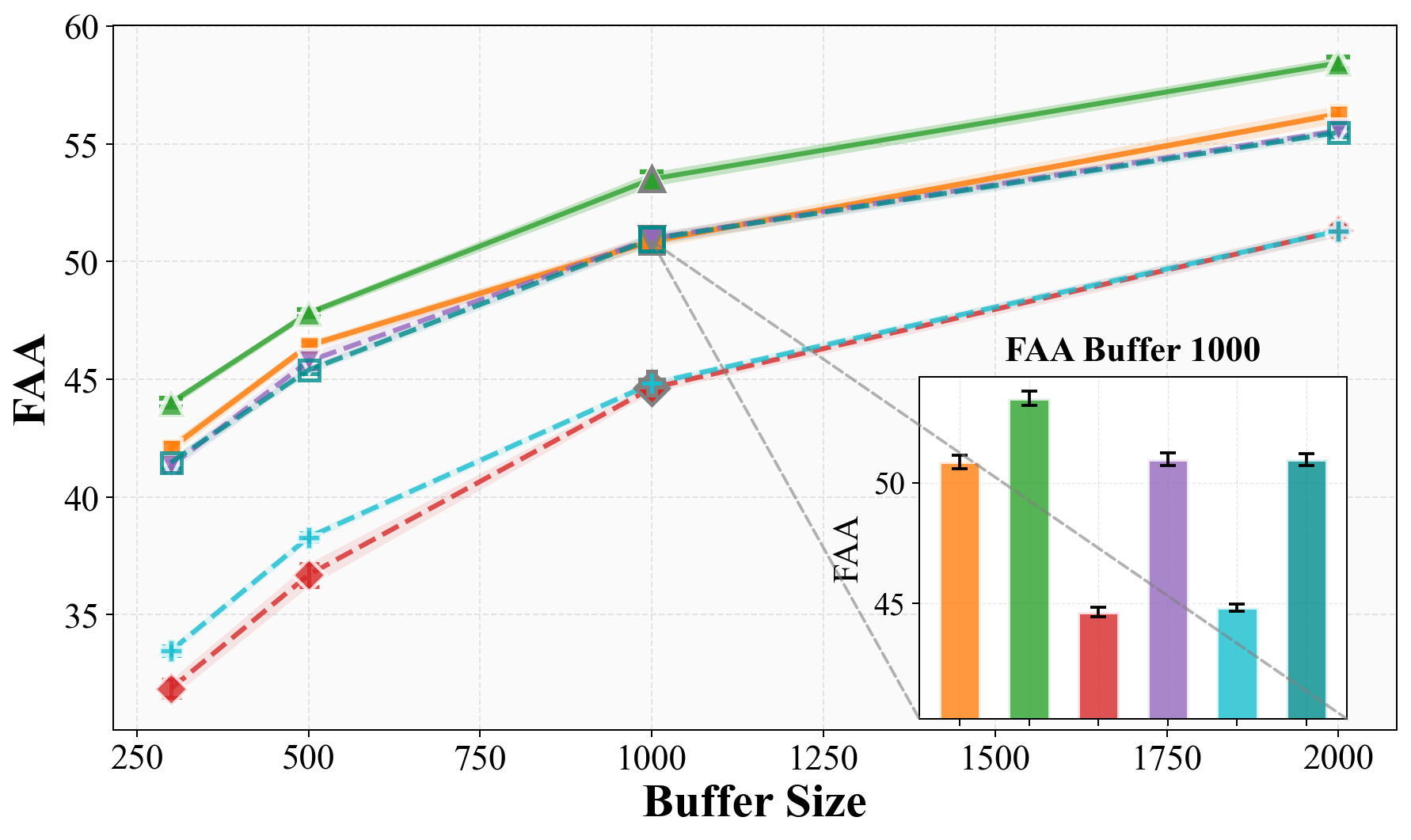}
        \caption{ER-ACE-STAR}
    \end{subfigure}
    \hfill
    \begin{subfigure}[H]{\columnwidth}
        \centering
        \includegraphics[width=\linewidth]{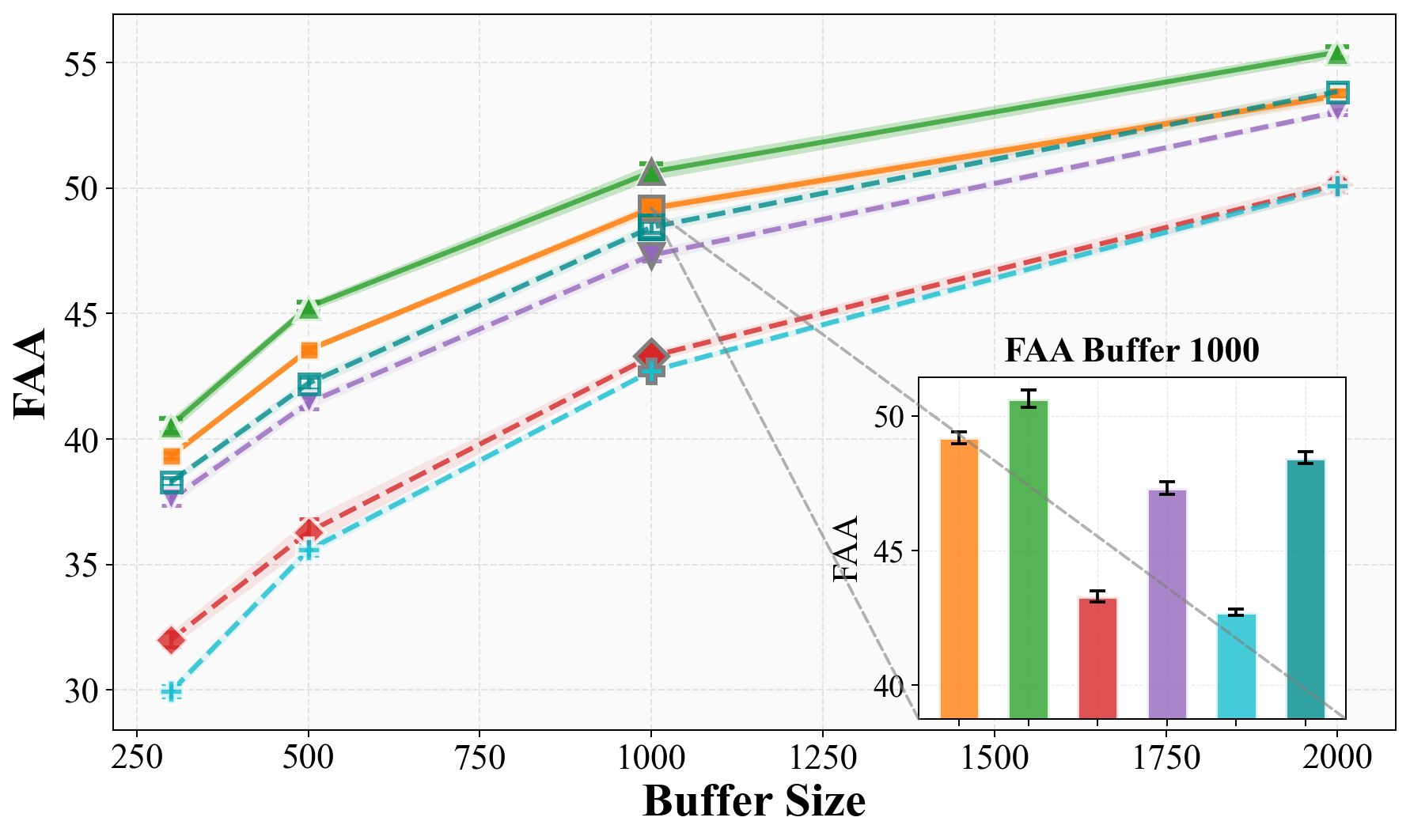}
        \caption{ER-ACE}
    \end{subfigure}
    \begin{subfigure}
    [H]{\columnwidth}
        \centering    \includegraphics[width=\linewidth]{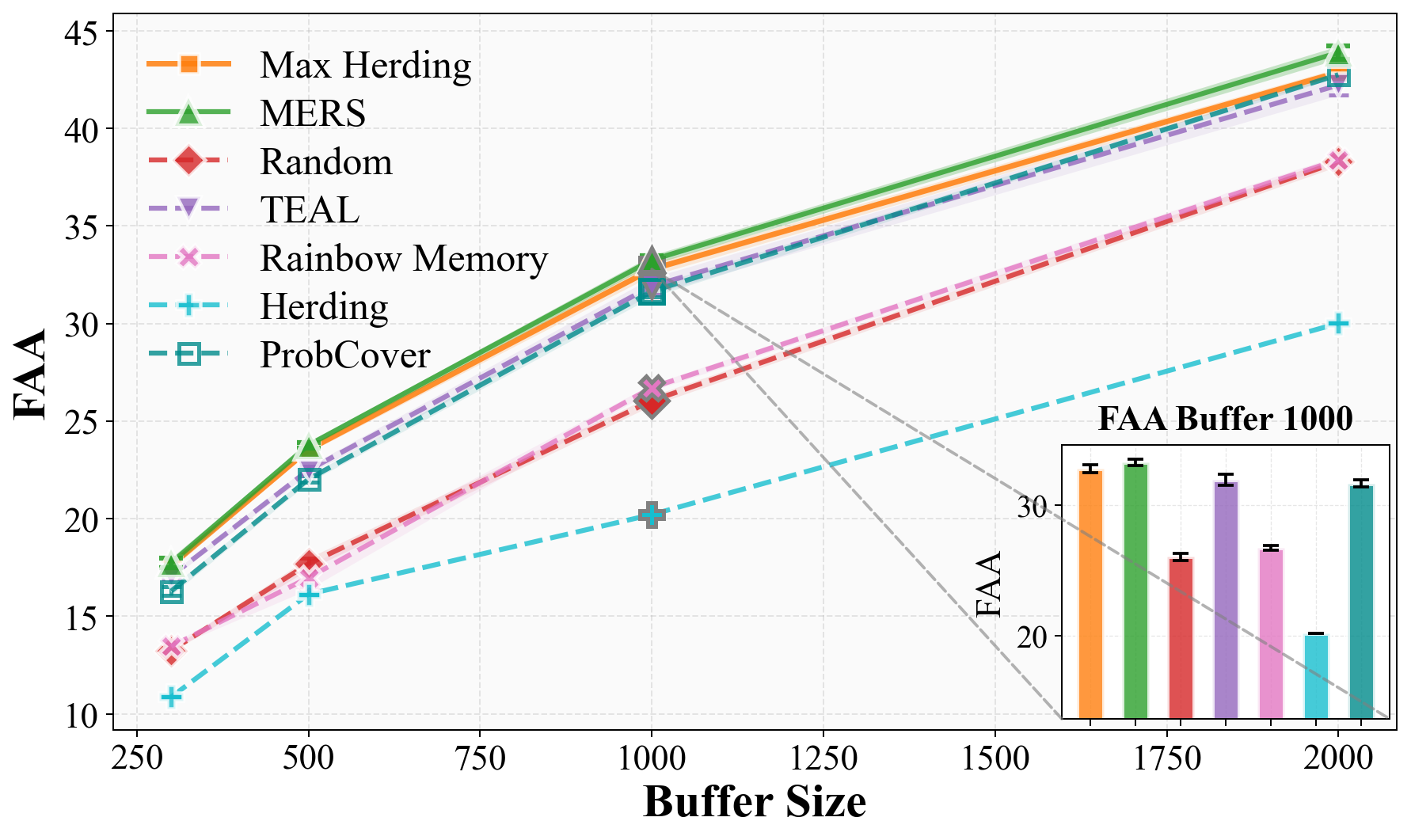}
        \caption{ER}
    \end{subfigure}
    \vspace{0.5em}

\caption{FAA as a function of memory size $|M|$ on Split CIFAR-100 for three continual learning algorithms, described in Section~\ref{sec:ER-methods}. Results with \method are compared against alternative selection strategies, described in Section~\ref{sec:selection-methods}. The selection-strategy legend is shown in panel (c).}
    \label{fig:mers_max_herding_er_ace_cifar100_faa}
\end{figure}

\subsection{Pretrained vs. Episodic Embeddings}
Following the same protocol as outlined above, results when using different SSL embeddings (see Section~\ref{sec:ssl-methods}) are presented in Fig.~\ref{fig:MERS-embeddings}, with complete FAA and AAA tables reported in Appendix~\ref{sec:appendix}.
\begin{figure}[t]
    \centering
    \begin{subfigure}[t]{0.48\linewidth}
        \centering
    \includegraphics[width=\linewidth]{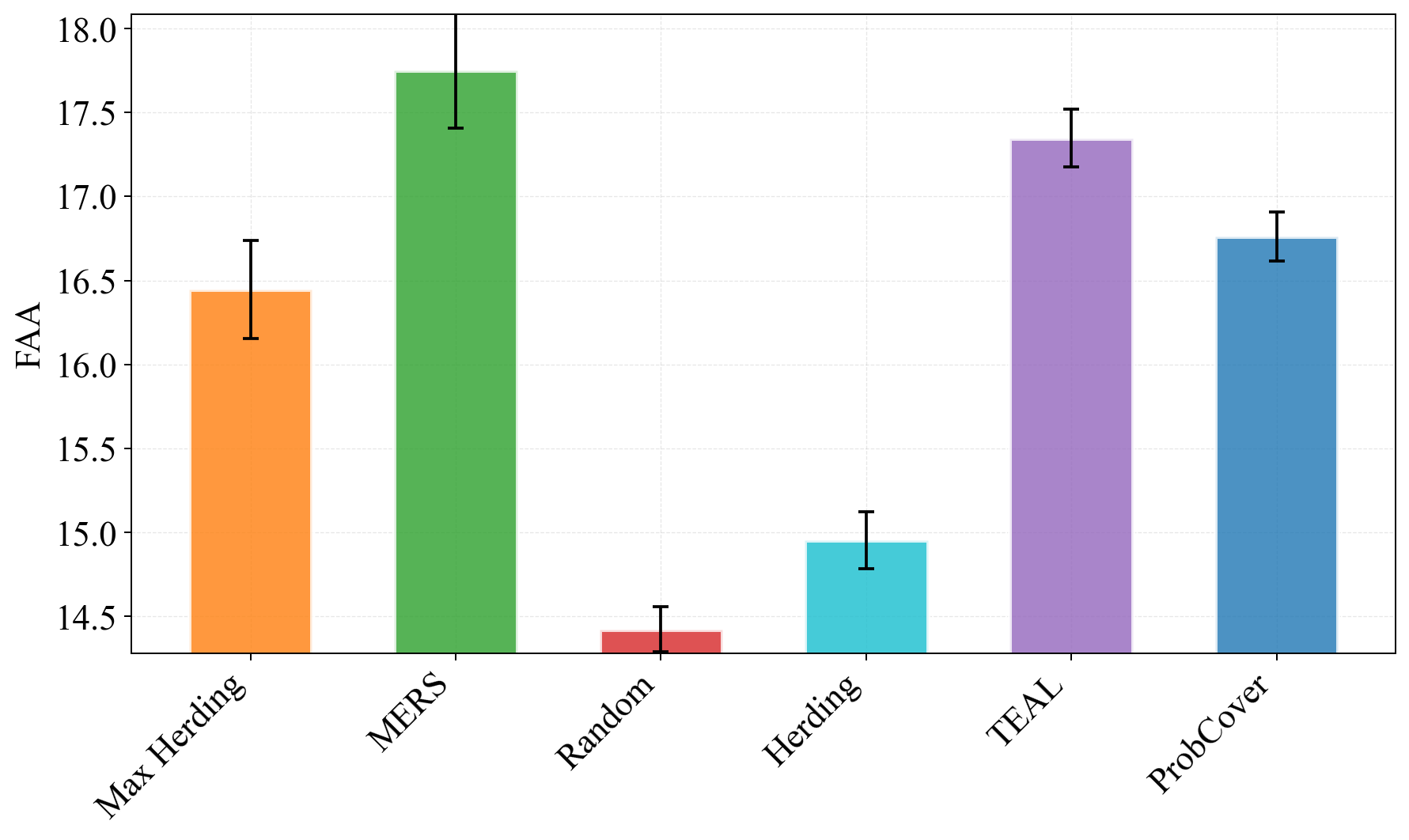}
        \caption{FAA}
    \end{subfigure}
    \hfill
    \begin{subfigure}[t]{0.48\linewidth}
        \centering
        \includegraphics[width=\linewidth]{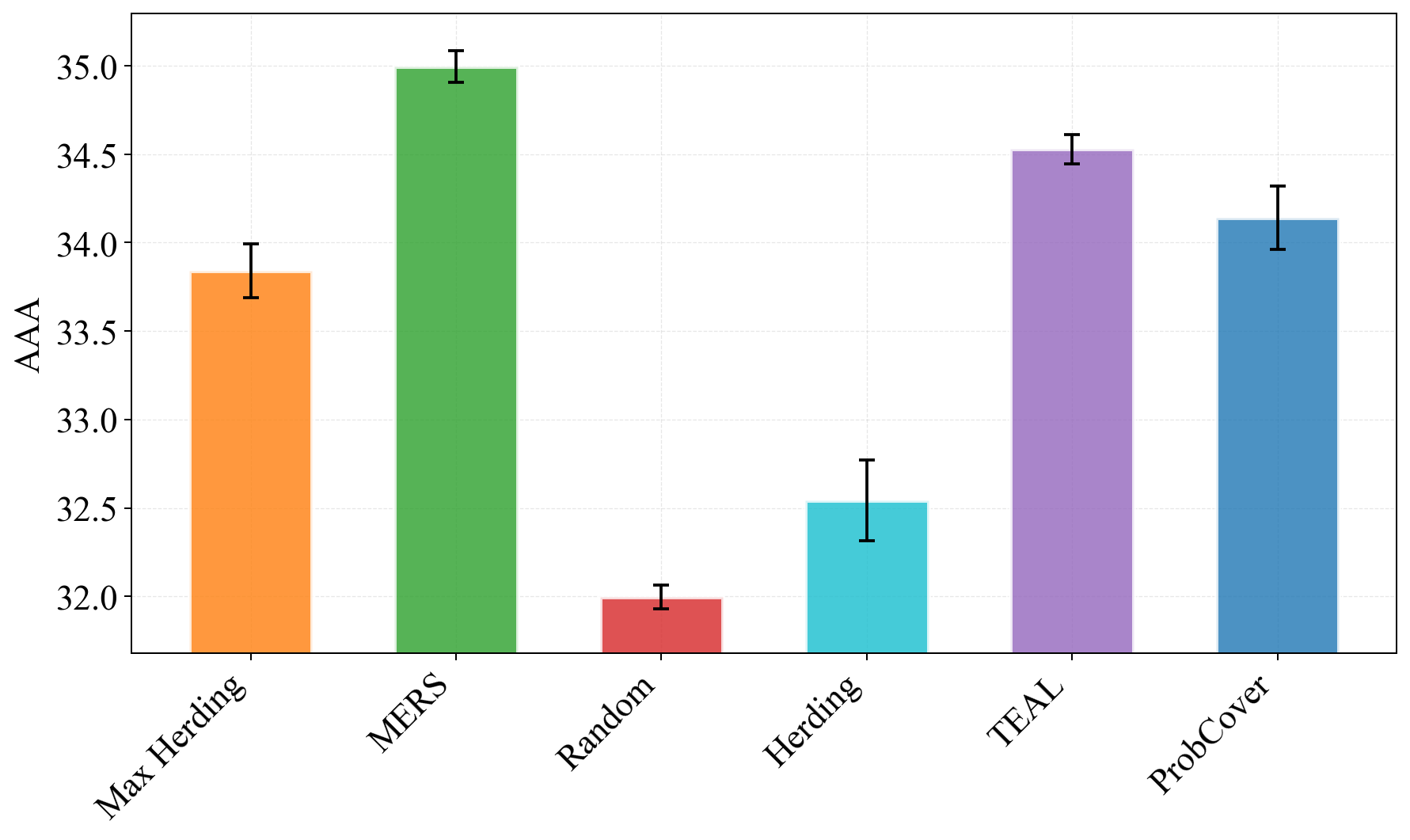}
        \caption{AAA}
    \end{subfigure}
\caption{FAA (left) and AAA (right) on  Split TinyImageNet for \textbf{ER-ACE} with buffer size $|\mathcal{M}| = 1000$, using \method, compared against alternative selection strategies.}
\label{fig:max_herding_tinyimg}
\end{figure}

\begin{figure}[ht]
    \centering
    \includegraphics[width=\linewidth]{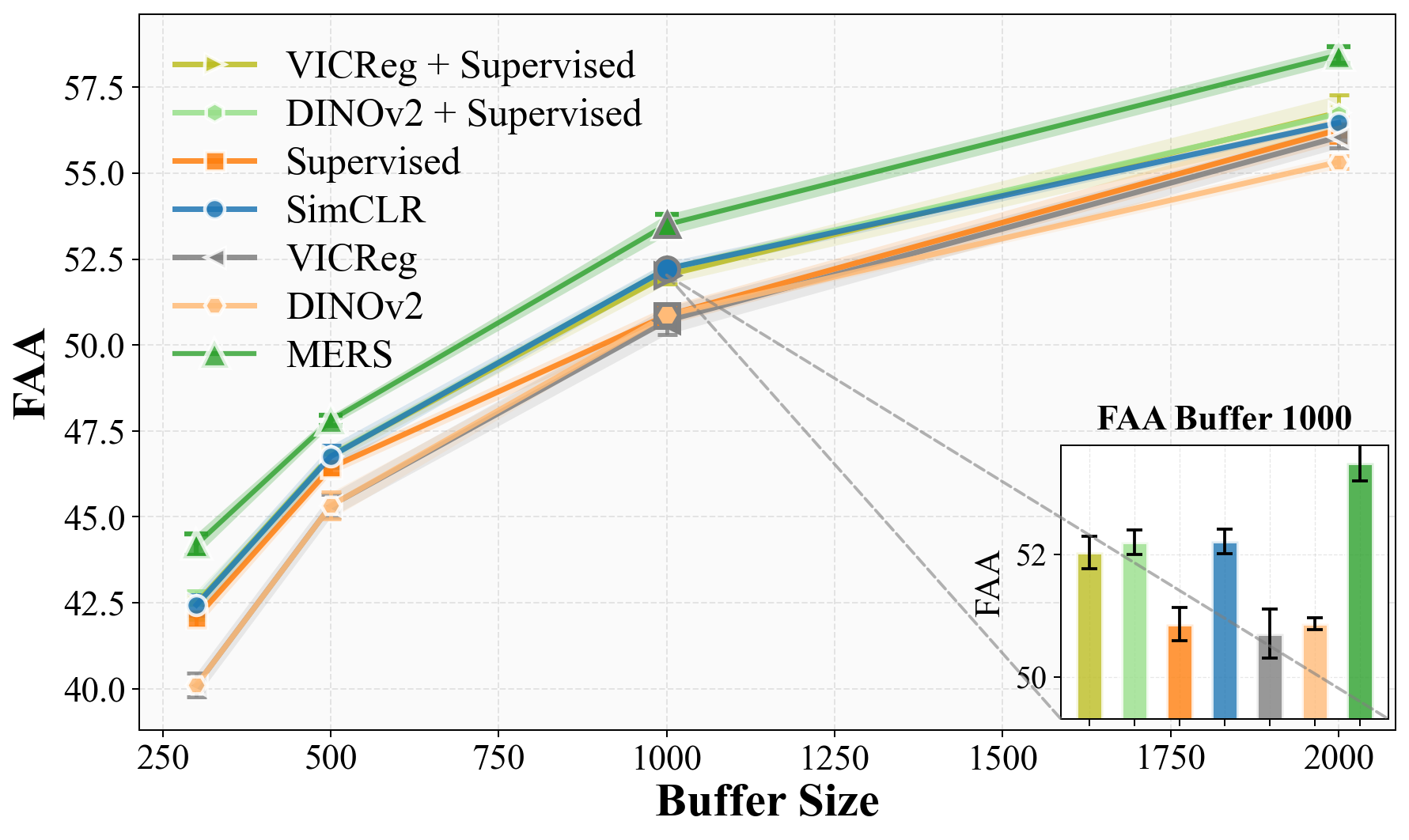}
    \caption{FAA of \method with ER-ACE-STAR on Split CIFAR-100 using different embeddings: SimCLR, VICReg and DINOv2}
    \label{fig:MERS-embeddings}
\end{figure}

\subsection{Selection stability and forgetting}
\label{sec:selection_stability}

We analyze selection stability and forgetting for Max-Herding with a supervised embedding, Max-Herding with SimCLR embedding, and the integrated \method approach. Results are reported in Fig.~\ref{fig:stability}, with complete stability and forgetting statistics provided in Appendix~\ref{sec:appendix_selection_stability}.

\begin{figure}[ht]
    \centering
    \begin{subfigure}[h]{0.48\textwidth}
        \centering
        \includegraphics[width=\linewidth]{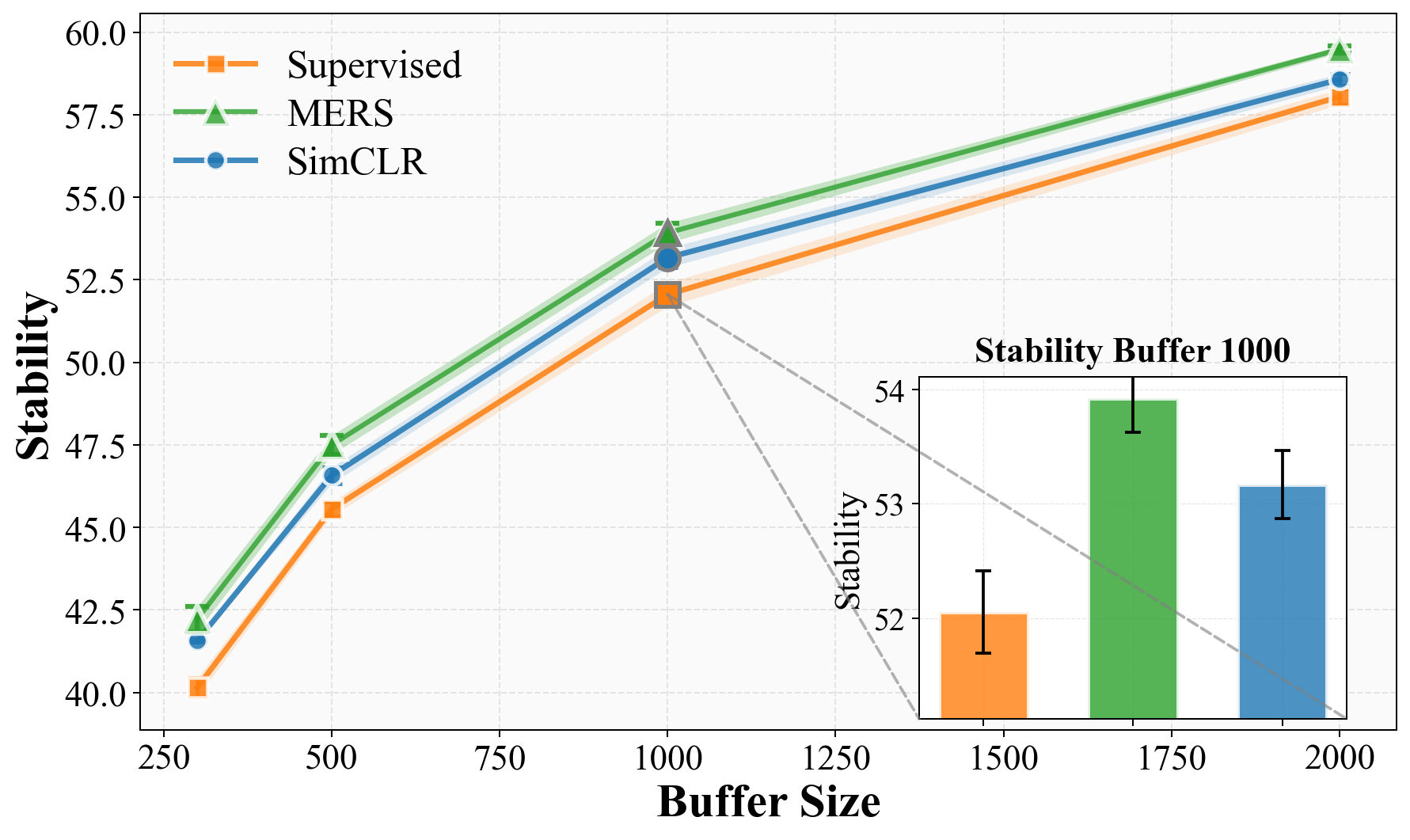}
        \caption{Stability}
    \end{subfigure}
    \hfill
    \begin{subfigure}[h]{0.48\textwidth}
        \centering
        \includegraphics[width=\linewidth]{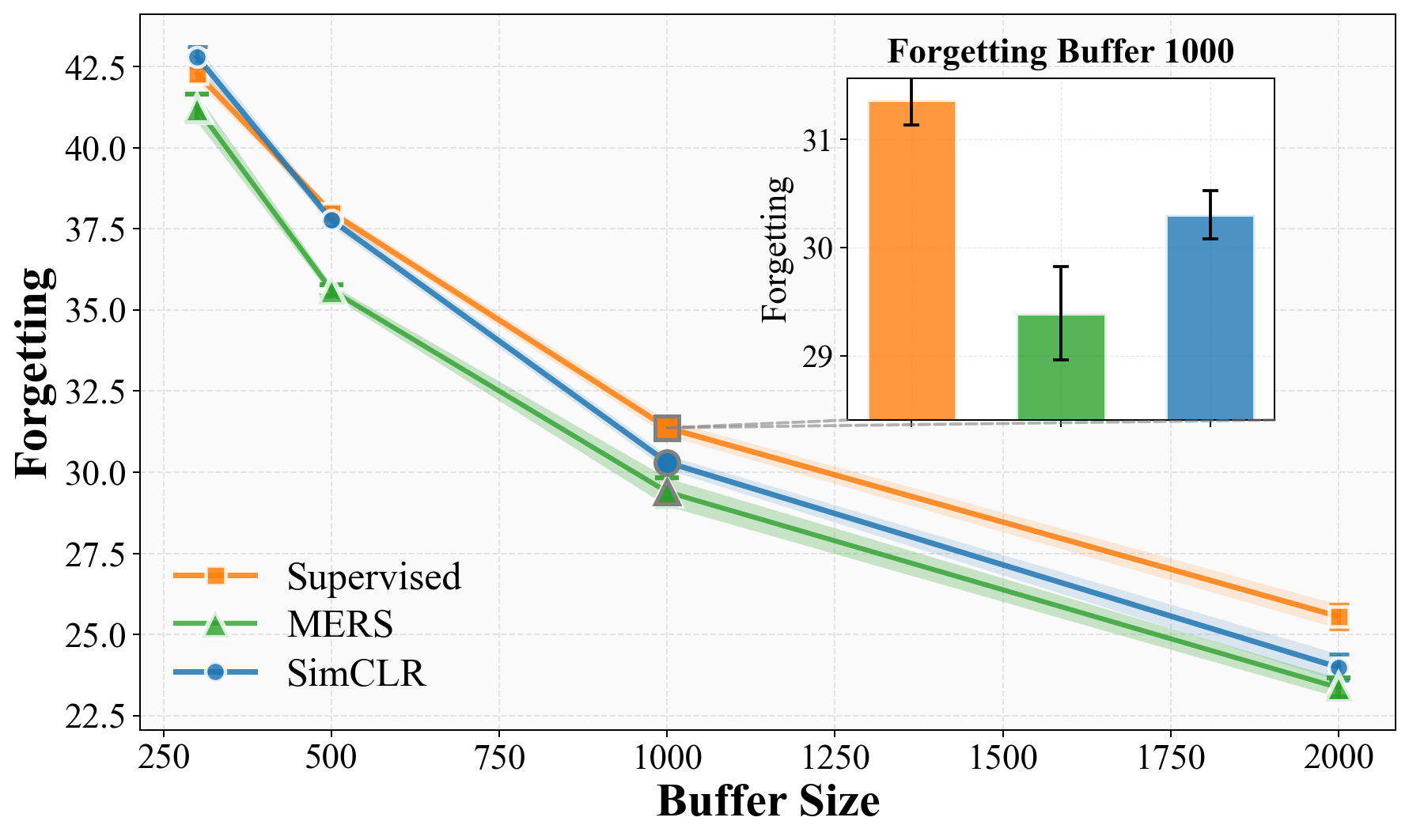}
        \caption{Forgetting}
    \end{subfigure}

    \caption{Stability and forgetting of ER-ACE-STAR with \method as a function of $|M|$ on Split CIFAR-100.}
    \label{fig:stability}
\end{figure}

We observe that Max-Herding based on SimCLR embeddings consistently yields higher stability and lower forgetting compared to its supervised counterpart. Furthermore, \method, which integrates supervised and self-supervised embedding spaces, achieves the most stable selection behavior overall, outperforming both Max-Herding variants across all evaluated settings.


\subsection{Discussion}

Across all buffer sizes, replay methods, and datasets, \method\ achieves the strongest performance. The integrated variant consistently matches or outperforms its constrained counterparts, with the largest gains in the low-budget regime (up to 1000 exemplars). While the gap narrows as the buffer grows, the integrated \method\ remains top-ranked, often tying for best. Overall, \method\ outperforms either embedding alone, with integration yielding the greatest benefit under tight memory constraints. Notably, these gains coincide with increased selection stability and reduced forgetting, 
suggesting that embedding integration plays a key role in the observed performance improvements.

The empirical findings reported in Section~\ref{sec:selection_stability} are consistent with our theoretical analysis in Section~\ref{sec:theory}. The improved stability and reduced forgetting observed with SimCLR and the integrated \method approach reflect a reduced distributional drift between stored exemplars and data encountered in later episodes.


\section{Ablation study}
\label{sec:ablation}

We conducted targeted ablations to identify which design choices of our \emph{\method} are most critical:

\textbf{RBF bandwidth \texorpdfstring{$\boldsymbol{\sigma}$}{sigma}
 in \texorpdfstring{\maxherding}.}
We tested three settings for $\sigma$: (i) median cosine distances, (ii) $\sigma=1$, and (iii) median $k$-NN distances. On CIFAR-100, (i) and (iii) coincide, while the constant value reduces FAA by \textbf{$\approx$ 1\%} in the small-buffer regime (see Fig.~\ref{fig:sigma}). As (i) is dataset-agnostic and robust across budgets, we adopt it as the default.

\begin{figure}[H]
    \centering
    \includegraphics[width=\linewidth]{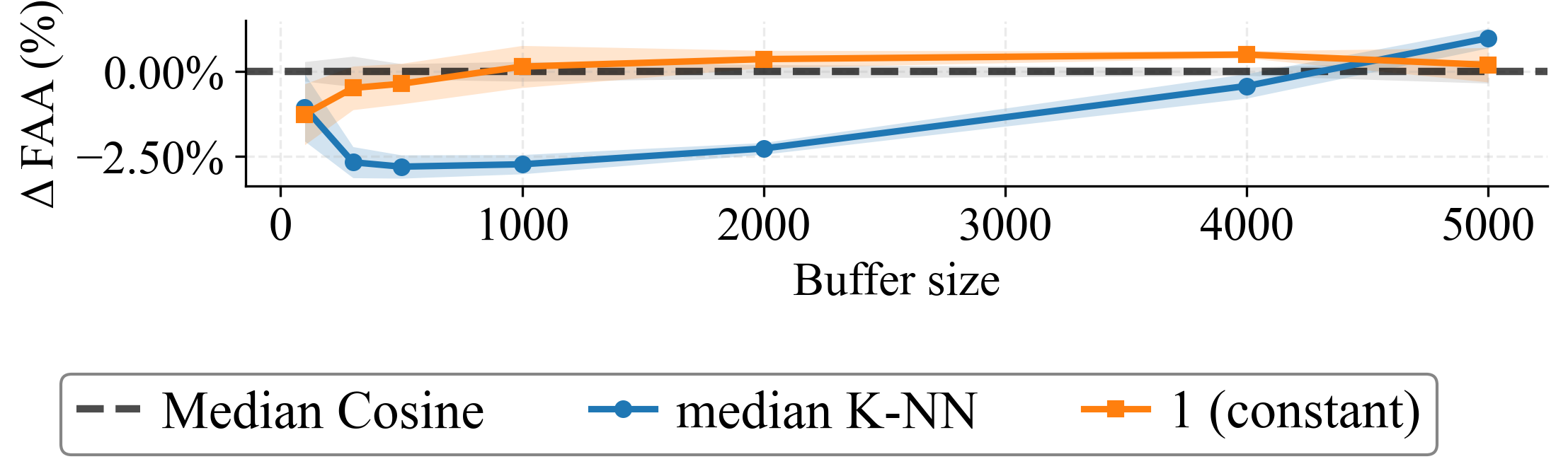}
      \caption{Improvements in FAA on CIFAR-100 as a function of \buffer while varying the RBF bandwidth $\sigma$ in \maxherding.}
    \label{fig:sigma}
\end{figure}

We conducted an ablation study on the embedding weight $\alpha$ using different density estimators. The results show a slight improvement when using the $\alpha$ defined in (\ref{eq:weight}), as reported in Appendix~\ref{app:ablation}.

We also conducted an ablation study using \maxherding with only SimCLR embeddings, and showed that \method achieves higher FAA and AAA accuracy, as reported in Fig.~\ref{fig:MERS-embeddings}

\section{Summary}
We present Multiple Embedding Replay Selection (\method), a plug-and-play sampler for replay-based continual learning that merges supervised and self-supervised feature spaces in a complementary manner. By building $k$-NN coverage graphs in each space, re-scaling them with density-aware weights, and greedily selecting exemplars that maximize a combined coverage score, \method  fills both class-discriminative and invariant regions of the data manifold. Across Split CIFAR-100 and Split TinyImageNet, it boosts final-average accuracy over single-embedding baselines when memory is tight, all without increasing the buffer size or changing model parameters. The method is plug-and-play, incurs only double selection-time overhead and self-supervised training. The approach opens avenues for dynamic, task-aware embedding integration in future work.
\subsection*{\textbf{Acknowledgments}}
This work was supported by a grant from the Gatsby Charitable Foundation and AFOSR award FA8655-24-1-7006.

{
    \bibliographystyle{icml2026}
    \small
    \bibliography{icml2026}
}

\appendix
\section{\probcover-based variant of \method\ }
\label{sec:probcover_variation}
\label{sec:appendix}

We study the integration of supervised and self-supervised embeddings within a coverage-based
selection strategy, namely \textbf{\probcover}~\cite{yehuda2022active}. \probcover is an active-learning
algorithm that formulates sample selection as a maximum coverage problem on a $\delta$-neighborhood
graph: given a small budget, it greedily selects points that maximize the number of previously
uncovered neighbors within a fixed radius $\delta$.

To adapt \probcover\ to the continual learning setting, we treat the current memory buffer as the
unlabeled pool and the exemplar set as the selected subset. We further extend the method to operate
over multiple embedding spaces, following the weighted multi-coverage formulation described in
Section~\ref{subsec:coverage_selection} of the main paper. The resulting procedure is summarized in
Algorithm~\ref{alg:Probcover-multiple-embedding}.

\subsection*{Selection of $\delta$ in \probcover}
\label{subsec:Selection_of_delta_probcover}

A critical hyperparameter in \probcover\ is the cover-ball radius $\delta$, which determines the
granularity of the induced neighborhood graph. Since different embeddings exhibit markedly
different geometric and density characteristics, using a fixed $\delta$ across embeddings is
suboptimal.

Following the nonparametric alignment strategy proposed in the main paper, we estimate $\delta$
from the data using class-conditional $k$-NN statistics. For a class $c$, let
$\mathcal{D}_c = \{x_i \mid y_i = c\}$. For each $x_i \in \mathcal{D}_c$, denote by
$\mathcal{N}_k(x_i)$ its $k$ nearest neighbors in $\mathcal{D}_c \setminus \{x_i\}$, and define
\[
r_i = \operatorname{median}_{x_j \in \mathcal{N}_k(x_i)} \|x_i - x_j\|.
\]
We then set
\[
\delta = \operatorname{median}_{x_i \in \mathcal{D}_c} r_i.
\]

The neighborhood size $k$ is chosen adaptively via the \emph{memory-aware ratio}
\[
k = \frac{|\mathcal{D}_c|}{\mathcal{M}_c},
\]
where $|\mathcal{D}_c|$ is the number of class-$c$ samples observed in the current episode and
$\mathcal{M}_c$ is the class-specific buffer capacity. This choice links the effective resolution of
the coverage graph to both the stream statistics and the available memory budget: larger buffers
yield finer partitions, while smaller buffers induce coarser coverage.


\begin{algorithm}[t]
\caption{\method \probcover}
\label{alg:Probcover-multiple-embedding}
\begin{algorithmic}[1]
\Require Dataset $C = \{x_1, \dots, x_n\}$, distances $D_m$, weights $\alpha_m$, buffer $\mathcal{M}$, budget $b$, ball-size $\delta$.
\Ensure Updated memory buffer $\mathcal{M}$.

\State $B^{(m)}_{\delta}(x_j) \gets \{x_i \in C \mid D_m(z^{(m)}_{x_i}, z^{(m)}_{x_j}) \le \delta\}$ 

\For{$t = 1$ to $b$} \algorithmiccomment{Greedy Set Cover selection}
    \State $x_t \gets \arg\max_{x_j \in C \setminus S} \sum_{m=1}^{M} \alpha_m \left| B_{\delta}^{(m)}(x_j) \cap \mathcal{U} \right|$
    \State $S \gets S \cup \{x_t\}$
    
    \For{$m = 1$ to $M$} 
    \algorithmiccomment{Update uncovered set for all embeddings}
        \State $\mathcal{U} \gets \mathcal{U} \setminus B_{\delta}^{(m)}(x_t)$
    \EndFor
\EndFor

\State $\mathcal{M} \gets \mathcal{M} \cup S$; \quad \Return $\mathcal{M}$
\end{algorithmic}
\end{algorithm}
\subsection{\method \probcover Main results}
We next evaluate \method instantiated with \probcover, following the same experimental protocol described in Section~\ref{subsec:main_results}. 

As shown in Fig.~\ref{fig:mers_probcover_cifar100_faa}, \method\ \probcover improves performance over the corresponding replay methods and selection strategies, particularly under tight memory constraints.
While \probcover can outperform the Max-Herding selection strategy in some configurations, \method\ \maxherding consistently achieves the strongest results overall.

\begin{figure*}[h]
    \centering
\begin{subfigure}[t]{0.33\textwidth}
        \centering
        \includegraphics[width=\linewidth]{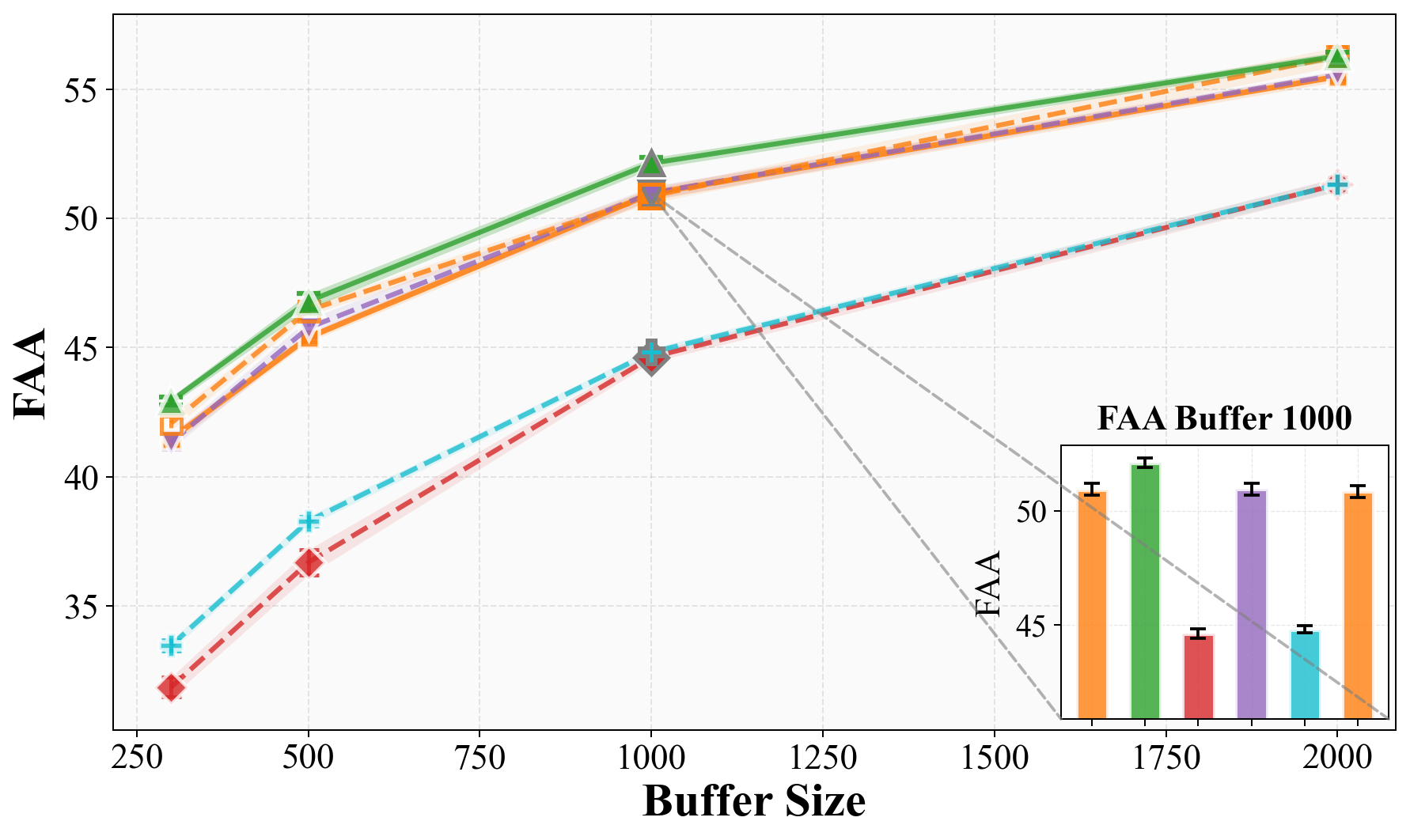}
        \caption{ER-ACE-STAR}
    \end{subfigure}
    \hfill
    \begin{subfigure}[t]{0.33\textwidth}
        \centering
        \includegraphics[width=\linewidth]{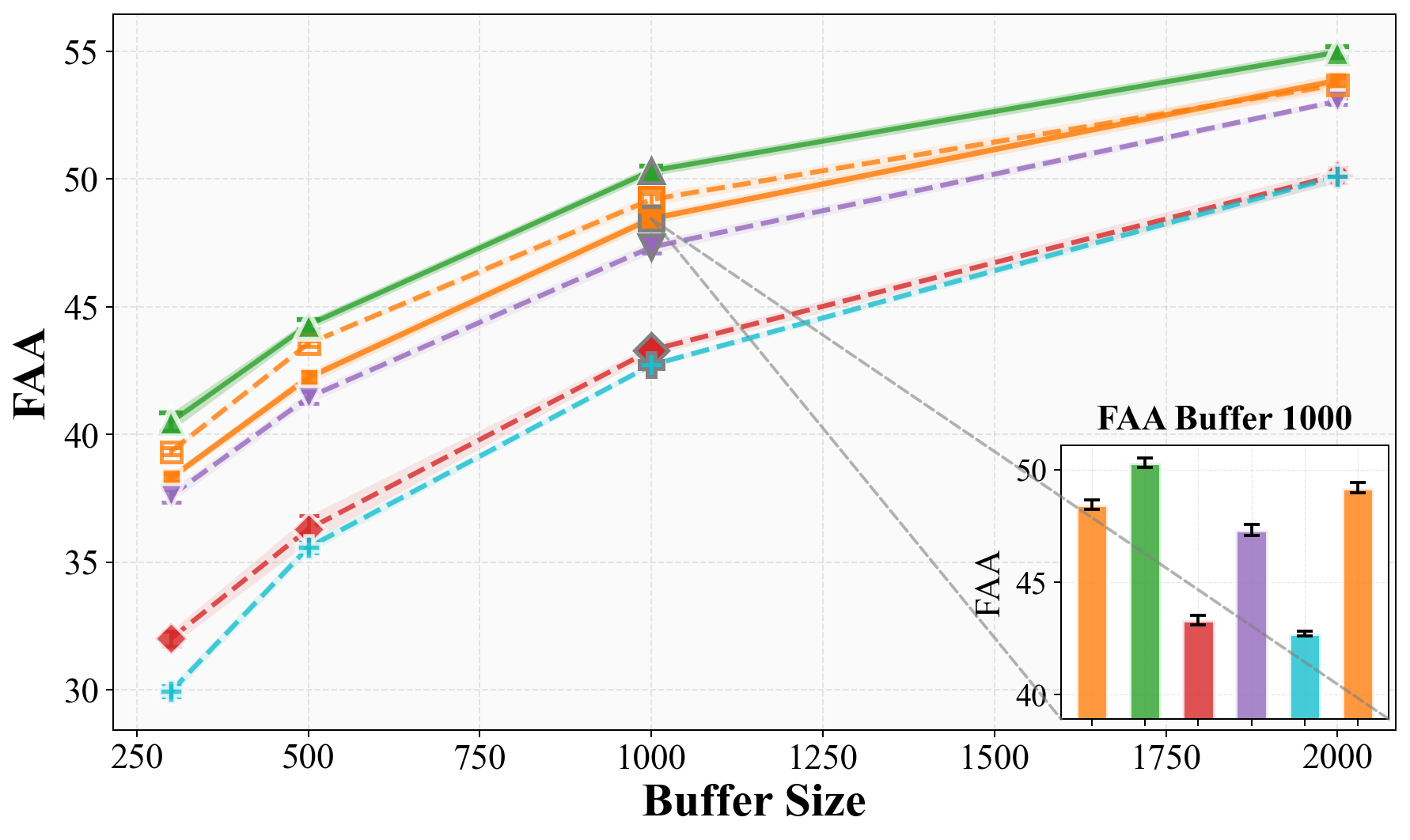}
        \caption{ER-ACE}
    \end{subfigure}
    \begin{subfigure}
    [t]{0.33\textwidth}
        \centering    \includegraphics[width=\linewidth]{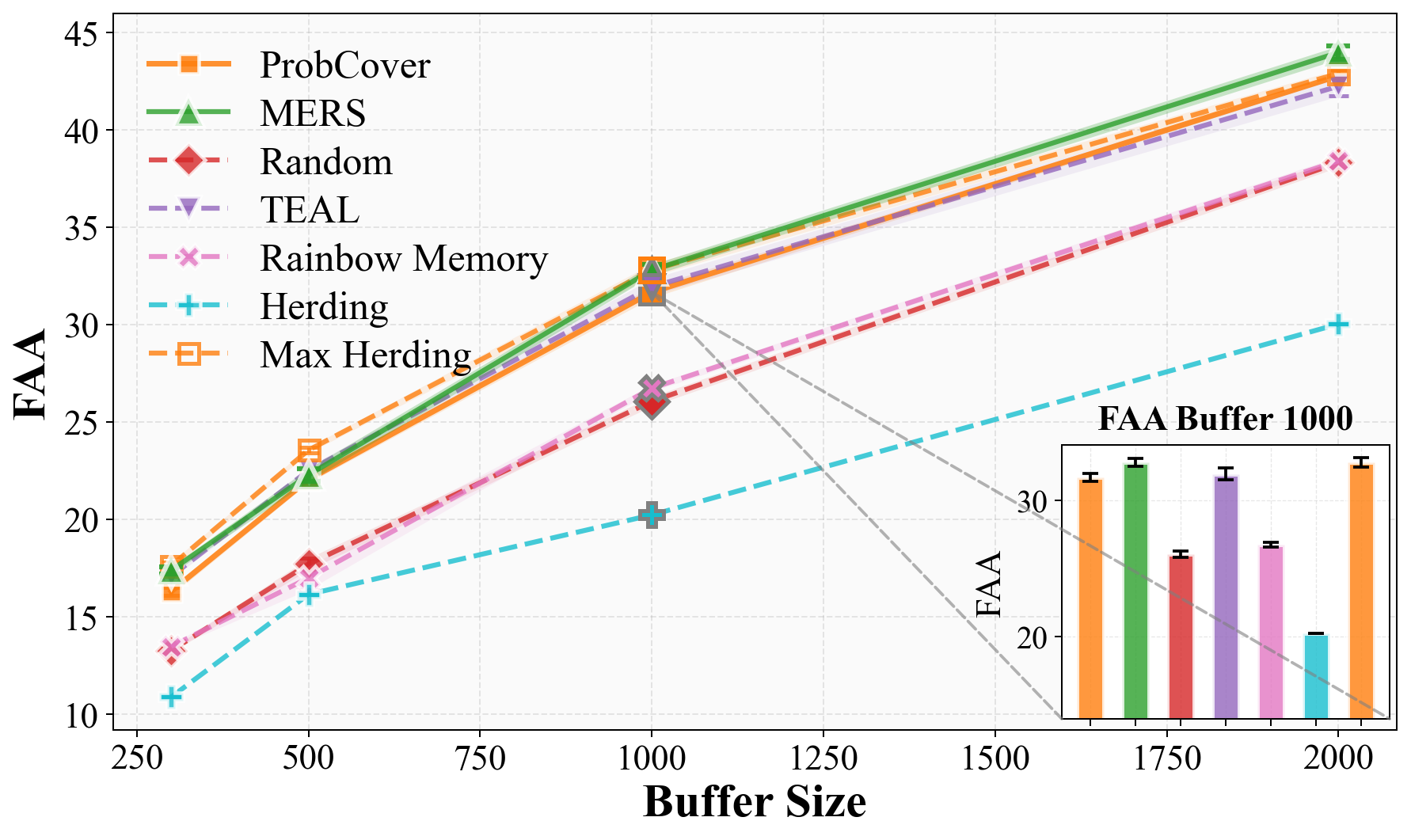}
        \caption{ER}
    \end{subfigure}
    \vspace{0.5em}

\caption{\textbf{\method \probcover}: FAA as a function of memory size $|M|$ on Split CIFAR-100 for three continual learning algorithms, see Section~\ref{sec:ER-methods}. Results with \method are compared against alternative selection strategies, see Section~\ref{sec:selection-methods}. }
    \label{fig:mers_probcover_cifar100_faa}
\end{figure*}
\begin{figure*}[h]
    \centering
\begin{subfigure}[t]{0.49\textwidth}
        \centering
        \includegraphics[width=\linewidth]{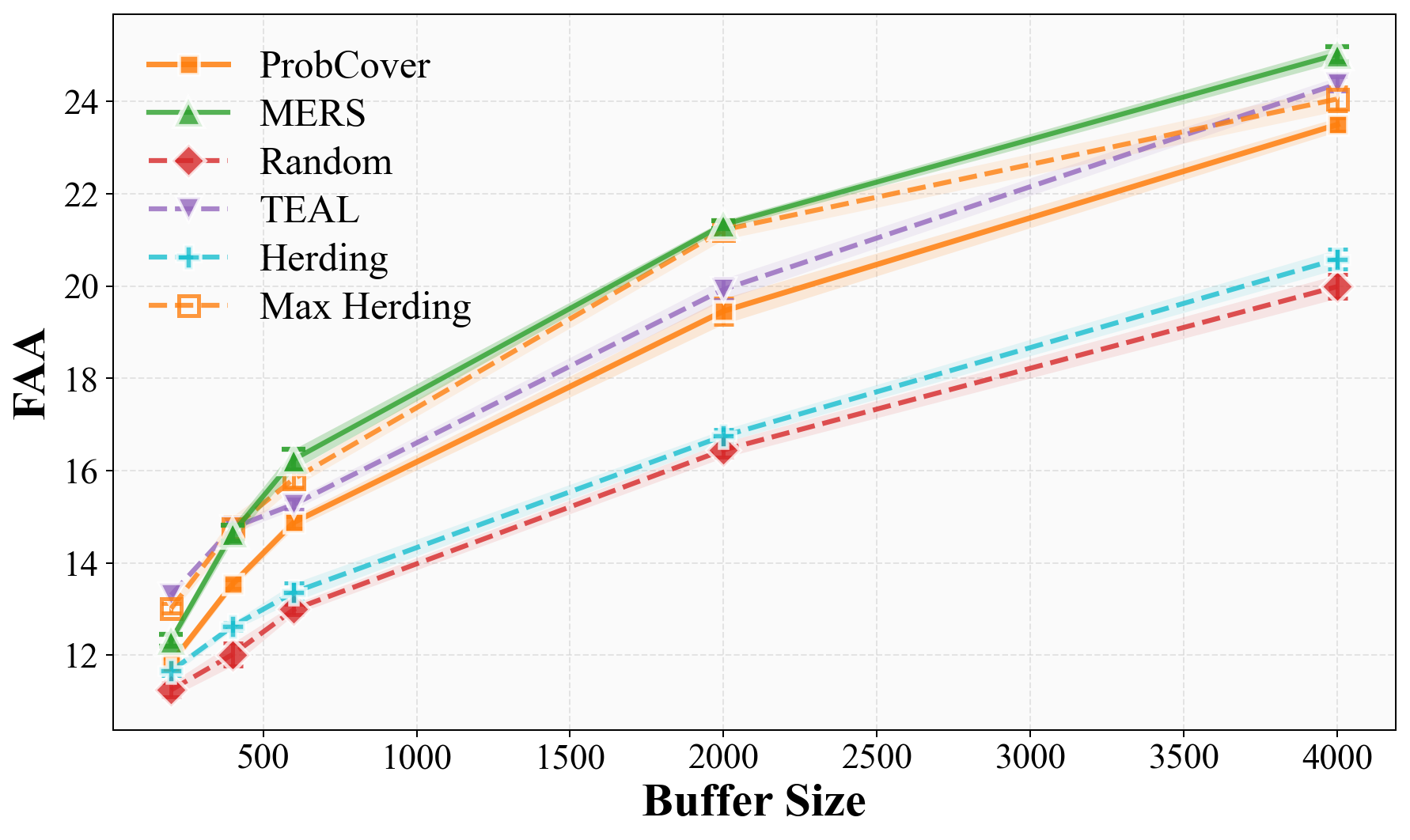}
        \caption{FAA}
    \end{subfigure}
    \hfill
    \begin{subfigure}[t]{0.49\textwidth}
        \centering
          \includegraphics[width=\linewidth]{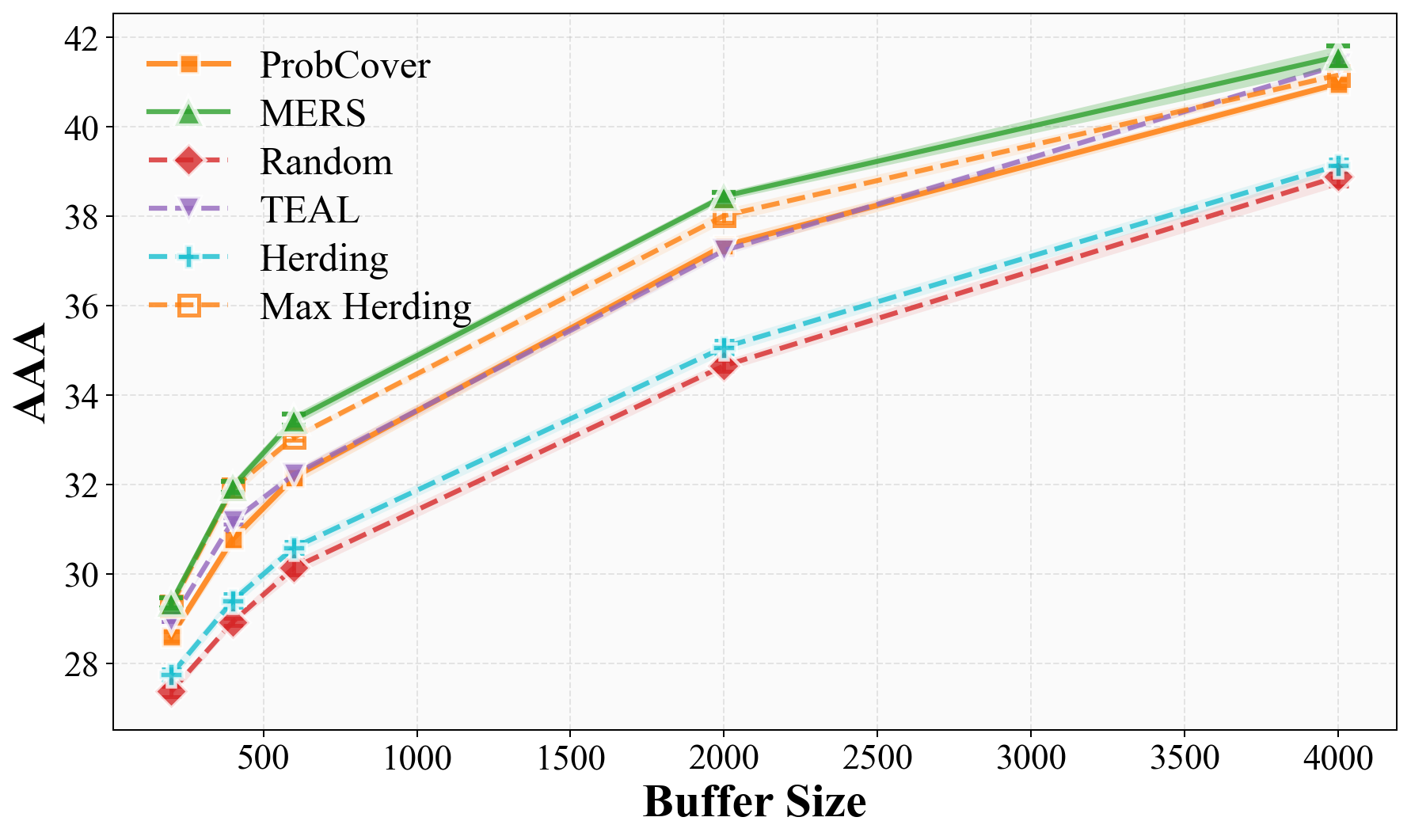}
        \caption{AAA}
    \end{subfigure}

\caption{\textbf{\method \probcover}: FAA (right) and AAA (left) as a function of memory size $|M|$ on  Split CIFAR-100 with ER-ACE. Results with \method are compared against alternative selection strategies. }
    \label{fig:mers_probcover_tinyimg_er_ace}
\end{figure*}
\subsection{Selection stability}

\begin{figure*}
\begin{minipage}[t]{0.99\textwidth}
    \centering
    \begin{subfigure}[t]{0.48\textwidth}
        \centering
        \includegraphics[width=\linewidth]{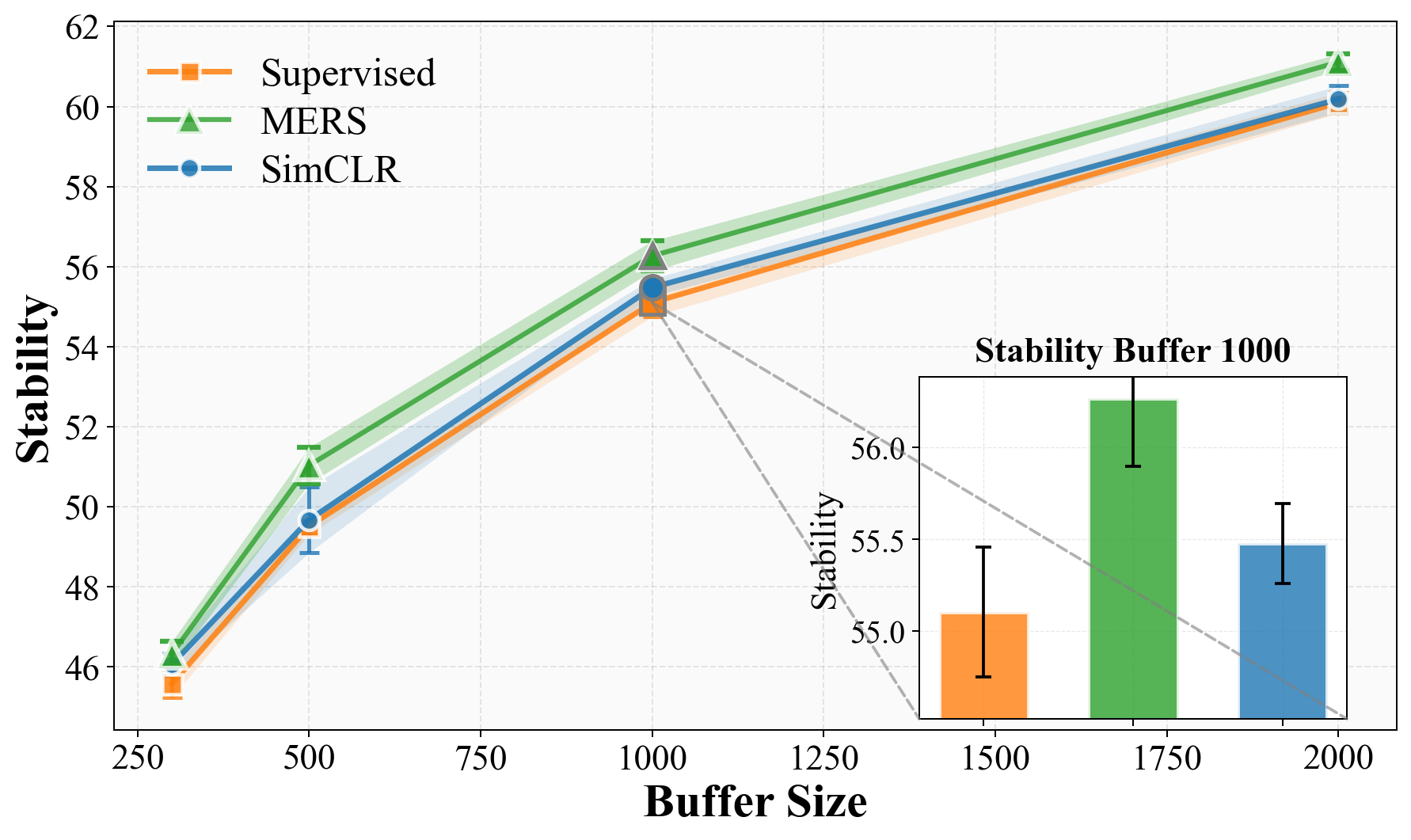}
        \caption{Stability}
    \end{subfigure}
    \hfill
    \begin{subfigure}[t]{0.48\textwidth}
        \centering
        \includegraphics[width=\linewidth]{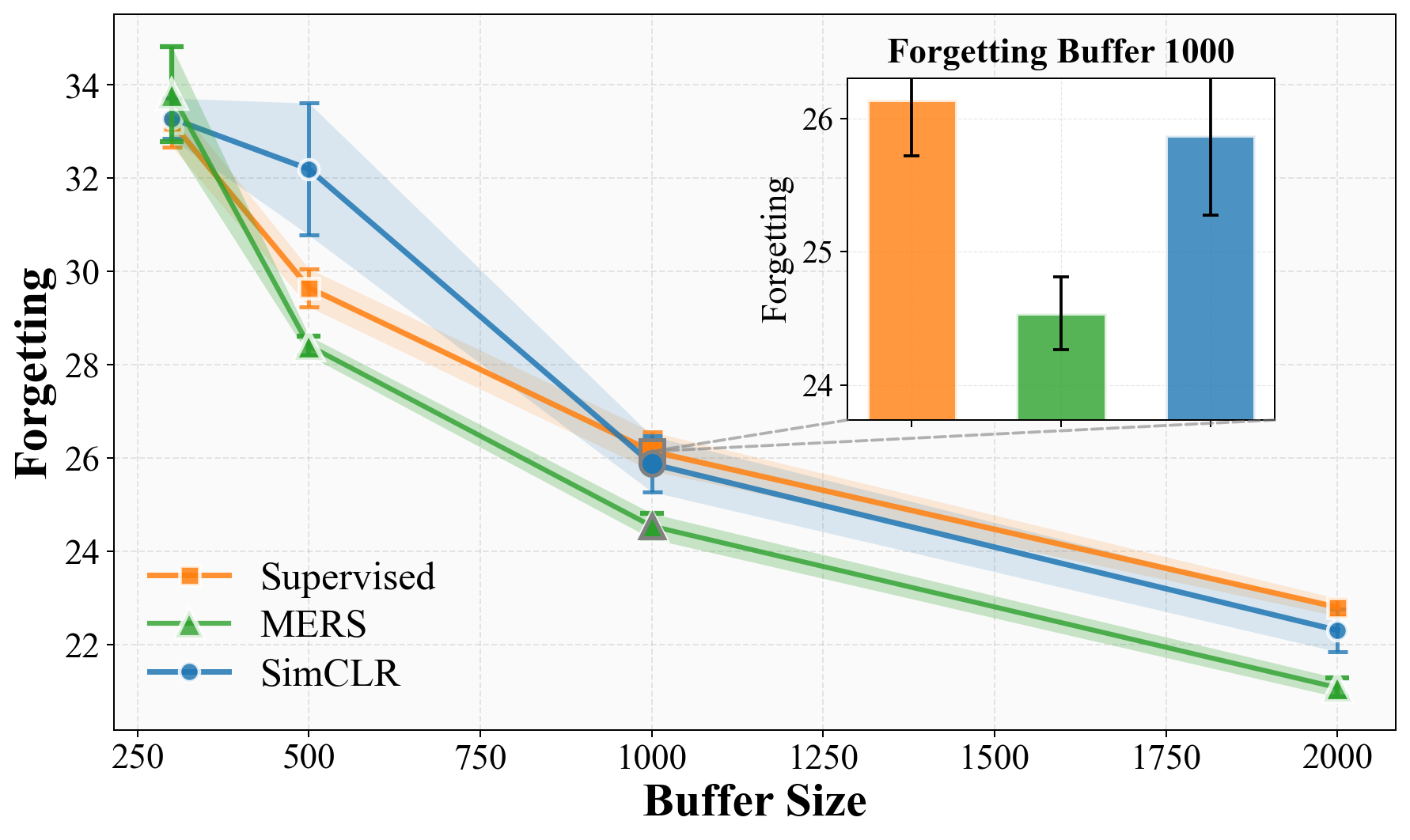}
        \caption{Forgetting}
    \end{subfigure}

    \caption{\textbf{\method \probcover}: Stability and forgetting of ER-ACE-STAR with \method as a function of $|M|$ on Split CIFAR-100.}\label{fig:stability_forgetting_er_ace_star_cifar100_probcover}

\end{minipage}
\end{figure*}
\begin{figure*}
\begin{minipage}[t]{0.99\textwidth}
    \centering
    \begin{subfigure}[t]{0.48\textwidth}
        \centering
        \includegraphics[width=\linewidth]{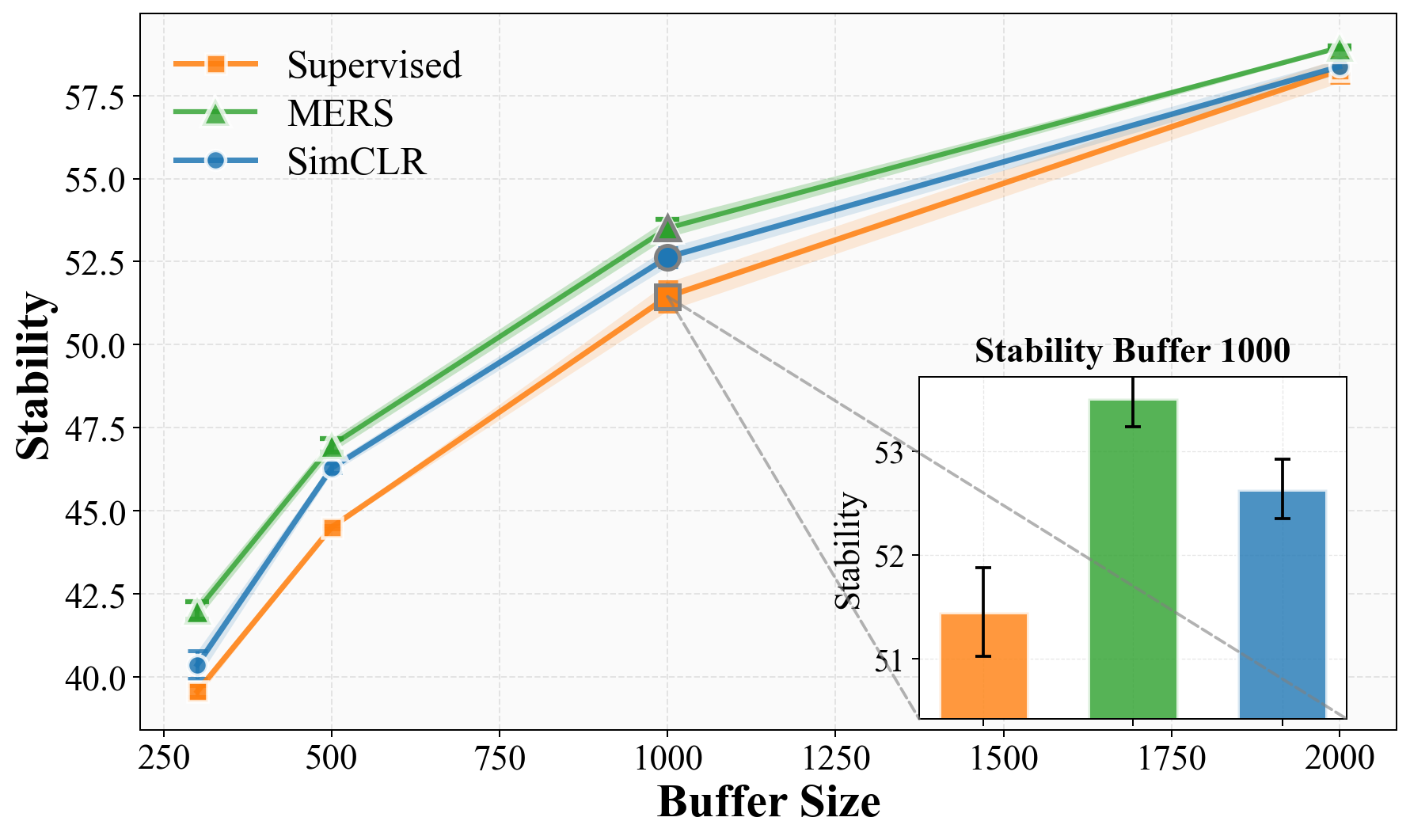}
        \caption{Stability}
    \end{subfigure}
    \hfill
    \begin{subfigure}[t]{0.48\textwidth}
        \centering
        \includegraphics[width=\linewidth]{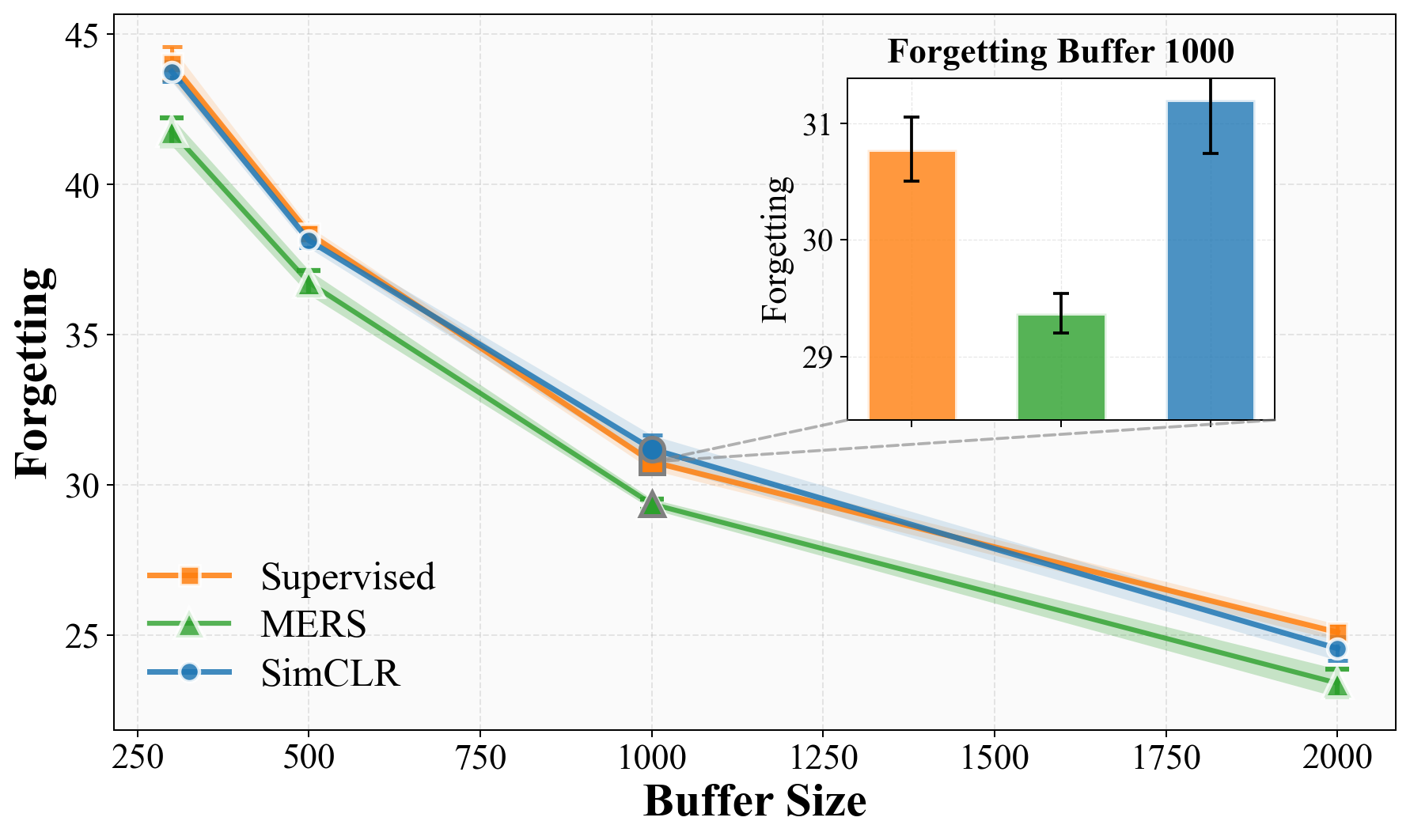}
        \caption{Forgetting}
    \end{subfigure}

    \caption{\textbf{\method \probcover}: Stability and forgetting of ER-ACE with \method as a function of $|M|$ on Split CIFAR-100.}\label{fig:stability_forgetting_er_ace_cifar100_probcover}

\end{minipage}
\end{figure*}
\begin{figure*}
\begin{minipage}[t]{0.99\textwidth}
    \centering
    \begin{subfigure}[t]{0.48\textwidth}
        \centering
        \includegraphics[width=\linewidth]{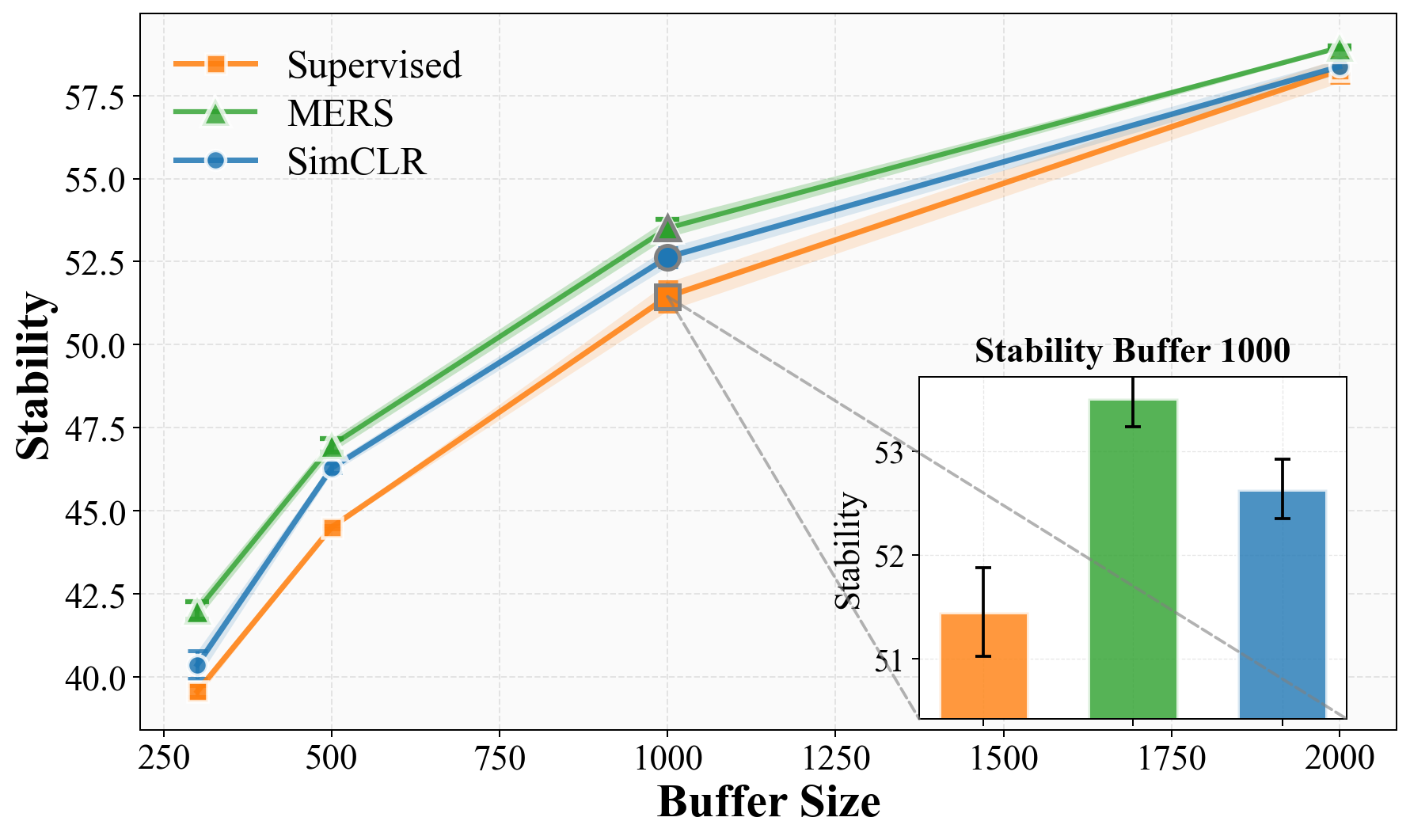}
        \caption{Stability}
    \end{subfigure}
    \hfill
    \begin{subfigure}[t]{0.48\textwidth}
        \centering
        \includegraphics[width=\linewidth]{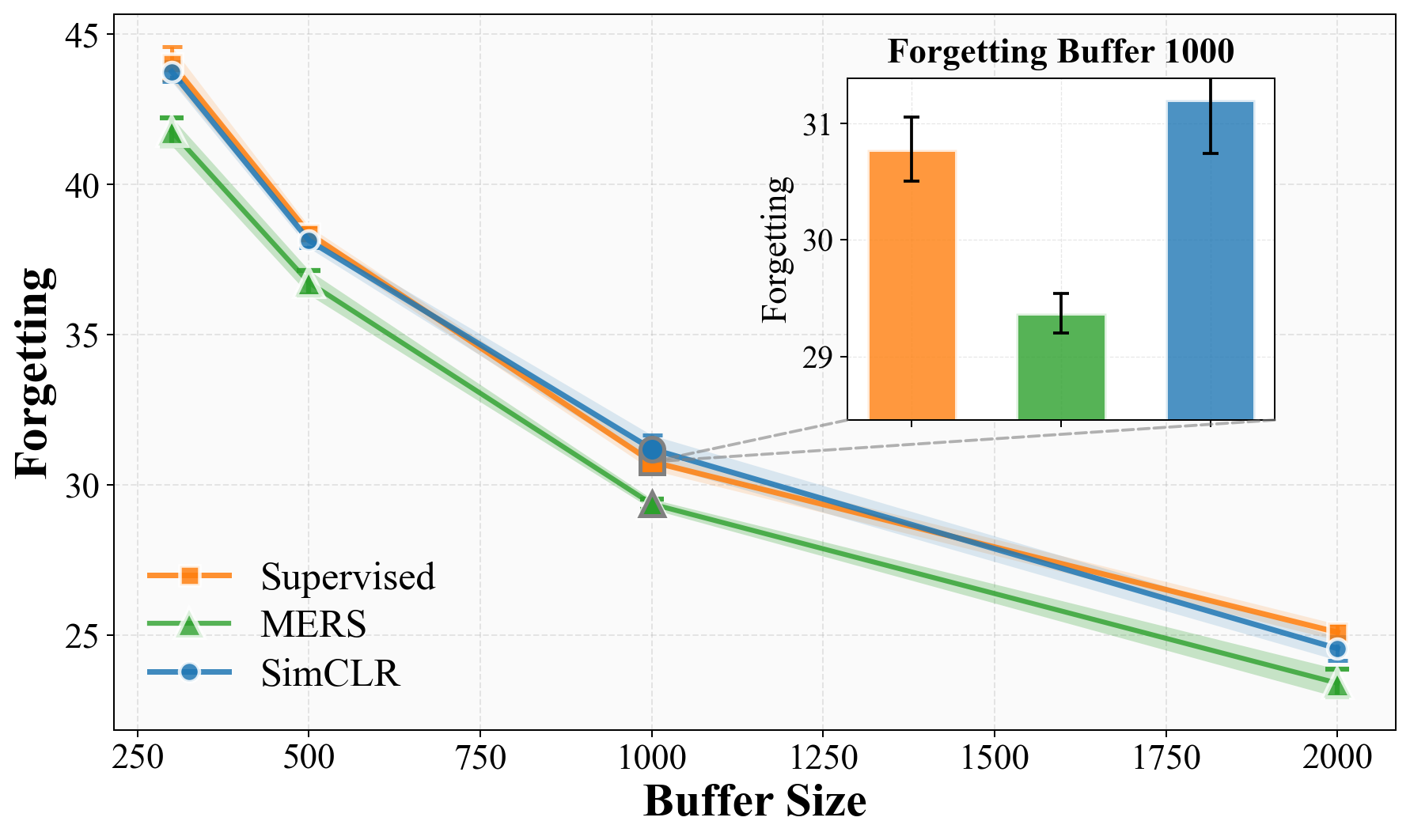}
        \caption{Forgetting}
    \end{subfigure}

    \caption{\textbf{\method \probcover}: Stability and forgetting of ER with \method as a function of $|M|$ on Split CIFAR-100.}\label{fig:stability_forgetting_er_cifar100_probcover}

\end{minipage}
\end{figure*}
Following the selection stability and forgetting analysis presented in Section~\ref{sec:selection_stability}, we analyze selection stability and forgetting for \method\ \probcover under varying memory budgets. Results for ER-ACE-STAR, ER-ACE, and ER on Split CIFAR-100 are shown in Figs.~\ref{fig:stability_forgetting_er_ace_star_cifar100_probcover}–\ref{fig:stability_forgetting_er_cifar100_probcover}.

Consistent with the trends observed for Max-Herding, \probcover based on self-supervised SimCLR embeddings exhibits higher selection stability and lower forgetting compared to \probcover using supervised embeddings.
This indicates that self-supervised representations lead to more consistent buffer composition over time, independent of the specific coverage objective.
While \probcover remains less stable than the corresponding Max-Herding variant, it improves over supervised embedding-based selection and reinforces the evidence presented in the main text regarding the stabilizing effect of self-supervised embeddings.
\section{Submodularity and greedy approximation for Multiple Embedding coverage}
\label{app:submodular}
\begin{proposition}
The function $F : 2^X \to \mathbb{R}_{\ge 0}$ defined in Definition~\ref{def:weighted-coverage}
is non-negative, normalized, monotone, and submodular.
\end{proposition}
\begin{proof}
By Definition~\ref{def:weighted-coverage}, there exists a finite index set $U$,
non-negative weights $\{w_u\}_{u \in U}$, and subsets $\{C_u \subseteq X\}_{u \in U}$
such that for every $L \subseteq X$,
\begin{equation*}
  F(L)
  \;=\;
  \sum_{u \in U} w_u \, \mathbf{1}\bigl[\,L \cap C_u \neq \emptyset\,\bigr],
\end{equation*}
where $\mathbf{1}[\cdot]$ is the indicator function.

\paragraph{Non-negativity and normalization.}
Since all weights $w_u$ are non-negative and indicators are in $\{0,1\}$, we have
$F(L) \ge 0$ for all $L \subseteq X$. For $L = \emptyset$ we have
$\emptyset \cap C_u = \emptyset$ for every $u \in U$, hence all indicators are zero and
$F(\emptyset) = 0$. Thus $F$ is non-negative and normalized.

\paragraph{Monotonicity.}
Let $A \subseteq B \subseteq X$. If an index $u \in U$ is covered by $A$, i.e.,
$A \cap C_u \neq \emptyset$, then since $A \subseteq B$ we also have
$B \cap C_u \neq \emptyset$. Therefore,
\[
  \{u \in U : A \cap C_u \neq \emptyset\}
  \;\subseteq\;
  \{u \in U : B \cap C_u \neq \emptyset\},
\]
and by non-negativity of the weights,
\[
  F(A)
  \;=\;
  \sum_{u : A \cap C_u \neq \emptyset} w_u
  \;\le\;
  \sum_{u : B \cap C_u \neq \emptyset} w_u
  \;=\;
  F(B).
\]
Thus $F$ is monotone.

\paragraph{Submodularity.}
To show submodularity, let $A \subseteq B \subseteq X$ and $x \in X \setminus B$.
Consider the marginal gains
\[
  \Delta_x(A) := F(A \cup \{x\}) - F(A),
  \]
\[  \Delta_x(B) := F(B \cup \{x\}) - F(B).
\]
By the definition of $F$,
\begin{equation*}
\begin{aligned}
  \Delta_x(A)
  &= \sum_{u \in U} w_u
     \Bigl(
       \mathbf{1}\bigl[(A \cup \{x\}) \cap C_u \neq \emptyset\bigr] \\
 &\quad - \mathbf{1}\bigl[A \cap C_u \neq \emptyset\bigr]
     \Bigr) \\
  &= \sum_{u \in U : x \in C_u,\; A \cap C_u = \emptyset} w_u.
\end{aligned}
\end{equation*}
Indeed, $u$ contributes to the marginal gain for $A$ if and only if $u$ was not
covered by $A$ (so $A \cap C_u = \emptyset$) but becomes covered after adding $x$,
which happens precisely when $x \in C_u$.

Analogously,
\[
  \Delta_x(B)
  = \sum_{u \in U : x \in C_u,\; B \cap C_u = \emptyset} w_u.
\]
Since $A \subseteq B$, we have
\begin{equation*}
\begin{aligned}
\
  &\{u \in U : x \in C_u,\; B \cap C_u = \emptyset\}
  \;\subseteq\; \\
  &\{u \in U : x \in C_u,\; A \cap C_u = \emptyset\},
\
\end{aligned}
\end{equation*}
and all weights are non-negative. Therefore
\[
  \Delta_x(B)
  \;\le\;
  \Delta_x(A),
\]
which is exactly the submodularity inequality
\[
  F(A \cup \{x\}) - F(A)
  \;\ge\;
  F(B \cup \{x\}) - F(B).
\]

\paragraph{Greedy approximation guarantee.}
Since $F$ is non-negative, normalized, monotone, and submodular, the greedy
algorithm yields a $(1 - 1/e)$-approximation under a cardinality constraint
\cite{nemhauser1978analysis}, i.e.,
\[
  F(L_b) \;\ge\; (1 - 1/e)\,F(L^*).
\]
\end{proof}

\section{Time and Space complexity of \method}
We analyse the computational cost under the standard setting in which the
selection strategy is invoked \emph{once per training episode}.
Let $n$ be the number of examples from the current episode that belong to
class~$c$, $M$ the number of distinct embedding spaces, $d$ the dimensionality
of each embedding, and $b$ the class-wise memory-buffer budget (\ the
number of items that \buffer\ may store for class~$c$).
\paragraph{Self-supervised stage.}During every episode, \method{} is called exactly once. Running SimCLR for $E_{\text{ssl}}$ epochs on $A=2$ views of the $n$ episode images costs  
$$T_{\text{SimCLR}} = O(E_{\text{ssl}}\,A\,n\,P)$$
with $P$ trainable parameters.  
Self-supervised training consumes $$S_{\text{SimCLR}} = O(P + s\,f)$$ space model parameters $P$ plus the current batch’s $s$ activations of size $f$, and the batch size $s$.
The SimCLR weights are discarded after each episode, persistent memory
is dominated by the replay images.
\subsubsection{\method \probcover}  The algorithm consists of two stages:

\paragraph{(i) Ball-graph construction.}
For every embedding $m\in\{1,\dots,M\}$ we compute all pairwise cosine
distances in $\mathbb{R}^{d}$ to obtain the $\delta$-neighbourhoods
$B^{(m)}_{\!\delta}(x)$.  
This step costs $T_{\text{graph}}=O\!\bigl(M\,n^{2}\,\max\{d,b\}\bigr)$
and stores 
$S_{\text{graph}}=O(M\,n^{2})$
adjacency edges.

\paragraph{(ii) Greedy covering.}
Across $b$ iterations we repeatedly pick the vertex that covers the largest
number of still-uncovered neighbours.  
The work per iteration yields
$T_{\text{cover}}=O(|E|+b\,n)\subseteq O(M\,n^{2}+b\,n).$

\paragraph{Overall complexity.}

$$
\begin{aligned}
T_{\text{MERS-\probcover}} &= O\!\bigl(M n^{2} \max{d,b}\bigr),\\[2pt]
S_{\text{MERS-\probcover}} &= O(M n^{2}),
\end{aligned}
$$

The original \emph{\probcover} analysis~\cite{yehuda2022active} reports a
running time of $O\!\bigl(n^{2}\max\{d,b\}\bigr)$.  Our derivation
shows that the Multiple Embedding extension, \textsc{\method–\probcover}, retains the
same quadratic dependence on~$n$ and on $\max{d,b}$, differing only
by the multiplicative factor~$M$ (which equals~2 in all of our experiments).
\paragraph{(ii) Greedy \maxherding selection.}
\paragraph{(i) Integrated-kernel construction.}
We assemble the Gram matrix
$$K_{ij}=k(x_i,x_j)=\sum_{m=1}^{M}\alpha_m\,k_m\!\bigl(x_i^{(m)},x_j^{(m)}\bigr).$$
Forming its $\tfrac{1}{2}n(n-1)$ entries costs
$$T_{\text{kernel}}=O(mn^{2}d),\qquad S_{\text{kernel}}=O(n^{2}).$$

\paragraph{(ii) Greedy selection.}
Each of the $b$ iterations scans all candidates ($\le n$) and exploits the
pre-computed kernel:
$$T_{\text{MaxHerding}}=O(b\,n^{2}),\qquad S_{\text{MaxHerding}}=O(n).$$

\paragraph{Overall complexity.}
$$
\begin{aligned}
T_{\text{MERS--\maxherding}} &= O\!\bigl(mn^{2}(d+b)\bigr),\\[2pt]
S_{\text{MERS--\maxherding}} &= O\!\bigl(n^{2}+n d\bigr).
\end{aligned}
$$

\section{Detailed Theoretical Analysis}
\label{sec:Detailed_Theoretical_Analysis}

In this section we present a theoretical analysis that motivates sampling from a mixture of supervised and self-supervised representations. While the benefits of supervised embeddings are clear - they capture class-discriminative structure, the goal here is to formalize the complementary value of self-supervised representations and explain when they can improve robustness to future classes. Specifically, in Section~\ref{sec:benefits} we show that sampling from SSL embeddings is likely to yield a tighter (smaller) bound on the train-test risk gap than sampling from SL embedding.

To this end we make the following assumptions:\begin{enumerate}
    \item \label{ass:1}
    \textbf{Geometry under supervision vs.\ self-supervision.}
    Supervised learning (SL) tends to concentrate representation variability in a relatively low-dimensional, class-discriminative subspace, whereas self-supervised learning (SSL) tends to preserve a broader set of non-label factors that are stable across views and yields representations that are universally good for images (or domain objects), regardless of class label.
    
    \item \label{ass:2}
    \textbf{Matched global scale (equal compression).}
    When comparing SL and SSL for coverage-based selection, we normalize the embeddings so that both have the same global scale/compression level. 
\end{enumerate}

Assumption~\ref{ass:1} is motivated by standard information-theoretic and geometric perspectives: (i) supervised training encourages \emph{label-sufficient} compression of representations \citep{tishby2000information,alemi2017deep}; (ii) contrastive self-supervision can be viewed as maximizing agreement (shared information) between augmented views while simultaneously promoting spread/uniformity (or decorrelation) of representations \citep{oord2018representation,wang2020understanding}.

Assumption~\ref{ass:2} follows from the scale handling in our selection objectives. Both \probcover\ and \maxherding\ include an explicit length-scale hyper-parameter ($\delta$ and $\sigma$, respectively) that is chosen so as to make the procedure effectively scale-invariant. Therefore, when comparing the selected sets under two different embeddings, we first align their global scale to ensure a fair comparison and to prevent trivial differences caused by an overall rescaling.

For the purposes of the following analysis, we assume there exists a feature space $\R^n$ in which the class-conditional distribution of each class, past and future, can be approximated by a Gaussian $N(\mu,\Sigma)$ in $\R^n$ with $\Sigma$ positive-definite. We interpret this space as emphasizing class-relevant factors of variation, abstracting away label-irrelevant features due to such factors as illumination, pose, or background.

\subsection{Selective feature compression increases class conditional divergence}
\label{sec:selective}

In this section we show that the probabilistic distortion induced by an anisotropic embedding is typically larger, as measured by KL divergence, than the distortion induced by an isotropic embedding, or by an embedding that preserves the isotropy of the original distribution.

Consider a single class from the current episode. Without loss of generality, assume its mean is at the origin and its class-conditional distribution in the reference feature space is $$G_o := \cN(0,\Sigma).$$

\paragraph{Using Assumption~\ref{ass:1}.}
Our method \method\ selects a representative set for this class using an alternative embedding, which induces a (potentially) different class-conditional distribution in $\R^n$. By Assumption~\ref{ass:1}, we model the class-conditional distribution under SSL and SL as follows:
\paragraph{SSL.} As idealized proxies for a representation that preserves broad, view-stable factors and avoids label-induced anisotropy, we consider two SSL-induced class-conditional models:
\begin{equation*}
\begin{aligned}
G^{(1)}_{\mathrm{SSL}}
&:= \mathcal{N}(0, \sigma \Sigma),
\qquad \sigma \in (0,1), \\
G^{(2)}_{\mathrm{SSL}}
&:= \mathcal{N}(0, \sigma I_n).
\end{aligned}
\end{equation*}
    The first model, $G^{(1)}_{\mathrm{SSL}}$, corresponds to the idealized case in which SSL recovers the \emph{true} class geometry up to a global rescaling; while optimistic, it yields cleaner expressions and serves as a convenient analytic baseline. The second model, $G^{(2)}_{\mathrm{SSL}}$, represents an isotropic (whitened) geometry - a more faithful proxy for the ``uniformity'' pressure in contrastive objectives, which encourages representations to spread approximately uniformly on (or near) a sphere \citep{wang2020understanding}.
        
    \paragraph{SL.} We model label-driven selective compression by an anisotropic rescaling of the covariance. For some $m\in\{1,\dots,n-1\}$,
    \begin{equation*}
    \begin{split}
      &G_{\mathrm{SL}} := \cN\!\bigl(0,\Sigma^{1/2}D\Sigma^{1/2}\bigr),\\&D=\diag(\underbrace{\alpha,\ldots,\alpha}_{m\ \text{times}}, \underbrace{\beta,\ldots,\beta}_{n-m\ \text{times}}), \quad \alpha>\beta>0.
    \end{split}
    \end{equation*}
    Here, the $m$ directions scaled by $\alpha$ represent class-discriminative variability retained by supervision, while the remaining $n-m$ directions are compressed by $\beta$.

\paragraph{Enforcing Assumption~\ref{ass:2}.}

We match the \emph{volume} of the covariance ellipsoids, i.e., the Mahalanobis level sets
\begin{equation*}
E_{\Sigma'} \;:=\;\{x\in\R^n:\ x^\top(\Sigma')^{-1}x \le 1\}.
\end{equation*}
Since $\Vol(E_{\Sigma'})=\Vol(B_1)\sqrt{\det(\Sigma')}$, where $B_1$ is the unit ball in $\R^n$, equal volume is equivalent to matching determinants:
\begin{equation*}
\Vol(E_{\Sigma_1})=\Vol(E_{\Sigma_2}) \quad\Longleftrightarrow\quad \det(\Sigma_1)=\det(\Sigma_2).
\end{equation*}
Thus, this constraint is equivalent to
\begin{equation*}
\label{eq:equal-vol}
\begin{split}
i\! = \!\! 1\!\!:~~~  &\detm(\Sigma^{1/2}D\Sigma^{1/2})=\detm(\sigma\Sigma) ~\Longleftrightarrow~ \\&\detm(D)=\sigma^n ~\Longleftrightarrow~ \alpha^m\beta^{n-m}=\sigma^n. \\
i\! = \!\! 2\!\!:~~~  &\detm(\Sigma^{1/2}D\Sigma^{1/2})=\detm(\sigma I_n) ~\Longleftrightarrow~ \\&\detm(D)\cdot\detm(\Sigma)=\sigma^n ~\Longleftrightarrow~ \\&\alpha^m\beta^{n-m}=\frac{\sigma^n}{\detm(\Sigma)}.
\end{split}
\end{equation*}

\begin{lemma}[KL-divergence]
\label{lem:kl-closed-forms}
The KL-divergence between the true class conditional distribution $G_o$ and the SSL-induced distribution can be expressed as follows:
\begin{equation*}
\begin{aligned}
\label{eq:kl_i=1_SSL}
i\! = \!\! 1\!\!:\quad &\KL(G_o\|G^{(1)}_{\mathrm{SSL}})=\frac12\Big(\frac{n}{\sigma}-n+n\ln\sigma\Big),\quad 
\end{aligned}
\end{equation*}
\begin{equation*}
\label{eq:kl_i=2_SSL}
\begin{aligned}
i=2:\quad \KL(G_o\|G^{(2)}_{\mathrm{SSL}}) = \frac{1}{2} \biggl( &\frac{1}{\sigma}\text{tr}(\Sigma) - n \\ 
+ n\ln\sigma - \ln\det(\Sigma) \biggr);
\end{aligned}
\end{equation*}
The KL-divergence between $G_o$ and the SL-induced distribution is:

{\small 
\begin{align}
\label{eq:kl_i=1_SL}
i=1:\quad &\KL(G_o\|G_{\text{SL}}) = \frac{1}{2}\Big(\frac{m}{\alpha}+\frac{n-m}{\beta}-n+n\ln\sigma\Big), \\ 
\label{eq:kl_i=2_SL}
i=2:\quad &\KL(G_o\|G_{\text{SL}}) = \frac{1}{2}\bigg(\frac{m}{\alpha}+\frac{n-m}{\beta}-n+n\ln\sigma - \nonumber \\ 
&\quad \ln\det(\Sigma)\bigg).
\end{align}
}
\end{lemma}

\begin{proof}
These identities follow from the standard KL-divergence formula for zero-mean Gaussians with positive definite covariance matrices:
\begin{equation*}
\resizebox{1\columnwidth}{!}{%

$\KL\big(\cN(0,\Sigma_0)\,\|\,\cN(0,\Sigma_1)\big) =\frac12\left(\tr(\Sigma_1^{-1}\Sigma_0)-n+\ln\frac{\detm(\Sigma_1)}{\detm(\Sigma_0)}\right),$
}
\end{equation*}
and the equal-volume constraints in (\ref{eq:equal-vol}).
\end{proof}

\begin{proposition}
\label{prop:kl-ordering-amgm}
For $i=1$, anisotropy increases $\KL(G_o\|\,\cdot\,)$ under equal volume:
\begin{equation*}
\resizebox{1\columnwidth}{!}{%
$\KL(G_o\|G_{\mathrm{SL}})\;\ge\;\KL(G_o\|G^{(1)}_{\mathrm{SSL}}), \quad\frac{\KL(G_o\|G_{\mathrm{SL}})}{\KL(G_o\|G^{(1)}_{\mathrm{SSL}})}\xrightarrow[\beta\to 0]{} \infty,$}
\end{equation*}
with equality in the first expression iff $\alpha=\beta=\sigma$.
\end{proposition}

\begin{proof}
By Lemma~\ref{lem:kl-closed-forms}, together with \eqref{eq:kl_i=1_SSL} and \eqref{eq:kl_i=1_SL},
\begin{equation*}
\KL(G_o\|G_{\mathrm{SL}})-\KL(G_o\|G^{(1)}_{\mathrm{SSL}}) =\frac12\Big(\frac{m}{\alpha}+\frac{n-m}{\beta}-\frac{n}{\sigma}\Big).
\end{equation*}
To show that this expression is nonnegative, we apply the weighted AM--GM inequality to $\frac{1}{\alpha}$ and $\frac{1}{\beta}$ with weights $\frac{m}{n}$ and $\frac{n-m}{n}$:
\begin{equation*}
\begin{aligned}
\frac{m}{n}\frac{1}{\alpha}+\frac{n-m}{n}\frac{1}{\beta} \;\ge\; \Big(\frac{1}{\alpha}\Big)^{m/n}\Big(\frac{1}{\beta}\Big)^{(n-m)/n} \\=\frac{1}{\alpha^{m/n}\beta^{(n-m)/n}} =\frac{1}{\sigma}, 
\end{aligned}
\end{equation*}
where the last equality uses the equal-volume constraint
$\alpha^m\beta^{n-m}=\sigma^n$ in \eqref{eq:equal-vol}.

To see the asymptotic result, note that as $\beta\to 0$ under $\alpha^m\beta^{n-m}=\sigma^n$, we have $\alpha\to\infty$ and $\frac{m}{\alpha}\to 0$ while $\frac{n-m}{\beta}\to\infty$, which implies that $\KL(G_o\|G_{\mathrm{SL}})\to\infty$ whereas $\KL(G_o\|G^{(1)}_{\mathrm{SSL}})$ remains finite.
\end{proof}

\begin{proposition}
\label{prop:smallbeta-dominance}
For $i=2$, there exists $\beta_0>0$ such that 
\begin{equation*}
\begin{aligned}
\KL(G_o\|G_{\mathrm{SL}})\;\ge\;\KL\!\bigl(G_o\|G^{(2)}_{\mathrm{SSL}}\bigr)~~~\forall \beta<\beta_0,\\ \frac{\KL(G_o\|G_{\mathrm{SL}})} {\KL\!\bigl(G_o\|G^{(2)}_{\mathrm{SSL}}\bigr)} \xrightarrow[\beta\to 0]{} \infty.
\end{aligned}
\end{equation*}
\end{proposition}

\begin{proof}
By Lemma~\ref{lem:kl-closed-forms}, together with \eqref{eq:kl_i=2_SSL} and \eqref{eq:kl_i=2_SL},
\begin{equation*}
\begin{aligned}  
&\KL(G_o\|G_{\mathrm{SL}})-\KL(G_o\|G^{(2)}_{\mathrm{SSL}}) \\&=\frac12\Big(\frac{m}{\alpha}+\frac{n-m}{\beta}-\frac{\tr(\Sigma)}{\sigma}\Big).
\end{aligned}
\end{equation*}
As $\beta\to 0$ under $\alpha^m\beta^{n-m}=\sigma^n/\det\Sigma$, necessarily $\alpha\to\infty$, so $\frac{m}{\alpha}\to 0$ while $\frac{n-m}{\beta}\to\infty$. Therefore the difference above is positive for all sufficiently small $\beta$, proving the existence of $\beta_0$ and the asymptotic result.
\end{proof}

\noindent
The asymptotic results show that, in the highly anisotropic regime (e.g., when $\beta$ is very small, as suggested by a strong form of ``neural collapse'' \citep{PapyanHanDonohoPNAS2020pol}), the KL gap between the SL and SSL proxies can become arbitrarily large.

\subsection{Class conditional shift and domain adaptation}
\label{sec:domain-adapt}

In this section we cast class-incremental learning as a domain adaptation problem, where the effective data distribution shifts between episodes. Since \method selects representatives on a per-class basis within the current episode, we focus on the resulting class-conditional shift and study how it affects a downstream classification task: distinguishing the current class from the $K-1$ new classes that will appear in the next episode.

\paragraph{Single-class conditional shift assumption.}
Let $\R^n$ denote the input space and let $\cY=\{1,\dots,K\}$ be the label space. Label $Y=1$ corresponds to a class from the current episode; without loss of generality we assume its mean satisfies $\mu_1=0$. Labels $Y=2,\dots,K$ correspond to the $K-1$ classes that will appear in the next episode. We write $C_i$ for the class associated with label $Y=i$, for $i\in[K]$.

When constructing the training set for the next episode, classes $\{C_i\}_{i=2}^K$ are sampled from their original class-conditional distributions $\cN(\mu_i,\Sigma_i)$ in $\R^n$, as assumed above. In contrast, class $C_1$ is represented by the exemplars stored in the replay buffer, which reflect the (possibly distorted) class-conditional distribution induced by the new embedding.

As customary in domain adaptation, let $S := P_{\mathrm{tr}}(X,Y)$ denote the \emph{source/train} distribution and $T := P_{\mathrm{te}}(X,Y)$ the \emph{target/test} distribution. In our setting the two distributions coincide except for the class-conditional distribution of $C_1$.\footnote{The CIL training procedure rebalances the class prior, ensuring that $P(Y=1)$ matches between train and test regardless of the buffer size.} In particular,
\begin{equation*}
\begin{aligned}
&P_{\mathrm{tr}}(Y=y)=P_{\mathrm{te}}(Y=y)\ \ \forall y\in\cY, \qquad 
\\ &P_{\mathrm{tr}}(X\mid Y=y)=P_{\mathrm{te}}(X\mid Y=y)\ \ \forall y\neq 1,
\end{aligned}
\end{equation*}
but $P_{\mathrm{tr}}(X\mid Y=1)\neq P_{\mathrm{te}}(X\mid Y=1)$. Let $\pi_1 := P_{\mathrm{te}}(Y=1)=P_{\mathrm{tr}}(Y=1)$.

\paragraph{Domain adaptation bound}

For a classifier $h:\cX\to\cY$, the $0$--$1$ loss is $\ell_{01}(h(x),y) := \mathbf{1}_{\{h(x)\neq y\}}\in[0,1]$, and the corresponding risk is
\begin{equation*}
\resizebox{1\columnwidth}{!}{%
$R_D(h) := \Prb_{(X,Y)\sim D}\big[h(X)\neq Y\big]= \E_{(X,Y)\sim D}\big[\ell_{01}(h(X),Y)\big].$
}
\end{equation*}

\begin{theorem}[Train--test risk gap controlled by the shifted class]
\label{cor:tv-gap}
For any classifier $h$,
\begin{equation*}
\resizebox{1\columnwidth}{!}{%
$|R_T(h)-R_S(h)|\le\pi_1\,\TV\!\Big(P_{\mathrm{tr}}(X\mid Y=1),\,P_{\mathrm{te}}(X\mid Y=1)\Big).$}
\end{equation*}
where $\TV$ denotes the total variation distance.
\end{theorem}

\begin{proof}
It is known (see, e.g., \citealp{levinpereswilmer2017}) that for probability measures $S,T$ on the same measurable space and any measurable $f:\cX\times\cY\to[0,1]$,
\begin{equation*}
\resizebox{1\columnwidth}{!}{%
$
\begin{aligned}
&\big|\E_S f-\E_T f\big|\le \TV(S,T):=\sup_{A}|S(A)-T(A)|\\&= \sup_{0\le g\le 1}\big|\E_S g-\E_T g\big|.
\end{aligned}
$
}
\end{equation*}

Moreover, since by assumption $S(x,y)=\pi_y P_{\mathrm{tr}}(x\mid y)$,  $T(x,y)=\pi_y P_{\mathrm{te}}(x\mid y)$ and $P_{\mathrm{tr}}(x\mid y)=P_{\mathrm{te}}(x\mid y)$ for all $y\neq 1$, we get
\begin{equation*}
\resizebox{1\columnwidth}{!}{%

$\E_S[f(X,Y)]-\E_T[f(X,Y)]=\pi_1\Big(\E_{P_{\mathrm{tr}}(X\mid 1)}[f(X,1)]-\E_{P_{\mathrm{te}}(X\mid 1)}[f(X,1)]\Big).$
}
\end{equation*}
Taking $f(X,Y)=\ell_{01}(h(X),Y)$, we obtain
\begin{equation*}
\resizebox{0.95\columnwidth}{!}{%
$
\begin{aligned}
|R_T(h)-R_S(h)|=\pi_1\,\big|\E_{P_{\mathrm{tr}}(X\mid 1)}[\ell_{01}(h(X),1)]-\\ \E_{P_{\mathrm{te}}(X\mid 1)}[\ell_{01}(h(X),1)]big|\le\pi_1\,\TV\!\big(P_{\mathrm{tr}}(X\mid 1),P_{\mathrm{te}}(X\mid 1)\big),
\end{aligned}
$
}
\end{equation*}
which proves the claim.
\end{proof}

\begin{corollary}[KL-controlled train--test risk gap]
\label{cor:kl-gap}
For any classifier $h$,
\begin{equation*}
\resizebox{0.95\columnwidth}{!}{%
$|R_T(h)-R_S(h)| \le \pi_1 \sqrt{\frac12\,\KL\!\Big(P_{\mathrm{te}}(X\mid Y=1)\,\big\|\,P_{\mathrm{tr}}(X\mid Y=1)\Big)}.$}
\end{equation*}
Equivalently,
\begin{equation*}
\resizebox{0.95\columnwidth}{!}{
$R_T(h)\le R_S(h)+ \pi_1 \sqrt{\frac12\,\KL\!\Big(P_{\mathrm{te}}(X\mid Y=1)\,\big\|\,P_{\mathrm{tr}}(X\mid Y=1)\Big)}.$
}
\end{equation*}

\end{corollary}

\begin{proof}
The result follows from Pinsker's inequality \citep{cover2006elements}, which states that for distributions $P,Q$ with finite $\KL(P\|Q)$,
\begin{equation*}
\TV(P,Q)\le \sqrt{\frac12\,\KL(P\|Q)}.
\end{equation*}
\end{proof}

\subsection{The benefits of using the SSL embedding}
\label{sec:benefits}

\begin{proposition}[SSL yields a tighter DA-style bound than SL]
\label{prop:ssl-tighter}
Under the setup of Section~\ref{sec:selective} and the equal-volume normalization, the SSL embedding yields a tighter (smaller) bound on the train-test risk gap than the SL embedding.
\end{proposition}

\begin{proof}
In the notation of Section~\ref{sec:selective}, the test conditional for class $1$ is $P_{\mathrm{te}}(X\mid Y=1)=G_o$, while the corresponding training conditional is $P_{\mathrm{tr}}(X\mid Y=1)=G^{(i)}_{\mathrm{SSL}}$ (under SSL) or $G_{\mathrm{SL}}$ (under SL). Applying Corollary~\ref{cor:kl-gap} gives, for any classifier $h$,
\begin{equation*}
\begin{aligned}
|R_T(h)-R_S(h)| \le \pi_1 \sqrt{\frac12\,\KL\!\bigl(G_o\|G^{(i)}_{\mathrm{SSL}}\bigr)} \\
\text{(SSL embedding),}
\end{aligned}
\end{equation*}
and
\begin{equation*}
\begin{aligned}
|R_T(h)-R_S(h)| \le \pi_1 \sqrt{\frac12\,\KL\!\bigl(G_o\|G_{\mathrm{SL}}\bigr)} \\
\text{(SL embedding).}
\end{aligned}
\end{equation*}
Under equal volume, Proposition~\ref{prop:kl-ordering-amgm} implies $\KL(G_o\|G_{\mathrm{SL}})\ge \KL(G_o\|G^{(1)}_{\mathrm{SSL}})$, and Proposition~\ref{prop:smallbeta-dominance} shows that for sufficiently small $\beta$, $\KL(G_o\|G_{\mathrm{SL}})\ge \KL(G_o\|G^{(2)}_{\mathrm{SSL}})$. In either case, the KL term, and hence the right-hand side of the bound, is smaller under SSL than under SL, which proves the claim.
\end{proof}
\section{Hyperparameters}

\subsection{classification model}
we employ a ResNet-18 backbone trained for 100 epochs with a batch size of 10. The \textbf{ER-ACE} configuration begins with a learning rate of 0.01. The \textbf{ER} and \textbf{MIR} configuration begins with a learning rate of 0.1, for all configurations, SGD optimization includes Nesterov momentum of 0.9 and weight decay 0.0002. The learning rate is decayed by a factor of 0.3 every 66 epochs. All experiments were run with five random seeds (0-4).

\subsection{class order} We follow the canonical class order for each benchmark: Split CIFAR-100 uses classes $[1\dots100]$, and Split TinyImageNet uses classes $[1\dots200]$.
\subsection{Self-Supervised Training}
\label{subsec:ssl_training}
Our SimCLR and VICREeg implementation is adapted from solo-learn\cite{JMLR:v23:21-1155}, and is available in the source code. The self-supervised model is trained on the images observed in the current episode only, never on the full dataset. For DINOv2, we extract frozen embeddings from a pretrained foundational model, specifically the \texttt{dinov2\ vitb14} backbone, 768-dimensional, without any further fine-tuning.
\subsection{Feature Normalization}
Each feature vector is divided by its $\ell_2$ norm, yielding unit-norm representations. Similarities are therefore computed with the cosine distance.
\section{Compute resources}
Each experiment trained deep‑learning models on GPUs, consuming up to 22 GB of GPU memory and no more than 20 GB of system RAM.

\section{Source code}The complete source code is provided in the supplementary ZIP file and will be publicly released on GitHub upon acceptance. The source code includes a README that lists the commands required to reproduce all of the experiments described in this paper.
\section{Additional results}
\subsection{Main results tables}
The tables~\ref{tab:faa_cifar-100_er_ace_star}-~\ref{tab:aaa_cifar-100_er_ace} presents the complete tables for the results in Section~\ref{subsec:main_results}, evaluated with both the FAA and AAA metrics.
\begin{table*}[!t]
\caption{Final Averaged Accuracy (FAA) on Split CIFAR-100 with three CL algorithms ,averaged over 5 independent runs (mean ± standard error). For each \buffer, the best FAA is in bold.}
\begin{subtable}[t]{\textwidth}
\centering

\fontsize{9}{11}
\rmfamily
\setlength{\tabcolsep}{1mm}\caption{ER ACE STAR}
\label{tab:faa_cifar-100_er_ace_star}
\rowcolors{2}{tableShade}{white}%
\begin{tabular}{lcccccc}
\toprule
 & \multicolumn{1}{c}{Random} & \multicolumn{1}{c}{\probcover} & \multicolumn{2}{c}{\maxherding} & \multicolumn{1}{c}{Herding} & \multicolumn{1}{c}{TEAL}  \\
\cmidrule(lr){2-2} \cmidrule(lr){3-3} \cmidrule(lr){4-5} \cmidrule(lr){6-6} \cmidrule(lr){7-7} 
Buffer & Supervised & Supervised & Supervised & MERS & Supervised & Supervised  \\
\midrule
100 & 21.93 \std{0.17} & 29.32 \std{0.20} & 32.04 \std{0.32} & \textbf{33.43} \std{0.44} & 21.57 \std{0.30} & 29.68 \std{0.36}  \\
300 & 31.85 \std{0.38} & 41.47 \std{0.24} & 42.07 \std{0.28} & \textbf{44.00} \std{0.18} & 33.47 \std{0.21} & 41.33 \std{0.26}  \\
500 & 36.68 \std{0.49} & 45.39 \std{0.17} & 46.43 \std{0.19} & \textbf{47.81} \std{0.15} & 38.28 \std{0.19} & 45.76 \std{0.34} \\
1000 & 44.62 \std{0.20} & 50.96 \std{0.25} & 50.86 \std{0.27} & \textbf{53.50} \std{0.30} & 44.81 \std{0.15} & 50.98 \std{0.26}  \\
2000 & 51.31 \std{0.27} & 55.49 \std{0.20} & 56.27 \std{0.37} & \textbf{58.44} \std{0.24} & 51.30 \std{0.24} & 55.56 \std{0.15}  \\

\bottomrule
\end{tabular}
\end{subtable}

\begin{subtable}[t]{\textwidth}
\centering
\fontsize{9}{11}
\rmfamily
\setlength{\tabcolsep}{1mm}\caption{ER ACE}
\label{tab:faa_cifar-100_er_ace}
\rowcolors{2}{tableShade}{white}%
\begin{tabular}{lcccccc}
\toprule
 & \multicolumn{1}{c}{Random} & \multicolumn{1}{c}{\probcover} & \multicolumn{2}{c}{\maxherding} & \multicolumn{1}{c}{Herding} & \multicolumn{1}{c}{TEAL}  \\
\cmidrule(lr){2-2} \cmidrule(lr){3-3} \cmidrule(lr){4-5} \cmidrule(lr){6-6} \cmidrule(lr){7-7} 
Buffer & Supervised & Supervised & Supervised & MERS & Supervised & Supervised  \\
\midrule
100 & 21.80 \std{0.34} & 28.13 \std{0.35} & 29.35 \std{0.30} & \textbf{30.95} \std{0.44} & 22.08 \std{0.16} & 29.67\std{0.13} \\
300 & 32.01 \std{0.30} & 38.30 \std{0.15} & 39.33 \std{0.13} & \textbf{40.55} \std{0.28} & 29.94 \std{0.22} & 37.60 \std{0.25}  \\
500 & 36.29 \std{0.52} & 42.22 \std{0.25} & 43.55 \std{0.10} & \textbf{45.26} \std{0.19} & 35.58 \std{0.22} & 41.44 \std{0.23} \\
1000 & 43.30 \std{0.21} & 48.44 \std{0.22} & 49.19 \std{0.23} & \textbf{50.64} \std{0.32} & 42.71 \std{0.12} & 47.33 \std{0.24}  \\
2000 & 50.14 \std{0.30} & 53.85 \std{0.27} & 53.69 \std{0.26} & \textbf{55.42} \std{0.21} & 50.09 \std{0.21} & 53.04 \std{0.12}  \\
\bottomrule
\end{tabular}
\end{subtable}

\begin{subtable}[t]{\textwidth}
\centering
\fontsize{9}{11}
\rmfamily
\setlength{\tabcolsep}{1mm}\caption{ER}
\label{tab:faa_cifar-100_er}
\rowcolors{2}{tableShade}{white}%
\begin{tabular}{lccccccc}
\toprule
 & \multicolumn{1}{c}{Random} & \multicolumn{1}{c}{\probcover} & \multicolumn{2}{c}{\maxherding} & \multicolumn{1}{c}{Herding} & \multicolumn{1}{c}{TEAL} & \multicolumn{1}{c}{Rainbow} \\
\cmidrule(lr){2-2} \cmidrule(lr){3-3} \cmidrule(lr){4-5} \cmidrule(lr){6-6} \cmidrule(lr){7-7} \cmidrule(lr){8-8}
\buffer & Supervised & Supervised & Supervised & MERS & Supervised & Supervised & Supervised \\
\midrule

300 & 13.25 \std{0.10} & 16.29 \std{0.21} & \textbf{17.60} \std{0.18} & \textbf{17.74} \std{0.25} & 16.02\std{0.20} & 17.06 \std{0.13} & 13.46 \std{0.10}  \\
500 & 17.69 \std{0.30} & 22.03 \std{0.17} & \textbf{23.54} \std{0.15} & \textbf{23.78} \std{0.14} & 20.20\std{0.85} & 22.49 \std{0.20} & 16.98 \std{0.60}  \\
1000 & 26.04 \std{0.24} & 31.65 \std{0.29} & \textbf{32.78} \std{0.32} & \textbf{33.26} \std{0.24} & 29.80\std{0.35} & 31.92 \std{0.43} & 26.72 \std{0.17} \\
2000 & 38.30 \std{0.23} & 42.76 \std{0.09} & 42.88 \std{0.22} & \textbf{43.89} \std{0.33} & 41.74\std{0.29} & 42.22 \std{0.51} & 38.40 \std{0.22}  \\

\bottomrule
\end{tabular}
\end{subtable}
\end{table*}

\begin{figure*}[h]
    \centering
\begin{subfigure}[t]{\columnwidth}
        \centering
        \includegraphics[width=\linewidth]{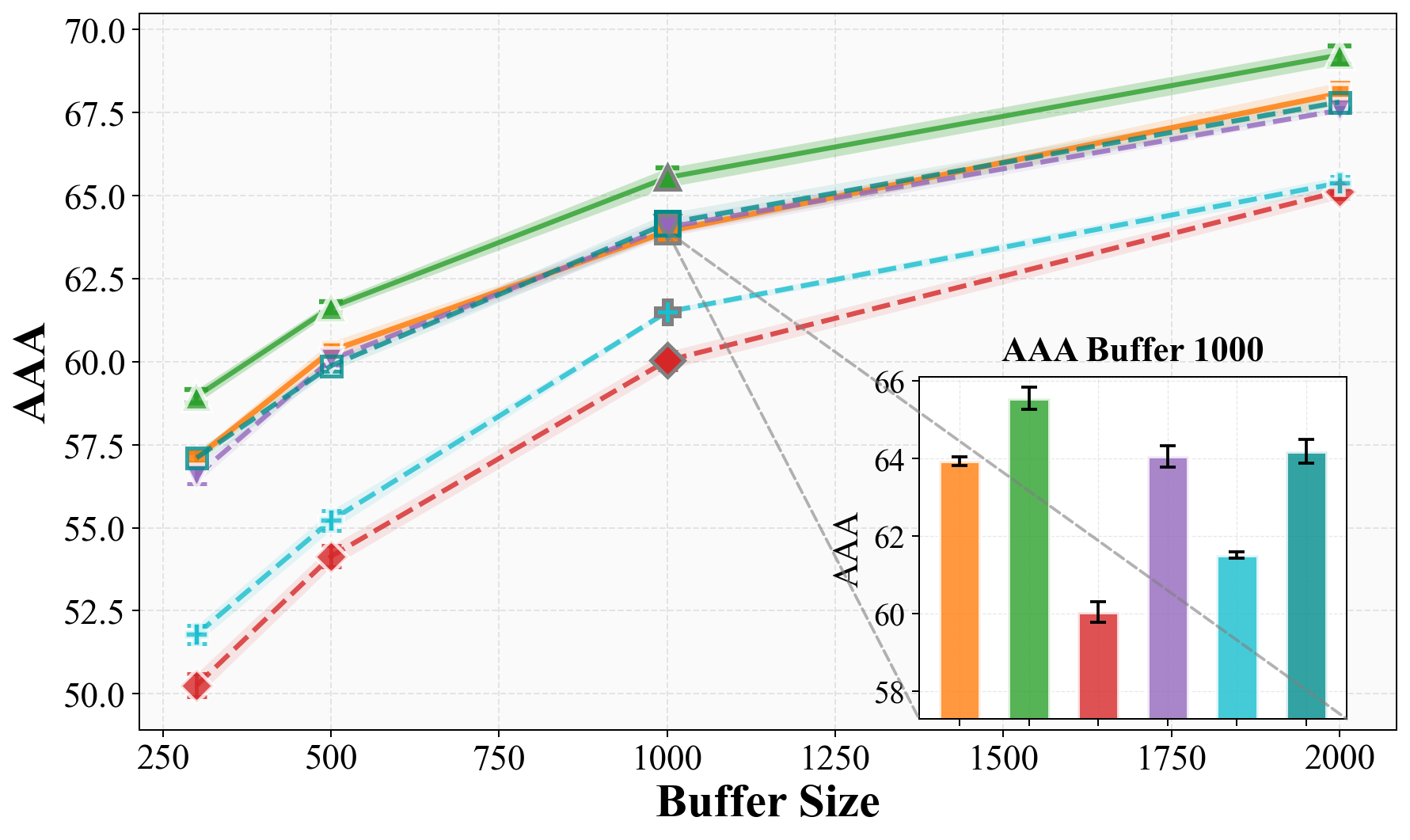}
        \caption{ER-ACE-STAR}
    \end{subfigure}
    \hfill
    \begin{subfigure}[t]{\columnwidth}
        \centering
        \includegraphics[width=\linewidth]{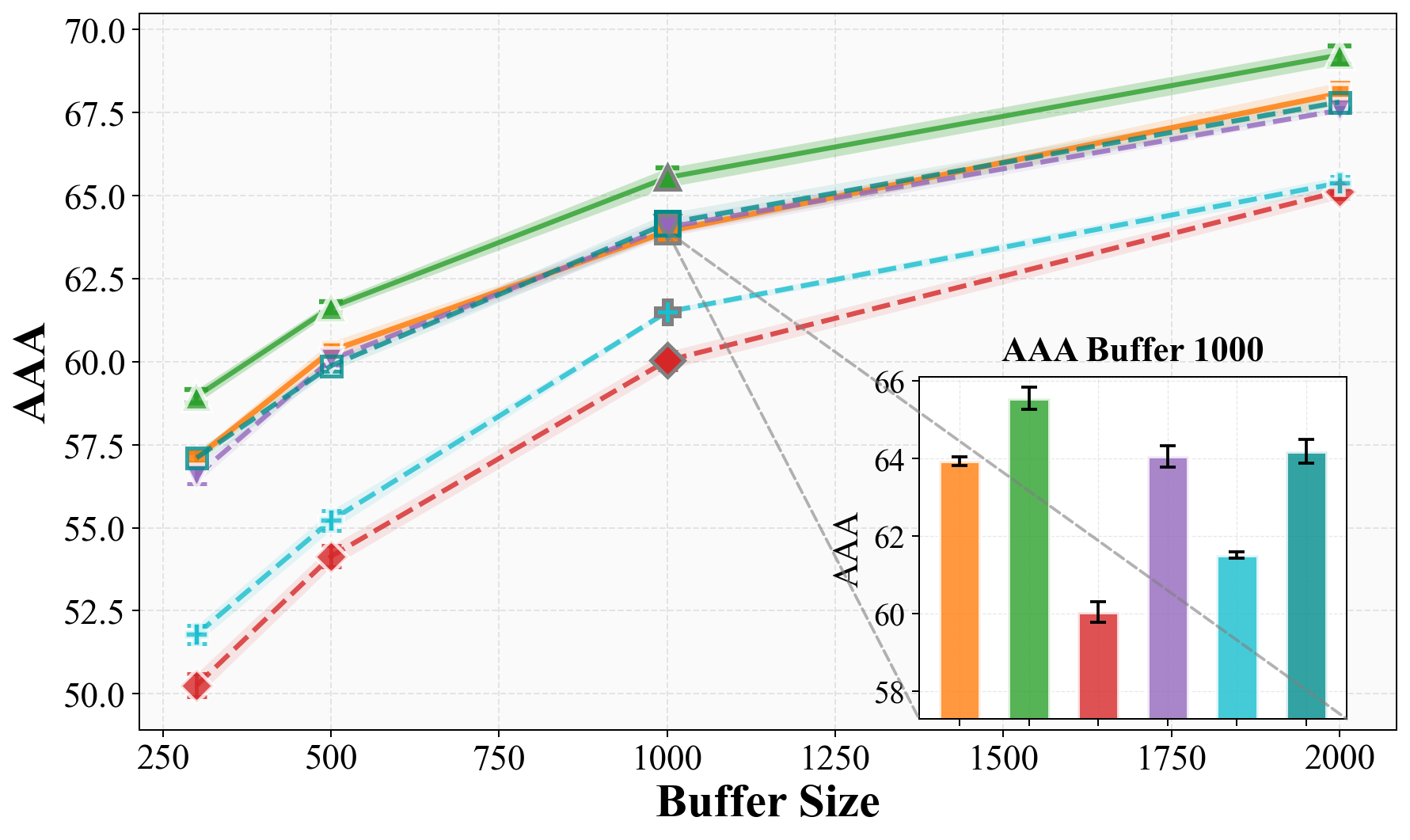}
        \caption{ER-ACE}
    \end{subfigure}

    \vspace{0.5em}

\caption{AAA as a function of memory size $|M|$ on Split CIFAR-100 for different continual learning algorithms. Results with \method are compared against alternative selection strategies.}

    \label{fig:mers_max_herding_er_ace_cifar100_aaa}
\end{figure*}

\subsection{Selection stability}
\label{sec:appendix_selection_stability}

\begin{figure*}
\begin{minipage}[t]{0.99\textwidth}
    \centering
    \begin{subfigure}[t]{0.48\textwidth}
        \centering
        \includegraphics[width=\linewidth]{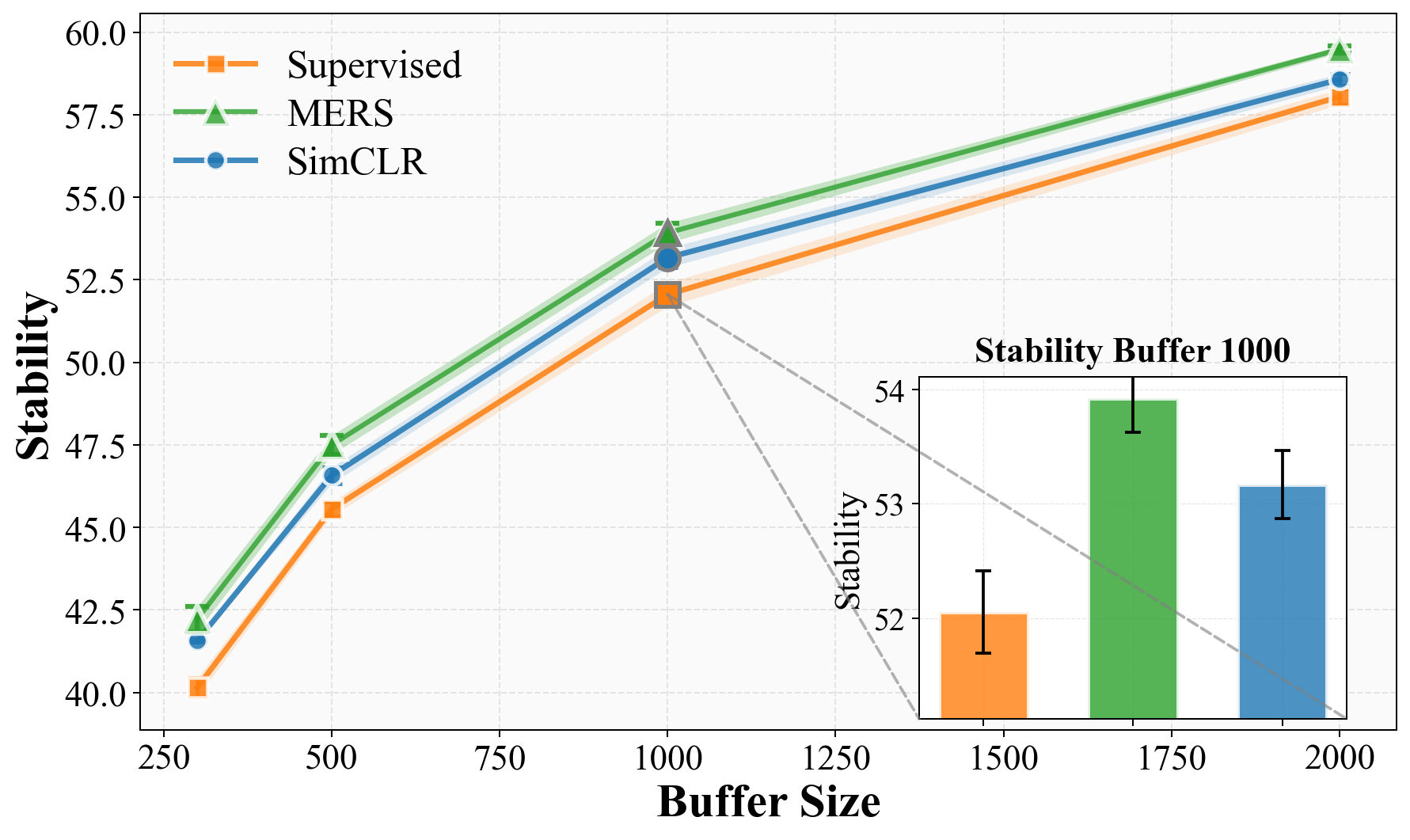}
        \caption{Stability}
    \end{subfigure}
    \hfill
    \begin{subfigure}[t]{0.48\textwidth}
        \centering
        \includegraphics[width=\linewidth]{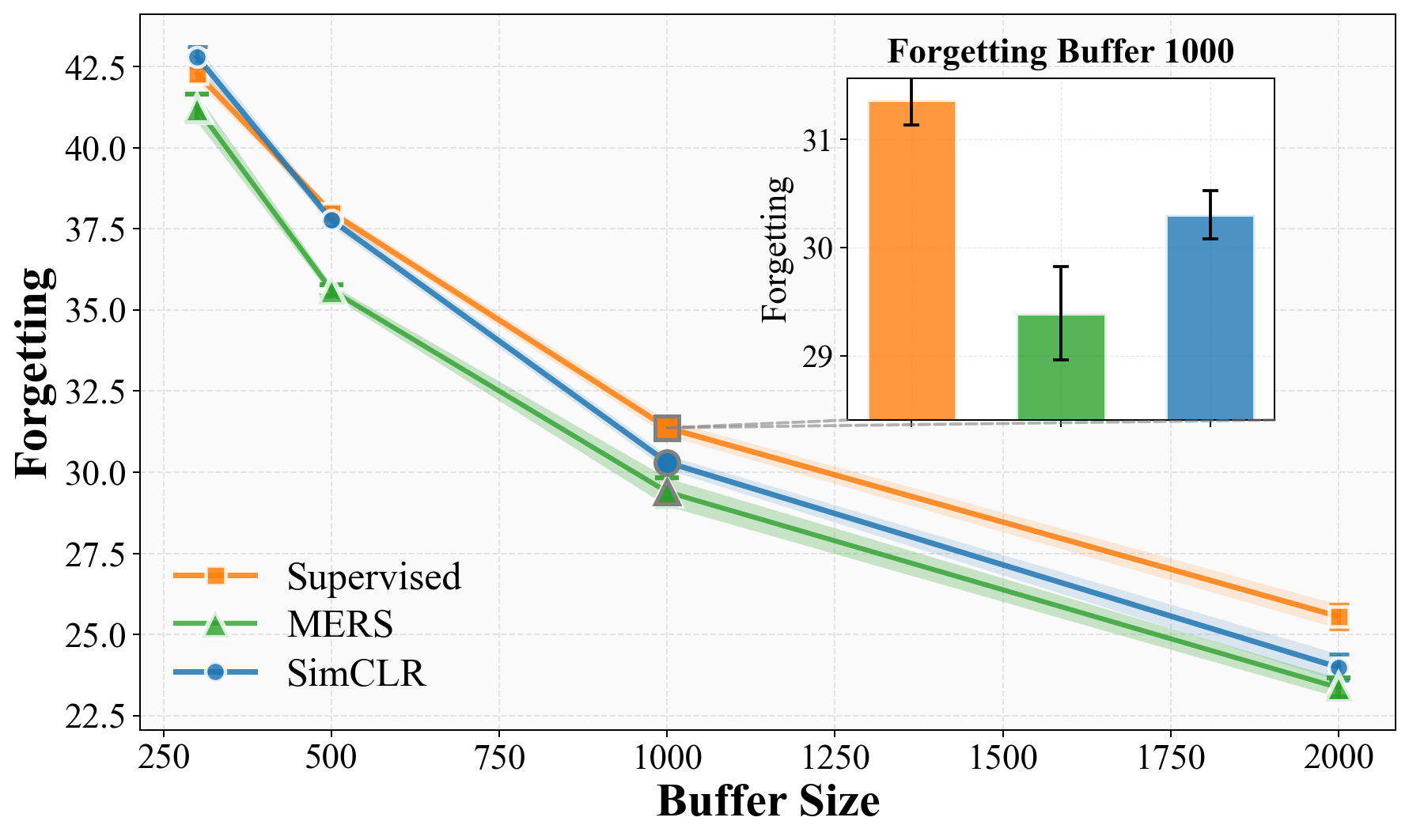}
        \caption{Forgetting}
    \end{subfigure}

    \caption{Stability and forgetting of ER-ACE with \method as a function of $|M|$ on Split CIFAR-100.}\label{fig:stability_forgetting_er_ace_cifar100}

\end{minipage}
\end{figure*}
\begin{figure*}
\begin{minipage}[t]{0.99\textwidth}
    \centering
    \begin{subfigure}[t]{0.48\textwidth}
        \centering
        \includegraphics[width=\linewidth]{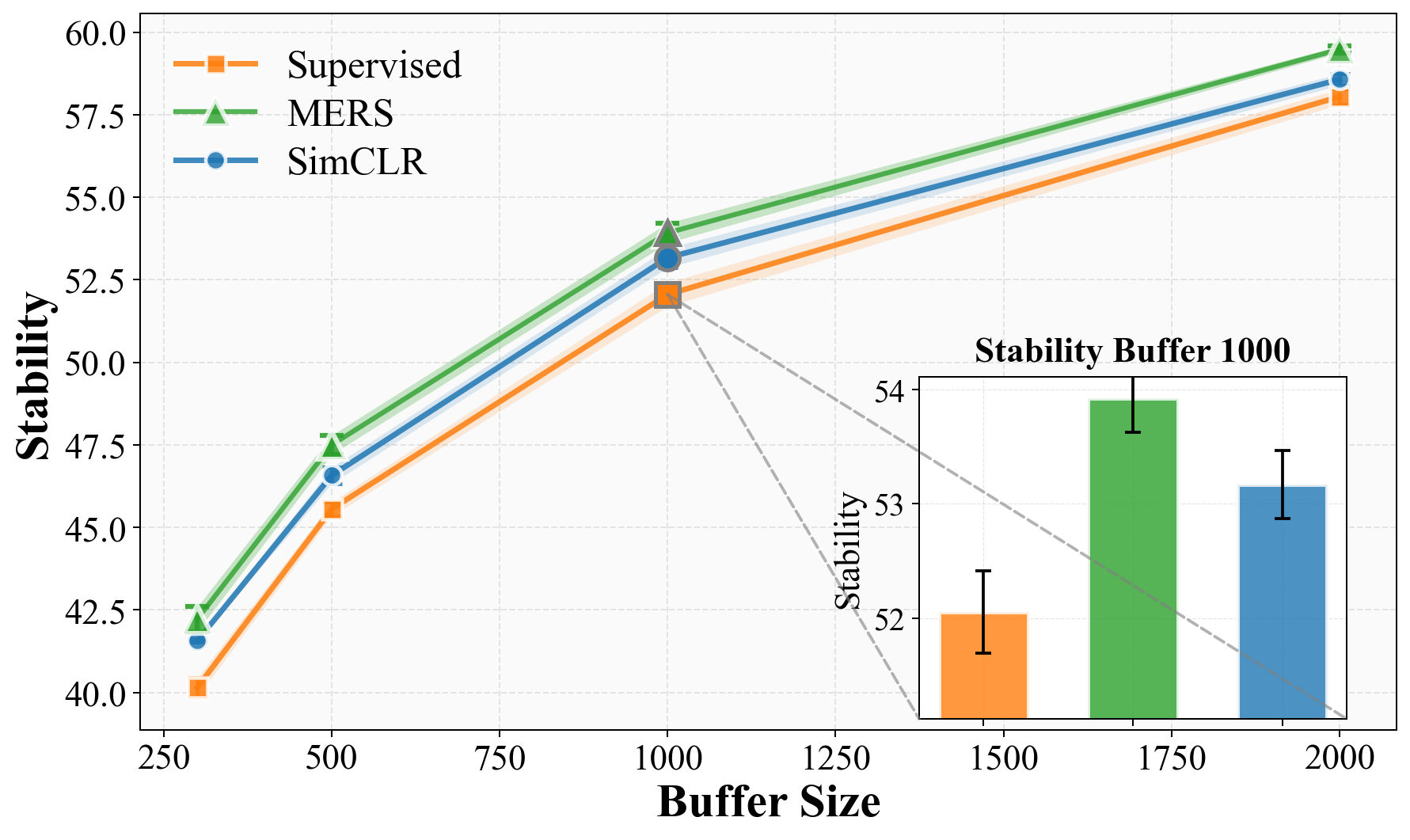}
        \caption{Stability}
    \end{subfigure}
    \hfill
    \begin{subfigure}[t]{0.48\textwidth}
        \centering
        \includegraphics[width=\linewidth]{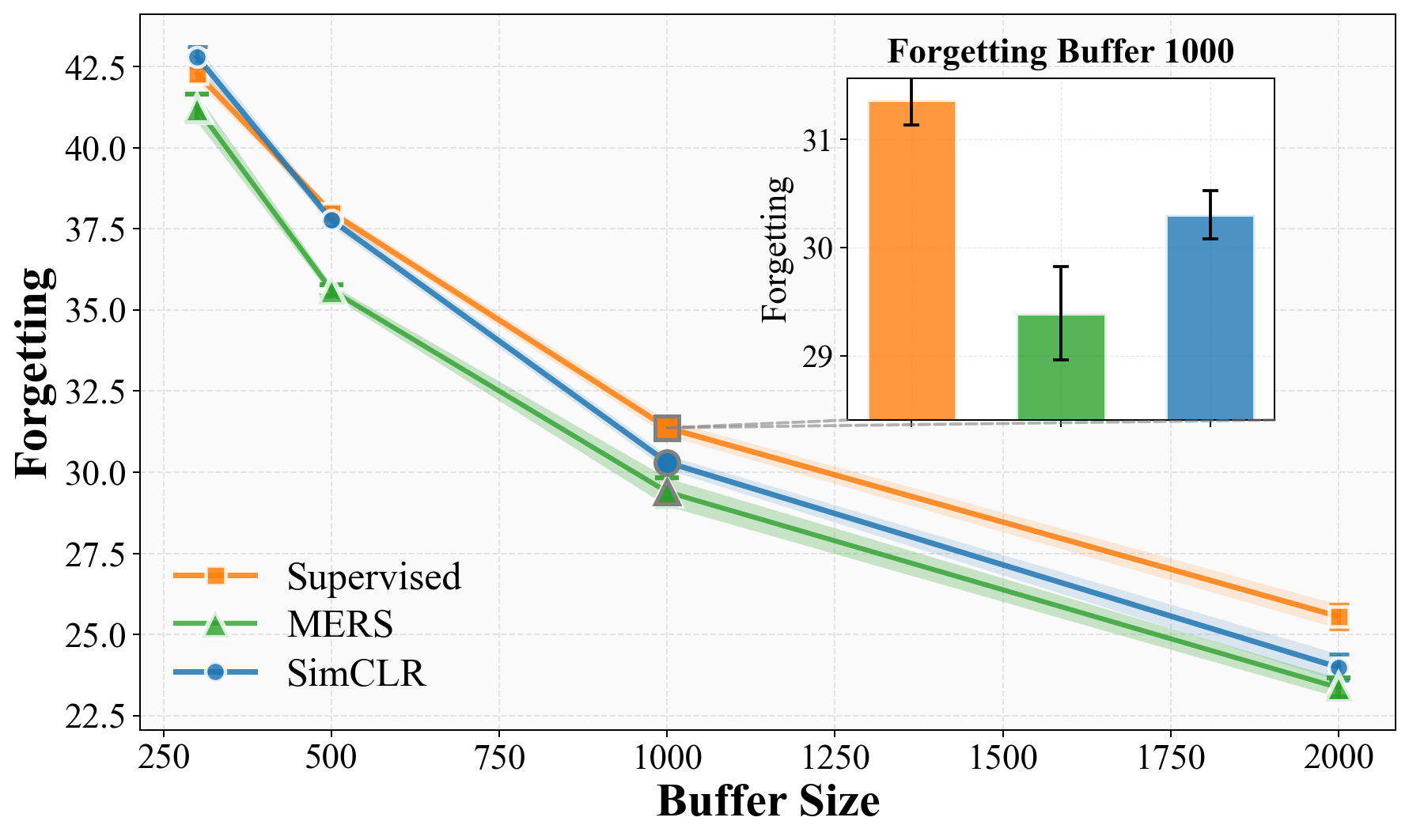}
        \caption{Forgetting}
    \end{subfigure}

    \caption{Stability and forgetting of ER with \method as a function of $|M|$ on Split CIFAR-100.}
    \label{fig:stability_forgetting_er_cifar100}

\end{minipage}
\hfill
\end{figure*}
We provide additional results for \textbf{ER} and We provide additional results for \textbf{ER} and \textbf{ER-ACE} on \textbf{Split CIFAR-100}, reported in Fig.~\ref{fig:stability_forgetting_er_ace_cifar100}-~\ref{fig:stability_forgetting_er_cifar100}

\section{Ablation Study}
\label{app:ablation}

We compare \method against a MaxHerding variant that relies solely on SSL embeddings. As shown in Fig.~\ref{fig:MERS-embeddings}, \method consistently achieves higher FAA and AAA accuracy, highlighting the benefit of combining Self-Supervised and Supervised representations.
Fig.~\ref{fig:weights_ablation} presents an ablation study on the effect of the embedding weight parameter $\alpha$ when using the median K-NN density defined in Eq.~\ref{eq:knn_density}, applied to \method \probcover on Split CIFAR-100 under the ER-ACE setting. The results indicate a slight but consistent improvement when using the formulation of $\alpha$ given in Eq.~\ref{eq:weight}.  


\begin{figure}[H]
    \centering
    \includegraphics[width=\linewidth]{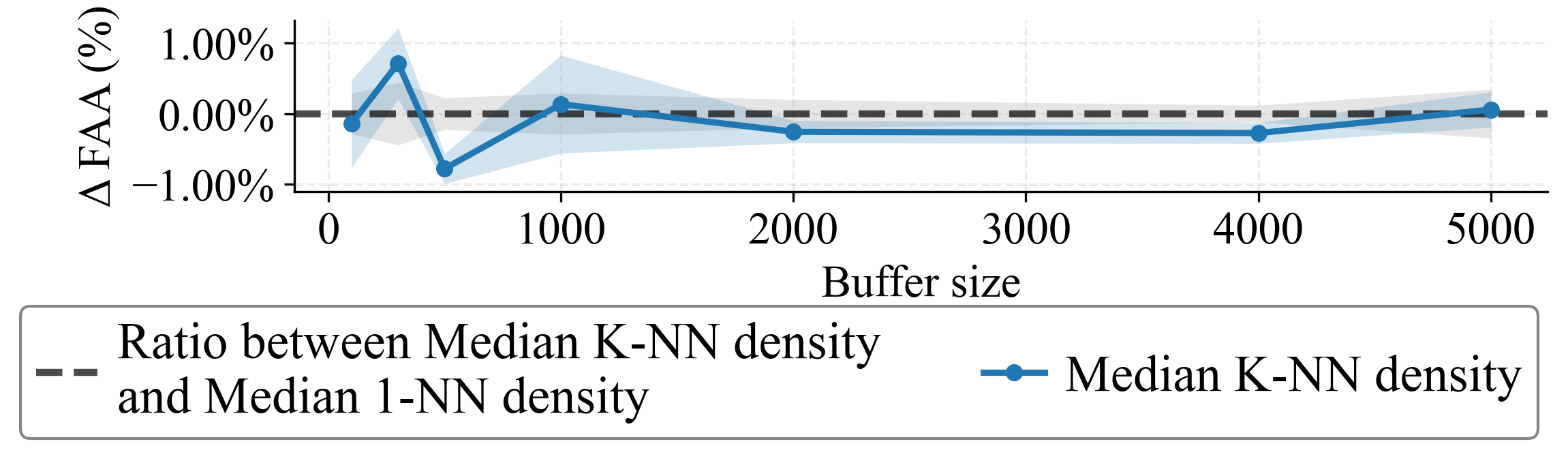}
    \caption{\textbf{\method \maxherding.} Ablation of the embedding weight $\alpha$ using K-NN density estimators on Split CIFAR-100 with ER-ACE. The baseline corresponds to Eq.~(\ref{eq:weight}), and a minor but consistent improvement is observed with this weighting.}
    \label{fig:weights_ablation}
\end{figure}


\section{Robustness to Episode Class Order in Continual Learning}\label{Robustness_to_Episode_Class_Order}
As in the experiments presented in Tables 1–2, we repeated them using different episode Class orders. Below are the Final Averaged Accuracy and the Anytime Averaged Accuracy Tables~\ref{tab:faa_42}-~\ref{tab:aaa_35}.

\begin{table*}
\caption{\textbf{FAA} with a different class ordering, averaged over 5 independent runs (mean ± standard error). Several sample-selection strategies and embedding spaces are compared across multiple replay-buffer sizes (\buffer). For each \buffer, the best AAA is in bold; result within the standard error of the
best are also bolded.}
\label{tab:faa_42}
\begin{subtable}[t]{\textwidth}
\centering
\caption{FAA on Split CIFAR-100 \textbf{ER ACE}.}

\fontsize{9}{11}\selectfont
\rmfamily
\setlength{\tabcolsep}{1mm}
\rowcolors{2}{tableShade}{white}%
\begin{tabular}{l|C{1.5cm}|ccc|ccc|ccc}
\toprule
& \multicolumn{1}{c|}{Random}
& \multicolumn{3}{c|}{\textbf{\method \probcover}}
& \multicolumn{3}{c|}{\textbf{\method \maxherding}}
 \\
\cmidrule(lr){2-2} \cmidrule(lr){3-5} \cmidrule(lr){6-8}
Buffer & Supervised & Supervised & SimCLR & \method & Supervised & SimCLR & \method  \\
\midrule
100 & 20.79 \std{0.27} & 29.81 \std{0.20} & 27.82 \std{0.25} & 29.35 \std{0.24} & 29.42 \std{0.11} & 29.20 \std{0.33} & \textbf{29.89} \std{0.22} &  \\
300 & 31.76 \std{0.07} & 38.84 \std{0.19} & 37.78 \std{0.14} & 39.47 \std{0.28} & 38.16 \std{0.35} & 38.73 \std{0.36} & \textbf{39.60} \std{0.25} & \\
500 & 35.80 \std{0.32} & 42.46 \std{0.23} & 42.25 \std{0.19} & 43.28 \std{0.23} & 42.72 \std{0.23} & 42.82 \std{0.21} & \textbf{43.71} \std{0.21} &\\
1000 & 42.27 \std{0.22} & 47.69 \std{0.23} & 48.11 \std{0.19} & 48.98 \std{0.27} & 47.77 \std{0.21} & 48.89 \std{0.27} & \textbf{50.00} \std{0.16} &  \\
2000 & 49.41 \std{0.18} & 52.99 \std{0.07} & 53.53 \std{0.24} & 54.17 \std{0.19} & 53.18 \std{0.20} & 54.20 \std{0.28} & \textbf{54.80} \std{0.19} &  \\
4000 & 55.32 \std{0.24} & 58.03 \std{0.11} & 58.79 \std{0.24} & 59.28 \std{0.18} & 58.52 \std{0.20} & 58.56 \std{0.34} & \textbf{59.03} \std{0.19} &  \\
5000 & 57.96 \std{0.21} & 60.10 \std{0.11} & \textbf{60.79} \std{0.19} & \textbf{60.84} \std{0.27} & 59.85 \std{0.18} & 59.90 \std{0.11} & \textbf{60.07} \std{0.13} &  \\

\bottomrule
\end{tabular}

\end{subtable}

\begin{subtable}[!t]{\textwidth}
\centering
\caption{FAA on Split CIFAR-100 \textbf{ER}.}

\fontsize{9}{11}
\rmfamily
\setlength{\tabcolsep}{1mm}

\rowcolors{2}{tableShade}{white}%
\begin{tabular}{l|C{1.5cm}|ccc|ccc|ccc}
\toprule
& \multicolumn{1}{c|}{Random}
& \multicolumn{3}{c|}{\textbf{\method \probcover}}
& \multicolumn{3}{c|}{\textbf{\method \maxherding}}
 \\
\cmidrule(lr){2-2} \cmidrule(lr){3-5} \cmidrule(lr){6-8} 

Buffer & Supervised & Supervised & SimCLR & \method & Supervised & SimCLR & \method  \\
\midrule
100 & 10.50 \std{0.14} & \textbf{13.02} \std{0.11} & 11.44 \std{0.13} & 12.18 \std{0.08} & 12.53 \std{0.18} & 12.24 \std{0.07} & 12.79 \std{0.12}  \\
300 & 14.67 \std{0.24} & \textbf{20.32} \std{0.26} & 19.01 \std{0.34} & \textbf{20.33} \std{0.21} & 19.62 \std{0.17} & 18.83 \std{0.56} & 19.52 \std{0.33}  \\
500 & 19.86 \std{0.31} & \textbf{25.37} \std{0.18} & 23.68 \std{0.35} & \textbf{25.01} \std{0.44} & \textbf{25.73} \std{0.22} & 24.25 \std{0.22} & \textbf{25.74} \std{0.55}  \\
1000 & 28.48 \std{0.22} & 34.37 \std{0.33} & 33.62 \std{0.29} & \textbf{35.18} \std{0.21} & 34.54 \std{0.34} & 34.43 \std{0.35} & \textbf{35.40} \std{0.20}  \\
2000 & 40.45 \std{0.23} & 43.84 \std{0.32} & 44.34 \std{0.31} & \textbf{45.38} \std{0.28} & 44.62 \std{0.19} & 44.76 \std{0.37} & \textbf{45.58} \std{0.40}  \\
4000 & 51.23 \std{0.22} & 53.91 \std{0.20} & 54.37 \std{0.20} & \textbf{54.97} \std{0.27} & \textbf{54.69} \std{0.32} & 54.01 \std{0.16} & \textbf{54.81} \std{0.16}  \\
5000 & 55.03 \std{0.20} & 56.75 \std{0.13} & 57.16 \std{0.25} & \textbf{57.79} \std{0.22} & 56.67 \std{0.23} & 56.49 \std{0.21} & 57.06 \std{0.23} \\

\bottomrule
\end{tabular}
\setlength{\tabcolsep}{1mm}
\end{subtable}

\begin{subtable}[t]{\textwidth}
\caption{FAA on  Split CIFAR-100 \textbf{ER ACE}.}

\centering
\fontsize{9}{11}\selectfont
\rmfamily
\setlength{\tabcolsep}{1mm}
\rowcolors{2}{tableShade}{white}%
\begin{tabular}{l|C{1.5cm}|ccc|ccc|ccc}
\toprule
& \multicolumn{1}{c|}{Random}
& \multicolumn{3}{c|}{\textbf{\method \probcover}}
& \multicolumn{3}{c|}{\textbf{\method \maxherding}} \\
\cmidrule(lr){2-2} \cmidrule(lr){3-5} \cmidrule(lr){6-8}
Buffer & Supervised & Supervised & SimCLR & \method & Supervised & SimCLR & \method  \\
\midrule
200 & 11.89 \std{0.13} & \textbf{13.95} \std{0.17} & 12.58 \std{0.05} & 13.33 \std{0.09} & 13.54 \std{0.07} & 13.50 \std{0.14} & 13.91 \std{0.25} \\
400 & 13.27 \std{0.12} & \textbf{15.69} \std{0.12} & 14.33 \std{0.12} & 15.16 \std{0.11} & 14.72 \std{0.21} & 15.00 \std{0.10} & \textbf{15.35} \std{0.19} \\
600 & 13.47 \std{0.08} & \textbf{16.44} \std{0.06} & 15.42 \std{0.19} & \textbf{16.64} \std{0.24} & 15.71 \std{0.18} & 16.07 \std{0.10} & \textbf{16.37} \std{0.31} \\
1000 & 14.50 \std{0.16} & 18.16 \std{0.19} & 16.99 \std{0.19} & \textbf{18.44} \std{0.11} & 17.51 \std{0.16} & 17.26 \std{0.15} & 17.89 \std{0.12} \\
2000 & 16.59 \std{0.15} & 20.50 \std{0.21} & 19.71 \std{0.19} & \textbf{21.03} \std{0.09} & 20.08 \std{0.23} & 19.50 \std{0.20} & 20.26 \std{0.27} \\
4000 & 19.11 \std{0.13} & 23.09 \std{0.15} & 22.94 \std{0.17} & \textbf{24.45} \std{0.20} & 23.18 \std{0.21} & 22.43 \std{0.33} & 22.94 \std{0.20} \\
6000 & 22.57 \std{0.06} & 25.41 \std{0.20} & 25.42 \std{0.30} & \textbf{26.40} \std{0.26} & 25.66 \std{0.19} & 24.66 \std{0.18} & 25.02 \std{0.15} \\

\bottomrule
\end{tabular}
\end{subtable}
\end{table*}
\begin{table*}
\caption{\textbf{FAA} with a different class ordering, averaged over 5 independent runs (mean ± standard error). Several sample-selection strategies and embedding spaces are compared across multiple replay-buffer sizes (\buffer). For each \buffer, the best AAA is in bold; result swithin the standard error of the
best are also bolded.}
\label{tab:faa_35}
\begin{subtable}[!t]{\textwidth}
\caption{FAA on Split CIFAR-100 \textbf{ER ACE}.}

\centering
\fontsize{9}{11}
\rmfamily
\setlength{\tabcolsep}{1mm}
\label{tab:faa_cifar-100_er_ace_seed_35}
\rowcolors{2}{tableShade}{white}%
\begin{tabular}{l|C{1.5cm}|ccc|ccc|}
\toprule
& \multicolumn{1}{c|}{Random}
& \multicolumn{3}{c|}{\textbf{\method \probcover}}
& \multicolumn{3}{c|}{\textbf{\method \maxherding}}
\\
\cmidrule(lr){2-2} \cmidrule(lr){3-5} \cmidrule(lr){6-8} 
Buffer & Supervised & Supervised & SimCLR & \method & Supervised & SimCLR & \method  \\
\midrule
100 & 20.59 \std{0.23} & 27.64 \std{0.44} & 25.67 \std{0.45} & 27.42 \std{0.32} & 28.10 \std{0.30} & 28.30 \std{0.44} & \textbf{29.35} \std{0.25}   \\
300 & 28.61 \std{0.05} & 37.84 \std{0.14} & 36.90 \std{0.31} & 38.88 \std{0.21} & 37.70 \std{0.34} & 37.63 \std{0.31} & \textbf{39.19} \std{0.19}   \\
500 & 35.30 \std{0.19} & 42.23 \std{0.19} & 41.75 \std{0.18} & 43.55 \std{0.23} & 42.02 \std{0.25} & 42.42 \std{0.13} & \textbf{44.02} \std{0.38}   \\
1000 & 41.99 \std{0.16} & 48.18 \std{0.22} & 48.36 \std{0.15} & 48.96 \std{0.25} & 47.89 \std{0.23} & 48.71 \std{0.30} & \textbf{49.63} \std{0.30}  \\
2000 & 49.00 \std{0.23} & 53.20 \std{0.22} & 53.51 \std{0.08} & 54.44 \std{0.28} & 53.67 \std{0.22} & 54.30 \std{0.14} & \textbf{55.11} \std{0.10}  \\
4000 & 56.89 \std{0.11} & 59.01 \std{0.27} & 59.18 \std{0.12} & \textbf{59.73} \std{0.15} & 59.30 \std{0.05} & 59.23 \std{0.14} & \textbf{59.67} \std{0.10}  \\
5000 & 58.75 \std{0.22} & 60.45 \std{0.18} & 60.65 \std{0.11} & \textbf{61.40} \std{0.22} & 60.17 \std{0.11} & 60.37 \std{0.09} & \textbf{60.94} \std{0.07}   \\

\bottomrule
\end{tabular}
\end{subtable}

\begin{subtable}[!t]{\textwidth}

\caption{FAA on Split CIFAR-100 \textbf{ER}.}

\centering
\fontsize{9}{11}
\rmfamily
\setlength{\tabcolsep}{1mm}
\begin{tabular}{l|C{1.5cm}|ccc|ccc|}
\toprule
& \multicolumn{1}{c|}{Random}
& \multicolumn{3}{c|}{\textbf{\method \probcover}}
& \multicolumn{3}{c|}{\textbf{\method \maxherding}}
 \\
\cmidrule(lr){2-2} \cmidrule(lr){3-5} \cmidrule(lr){6-8} 
Buffer & Supervised & Supervised & SimCLR & \method & Supervised & SimCLR & \method  \\
\midrule
100 & 9.95 \std{0.07} & \textbf{11.45} \std{0.11} & 10.32 \std{0.06} & 11.17 \std{0.17} & 11.19 \std{0.05} & 11.05 \std{0.18} & \textbf{11.51} \std{0.01}  \\
\textbf{}\textbf{}300 & 13.71 \std{0.09} & \textbf{18.87} \std{0.11} & 17.16 \std{0.14} & 18.71 \std{0.25} & 18.10 \std{0.34} & 18.01 \std{0.22} & 18.71 \std{0.25}  \\
500 & 17.41 \std{0.38} & 23.75 \std{0.39} & 22.37 \std{0.34} & \textbf{24.66} \std{0.14} & 24.11 \std{0.15} & 23.40 \std{0.15} & \textbf{24.66} \std{0.29}  \\
1000 & 27.44 \std{0.48} & 33.50 \std{0.17} & 32.51 \std{0.48} & 33.99 \std{0.27} & 33.70 \std{0.20} & 33.27 \std{0.16} & \textbf{34.64} \std{0.31}  \\
2000 & 39.78 \std{0.30} & 43.73 \std{0.01} & 43.74 \std{0.30} & 44.02 \std{0.20} & 44.06 \std{0.30} & 44.01 \std{0.20} & \textbf{45.22} \std{0.16}  \\
\bottomrule

\end{tabular}
\label{tab:faa_cifar-100_er_seed_35}

\end{subtable}

\begin{subtable}[t]{\textwidth}
\caption{FAA on  Split CIFAR-100 \textbf{ER ACE}.}

\centering
\fontsize{9}{11}\selectfont
\rmfamily
\setlength{\tabcolsep}{1mm}
\rowcolors{2}{tableShade}{white}%
\begin{tabular}{l|C{1.5cm}|ccc|ccc|}
\toprule
& \multicolumn{1}{c|}{Random}
& \multicolumn{3}{c|}{\textbf{\method \probcover}}
& \multicolumn{3}{c|}{\textbf{\method \maxherding}}
 \\
\cmidrule(lr){2-2} \cmidrule(lr){3-5} \cmidrule(lr){6-8}
Buffer & Supervised & Supervised & SimCLR & \method  & Supervised & SimCLR & \method \\
\midrule
200 & 11.39 \std{0.10} & \textbf{13.23} \std{0.12} & 12.22 \std{0.14} & 12.75 \std{0.13} & 13.11 \std{0.08} & 12.71 \std{0.11} & 12.95 \std{0.09} \\
400 & 11.98 \std{0.24} & \textbf{15.09} \std{0.21} & 13.70 \std{0.16} & 14.71 \std{0.24} & 13.84 \std{0.15} & 14.12 \std{0.21} & 14.51 \std{0.16} \\
600 & 12.90 \std{0.13} & \textbf{16.18} \std{0.12} & 14.68 \std{0.08} & 15.78 \std{0.18} & 14.97 \std{0.19} & 15.48 \std{0.13} & 15.14 \std{0.06}\\
1000 & 14.14 \std{0.09} & \textbf{17.67} \std{0.26} & 16.21 \std{0.22} & \textbf{17.47} \std{0.15} & 16.61 \std{0.11} & 16.32 \std{0.18} & 16.77 \std{0.10} \\
2000 & 15.94 \std{0.16} & 19.88 \std{0.24} & 18.60 \std{0.23} & \textbf{20.42} \std{0.29} & 19.70 \std{0.34} & 19.01 \std{0.14} & 19.21 \std{0.22} \\
4000 & 19.42 \std{0.22} & 22.86 \std{0.13} & 23.05 \std{0.35} & \textbf{24.08} \std{0.07} & 22.80 \std{0.12} & 21.84 \std{0.28} & 21.84 \std{0.23} \\
6000 & 22.05 \std{0.25} & 25.98 \std{0.34} & 25.63 \std{0.30} & \textbf{26.53} \std{0.13} & 25.14 \std{0.23} & 24.43 \std{0.28} & 25.23 \std{0.25} \\
\bottomrule
\end{tabular}
\label{tab:faa_tinyimg_er_ace_seed_35}
\end{subtable}


\end{table*}

\begin{table*}[!h]
\caption{\textbf{AAA} with a different class ordering, averaged over 5 independent runs (mean ± standard error). Several sample-selection strategies and embedding spaces are compared across multiple replay-buffer sizes (\buffer). For each \buffer, the best AAA is in bold; result swithin the standard error of the
best are also bolded.}
\label{tab:aaa_42}

\begin{subtable}[t]{\textwidth}
\caption{AAA on Split CIFAR-100 \textbf{ER ACE}.}

\centering
\fontsize{9}{11}\selectfont
\rmfamily
\setlength{\tabcolsep}{1mm}
\rowcolors{2}{tableShade}{white}%
\begin{tabular}{l|C{1.5cm}|ccc|ccc|ccc}
\toprule
& \multicolumn{1}{c|}{Random}
& \multicolumn{3}{c|}{\textbf{\method \probcover}}
& \multicolumn{3}{c|}{\textbf{\method \maxherding}}
 \\
\cmidrule(lr){2-2} \cmidrule(lr){3-5} \cmidrule(lr){6-8} 
Buffer & Supervised & Supervised & SimCLR & \method & Supervised & SimCLR & \method  \\
\midrule
100 & 39.94 \std{0.09} & 46.09 \std{0.08} & 45.60 \std{0.09} & \textbf{46.90} \std{0.18} & 46.21 \std{0.18} & 46.17 \std{0.11} & \textbf{46.78} \std{0.27}  \\
300 & 49.19 \std{0.10} & 53.33 \std{0.07} & 53.62 \std{0.09} & 54.30 \std{0.13} & 53.46 \std{0.30} & 53.92 \std{0.19} & \textbf{54.68} \std{0.14} \\
500 & 52.85 \std{0.08} & 56.55 \std{0.15} & 56.76 \std{0.12} & 57.07 \std{0.26} & 56.52 \std{0.12} & 57.34 \std{0.10} & \textbf{57.77} \std{0.07}  \\
1000 & 57.60 \std{0.13} & 60.65 \std{0.09} & 60.90 \std{0.10} & 61.46 \std{0.06} & 60.60 \std{0.23} & 61.25 \std{0.06} & \textbf{61.97} \std{0.17}  \\
2000 & 62.35 \std{0.14} & 64.36 \std{0.20} & 64.85 \std{0.11} & 64.91 \std{0.12} & 64.29 \std{0.13} & 65.01 \std{0.08} & \textbf{65.26} \std{0.15} \\
4000 & 66.73 \std{0.17} & 68.22 \std{0.09} & 68.39 \std{0.17} & \textbf{68.67} \std{0.14} & 68.20 \std{0.14} & 67.97 \std{0.08} & 68.16 \std{0.07}  \\
5000 & 68.32 \std{0.16} & 69.36 \std{0.08} & 69.92 \std{0.15} & \textbf{69.80} \std{0.17} & 69.40 \std{0.14} & 69.07 \std{0.04} & 69.20 \std{0.10}  \\

\bottomrule
\end{tabular}
\label{tab:aaa_cifar-100_er_ace_seed_42}
\end{subtable}

\begin{subtable}[!t]{\textwidth}
\caption{AAA on Split CIFAR-100 \textbf{ER}.}
\centering
\fontsize{9}{11}
\rmfamily
\setlength{\tabcolsep}{1mm}
\rowcolors{2}{tableShade}{white}%
\begin{tabular}{l|C{1.5cm}|ccc|ccc|ccc}
\toprule
& \multicolumn{1}{c|}{Random}
& \multicolumn{3}{c|}{\textbf{\method \probcover}}
& \multicolumn{3}{c|}{\textbf{\method \maxherding}}
 \\
\cmidrule(lr){2-2} \cmidrule(lr){3-5} \cmidrule(lr){6-8} 

Buffer & Supervised & Supervised & SimCLR & \method & Supervised & SimCLR & \method  \\
\midrule
100 & 28.19 \std{0.11} & 30.89 \std{0.07} & 29.72 \std{0.24} & \textbf{30.51} \std{0.16} & 30.31 \std{0.11} & 30.37 \std{0.15} & \textbf{30.49} \std{0.15}  \\
\textbf{}300 & 34.42 \std{0.42} & 38.55 \std{0.37} & 38.12 \std{0.33} & \textbf{39.30} \std{0.13} & 38.44 \std{0.25} & 37.92 \std{0.46} & 38.37 \std{0.42}  \\
500 & 40.31 \std{0.21} & 43.90 \std{0.27} & 42.60 \std{0.39} & \textbf{43.55} \std{0.34} & \textbf{43.74} \std{0.09} & 43.58 \std{0.37} & 44.68 \std{0.31} \\
1000 & 48.62 \std{0.24} & 51.78 \std{0.20} & 51.86 \std{0.22} & \textbf{52.58} \std{0.35} & 51.67 \std{0.37} & 52.27 \std{0.31} & \textbf{52.71} \std{0.25}  \\
2000 & 58.66 \std{0.23} & 59.70 \std{0.47} & 60.57 \std{0.20} & \textbf{61.48} \std{0.18} & 60.68 \std{0.13} & 60.73 \std{0.37} & 60.49 \std{0.28}  \\
4000 & 66.49 \std{0.18} & 67.67 \std{0.23} & 67.41 \std{0.17} & \textbf{68.30} \std{0.32} & 68.10 \std{0.12} & 67.35 \std{0.27} & 68.12 \std{0.11}  \\
5000 & 68.89 \std{0.17} & 69.58 \std{0.18} & 69.53 \std{0.22} & \textbf{70.13} \std{0.21} & 69.02 \std{0.21} & 69.21 \std{0.32} & 69.43 \std{0.05}  \\
\bottomrule
\end{tabular}
\setlength{\tabcolsep}{1mm}

\label{tab:aaa_cifar-100_er_seed_42}
\end{subtable}

\begin{subtable}[t]{\textwidth}
\caption{AAA on  Split CIFAR-100 \textbf{ER ACE}.}
\centering
\fontsize{9}{11}\selectfont
\rmfamily
\setlength{\tabcolsep}{1mm}
\rowcolors{2}{tableShade}{white}%
\begin{tabular}{l|C{1.5cm}|ccc|ccc|ccc}
\toprule
& \multicolumn{1}{c|}{Random}
& \multicolumn{3}{c|}{\textbf{\method \probcover}}
& \multicolumn{3}{c|}{\textbf{\method \maxherding}}
 \\
\cmidrule(lr){2-2} \cmidrule(lr){3-5} \cmidrule(lr){6-8} 
Buffer & Supervised & Supervised & SimCLR & \method & Supervised & SimCLR & \method  \\
\midrule
200 & 25.92 \std{0.07} & \textbf{28.32} \std{0.06} & 27.43 \std{0.10} & 28.17 \std{0.11} & 27.91 \std{0.09} & 27.93 \std{0.09} & \textbf{28.20} \std{0.08} \\
400 & 27.78 \std{0.16} & \textbf{30.60} \std{0.13} & 29.61 \std{0.07} & \textbf{30.50} \std{0.09} & 29.73 \std{0.05} & 29.94 \std{0.10} & 30.10 \std{0.10} \\
600 & 28.96 \std{0.07} & 31.60 \std{0.13} & 30.94 \std{0.09} & \textbf{31.82} \std{0.08} & 31.18 \std{0.09} & 31.40 \std{0.15} & 31.56 \std{0.13} \\
1000 & 30.49 \std{0.05} & 33.60 \std{0.15} & 33.05 \std{0.13} & \textbf{34.08} \std{0.10} & 33.43 \std{0.12} & 33.12 \std{0.20} & 33.22 \std{0.11} \\
2000 & 33.23 \std{0.13} & 36.09 \std{0.13} & 35.84 \std{0.11} & \textbf{36.86} \std{0.05} & 36.08 \std{0.14} & 35.51 \std{0.14} & 36.04 \std{0.20} \\
4000 & 36.95 \std{0.13} & 39.32 \std{0.13} & 39.06 \std{0.10} & \textbf{39.87} \std{0.12} & 39.10 \std{0.12} & 38.47 \std{0.09} & 38.57 \std{0.12} \\
6000 & 39.66 \std{0.12} & 40.90 \std{0.12} & 41.19 \std{0.10} & \textbf{41.67} \std{0.08} & 40.94 \std{0.15} & 40.08 \std{0.10} & 40.26 \std{0.16} \\

\bottomrule
\end{tabular}
\end{subtable}
\end{table*}
\begin{table*}[!t]
\caption{\textbf{AAA} with a different class ordering, averaged over 5 independent runs (mean ± standard error). Several sample-selection strategies and embedding spaces are compared across multiple replay-buffer sizes (\buffer). For each \buffer, the best AAA is in bold; result swithin the standard error of the
best are also bolded.}
\label{tab:aaa_35}
\begin{subtable}[t]{\textwidth}
\caption{AAA on Split CIFAR-100 \textbf{ER ACE}.}
\centering
\fontsize{9}{11}\selectfont
\rmfamily
\setlength{\tabcolsep}{1mm}
\rowcolors{2}{tableShade}{white}%
\begin{tabular}{l|C{1.5cm}|ccc|ccc|ccc}
\toprule
& \multicolumn{1}{c|}{Random}
& \multicolumn{3}{c|}{\textbf{\method \probcover}}
& \multicolumn{3}{c|}{\textbf{\method \maxherding}} \\
\cmidrule(lr){2-2} \cmidrule(lr){3-5} \cmidrule(lr){6-8} 
Buffer & Supervised & Supervised & SimCLR & \method & Supervised & SimCLR & \method  \\
\midrule
100 & 41.79 \std{0.14} & 47.76 \std{0.18} & 47.34 \std{0.14} & 48.44 \std{0.08} & 48.51 \std{0.17} & 48.63 \std{0.15} & \textbf{49.18} \std{0.11} \\
300 & 51.26 \std{0.11} & 56.10 \std{0.12} & 56.54 \std{0.18} & \textbf{57.50} \std{0.13} & 56.33 \std{0.17} & 57.15 \std{0.14} & \textbf{57.78} \std{0.13}  \\
500 & 56.12 \std{0.28} & 59.79 \std{0.14} & 60.32 \std{0.15} & \textbf{61.35} \std{0.08} & 60.28 \std{0.12} & 60.63 \std{0.12} & \textbf{61.45} \std{0.11}  \\
1000 & 61.84 \std{0.20} & 64.49 \std{0.20} & 65.41 \std{0.12} & \textbf{65.54} \std{0.14} & 64.56 \std{0.12} & 65.04 \std{0.17} & \textbf{65.79} \std{0.22}  \\
\textbf{}2000 & 66.46 \std{0.05} & 68.70 \std{0.17} & 68.92 \std{0.18} & \textbf{69.02} \std{0.15} & 68.87 \std{0.12} & 69.22 \std{0.09} & \textbf{69.29} \std{0.17}  \\
4000 & 71.57 \std{0.10} & 72.34 \std{0.17} & 72.52 \std{0.11} & 73.00 \std{0.03} & 72.53 \std{0.06} & 72.30 \std{0.21} & 72.44 \std{0.13}  \\
5000 & 72.95 \std{0.09} & 73.61 \std{0.10} & 73.48 \std{0.09} & \textbf{74.22} \std{0.12} & 73.14 \std{0.09} & 73.36 \std{0.14} & 73.66 \std{0.03}  \\
\bottomrule
\end{tabular}
\label{tab:aaa_cifar-100_er_ace_seed_35}

\end{subtable}

\begin{subtable}[!t]{\textwidth}
\caption{AAA on Split CIFAR-100 \textbf{ER}.}
\centering
\fontsize{9}{11}
\rmfamily
\setlength{\tabcolsep}{1mm}

\rowcolors{2}{tableShade}{white}%
\begin{tabular}{l|C{1.5cm}|ccc|ccc|ccc}
\toprule
& \multicolumn{1}{c|}{Random}
& \multicolumn{3}{c|}{\textbf{\method \probcover}}
& \multicolumn{3}{c|}{\textbf{\method \maxherding}}
 \\
\cmidrule(lr){2-2} \cmidrule(lr){3-5} \cmidrule(lr){6-8} 

Buffer & Supervised & Supervised & SimCLR & \method & Supervised & SimCLR & \method  \\
\midrule
\textbf{}\textbf{}100 & 29.82 \std{0.13} & \textbf{32.42} \std{0.04} & 31.62 \std{0.12} & 32.58 \std{0.20} & 32.28 \std{0.08} & 32.09 \std{0.16} & 32.58 \std{0.17}  \\
300 & 37.89 \std{0.08} & 42.18 \std{0.13} & 41.20 \std{0.14} & \textbf{42.32} \std{0.15} & 41.33 \std{0.18} & 41.74 \std{0.14} & \textbf{42.39} \std{0.04}\\
500 & 43.07 \std{0.17} & 47.15 \std{0.13} & 47.20 \std{0.09} & \textbf{48.48} \std{0.09} & 47.70 \std{0.09} & 47.64 \std{0.12} & \textbf{48.52} \std{0.19} \\
1000 & 52.56 \std{0.13} & 55.93 \std{0.05} & 55.92 \std{0.29} & 56.60 \std{0.31} & 56.30 \std{0.08} & 56.35 \std{0.20} & \textbf{57.20} \std{0.11}  \\
2000 & 62.59 \std{0.15} & 64.42 \std{0.21} & 64.67 \std{0.11} & 64.56 \std{0.04} & 64.45 \std{0.13} & 64.52 \std{0.07} & \textbf{64.96} \std{0.10}  \\

\bottomrule
\end{tabular}
\setlength{\tabcolsep}{1mm}

\label{tab:aaa_cifar-100_er_seed_35}
\end{subtable}
\begin{subtable}[t]{\textwidth}
\caption{AAA on  Split CIFAR-100 \textbf{ER-ACE}.}

\centering
\fontsize{9}{11}\selectfont
\rmfamily
\setlength{\tabcolsep}{1mm}
\rowcolors{2}{tableShade}{white}%
\begin{tabular}{l|C{1.5cm}|ccc|ccc|}
\toprule
& \multicolumn{1}{c|}{Random}
& \multicolumn{3}{c|}{\textbf{\method \probcover}}
& \multicolumn{3}{c|}{\textbf{\method \maxherding}}
 \\
\cmidrule(lr){2-2} \cmidrule(lr){3-5} \cmidrule(lr){6-8} 
Buffer & Supervised & Supervised & SimCLR & \method & Supervised & SimCLR & \method \\
\midrule
200 & 26.65 \std{0.04} & \textbf{28.49} \std{0.11} & 27.57 \std{0.06} & 28.24 \std{0.09} & 28.16 \std{0.19} & 28.19 \std{0.05} & \textbf{28.31} \std{0.12} \\
400 & 28.01 \std{0.07} & \textbf{30.93} \std{0.23} & 29.82 \std{0.13} & 30.79 \std{0.15} & 30.06 \std{0.02} & 30.28 \std{0.12} & 30.49 \std{0.12} \\
600 & 29.02 \std{0.12} & \textbf{32.01} \std{0.09} & 31.05 \std{0.16} & \textbf{32.21} \std{0.14} & 31.76 \std{0.18} & 31.45 \std{0.09} & 31.66 \std{0.08} \\
1000 & 31.03 \std{0.15} & 33.92 \std{0.13} & 32.97 \std{0.07} & \textbf{34.43} \std{0.16} & 33.52 \std{0.14} & 33.15 \std{0.16} & 33.11 \std{0.12} \\
2000 & 34.01 \std{0.16} & 36.25 \std{0.22} & 35.97 \std{0.17} & \textbf{36.96} \std{0.10} & 36.65 \std{0.11} & 36.11 \std{0.18} & 36.11 \std{0.19} \\
4000 & 37.83 \std{0.13} & 39.51 \std{0.11} & 39.78 \std{0.21} & \textbf{40.28} \std{0.12} & 39.37 \std{0.13} & 38.70 \std{0.10} & 39.05 \std{0.11} \\
6000 & 40.15 \std{0.24} & 42.05 \std{0.26} & 41.66 \std{0.15} & \textbf{42.57} \std{0.15} & 41.37 \std{0.14} & 40.76 \std{0.24} & 41.43 \std{0.04} \\

\bottomrule
\end{tabular}
\setlength{\tabcolsep}{1mm}
\label{tab:aaa_tinyimg_er_ace_seed_35}
\end{subtable}
\begin{subtable}[t]{\textwidth}
\caption{AAA on  Split CIFAR-100 \textbf{ER}.}

\centering
\fontsize{9}{11}\selectfont
\rmfamily
\setlength{\tabcolsep}{1mm}
\rowcolors{2}{tableShade}{white}%
\begin{tabular}{l|C{1.5cm}|ccc|ccc|ccc}
\toprule
& \multicolumn{1}{c|}{Random}
& \multicolumn{3}{c|}{\textbf{\method \probcover}}
& \multicolumn{3}{c|}{\textbf{\method \maxherding}}
\\
\cmidrule(lr){2-2} \cmidrule(lr){3-5} \cmidrule(lr){6-8} 
Buffer & Supervised & Supervised & SimCLR & \method & Supervised & SimCLR & \method \\
\midrule
200 & 21.09 \std{0.10} & 21.06 \std{0.04} & 21.03 \std{0.02} & 21.00 \std{0.09} & 21.12 \std{0.13} & \textbf{21.22} \std{0.02} & 21.16 \std{0.12} \\
400 & 20.94 \std{0.07} & 21.48 \std{0.11} & 21.06 \std{0.04} & \textbf{21.59} \std{0.06} & 21.56 \std{0.05} & 21.34 \std{0.05} & 21.33 \std{0.09} \\
600 & 21.16 \std{0.10} & \textbf{22.15} \std{0.09} & 21.59 \std{0.09} & 21.91 \std{0.11} & \textbf{22.17} \std{0.10} & 21.75 \std{0.08} & 21.78 \std{0.06} \\
1000 & 21.91 \std{0.15} & 23.30 \std{0.05} & 22.91 \std{0.15} & \textbf{23.64} \std{0.12} & 23.30 \std{0.10} & 22.82 \std{0.09} & 22.83 \std{0.08} \\
2000 & 25.57 \std{0.14} & \textbf{27.72} \std{0.14} & 26.72 \std{0.10} & \textbf{27.64} \std{0.09} & 27.25 \std{0.15} & 26.46 \std{0.16} & 27.03 \std{0.10} \\
4000 & 33.03 \std{0.11} & 35.37 \std{0.30} & 34.41 \std{0.18} & \textbf{36.29} \std{0.15} & 35.24 \std{0.17} & 34.08 \std{0.08} & 34.45 \std{0.23} \\
6000 & 39.88 \std{0.15} & 41.56 \std{0.17} & 41.02 \std{0.17} & \textbf{41.70} \std{0.14} & 40.87 \std{0.24} & 39.76 \std{0.18} & 40.65 \std{0.07} \\

\bottomrule
\end{tabular}
\setlength{\tabcolsep}{1mm}
\end{subtable}
\end{table*}

\end{document}